\documentclass{article} 
\usepackage{iclr2027_conference,times}

\usepackage{amsmath,amsfonts,bm}

\def\eqref#1{equation~\ref{#1}}

\def\1{\bm{1}}

\DeclareMathAlphabet{\mathsfit}{\encodingdefault}{\sfdefault}{m}{sl}
\SetMathAlphabet{\mathsfit}{bold}{\encodingdefault}{\sfdefault}{bx}{n}

\usepackage{hyperref}
\usepackage{url}
\usepackage{multirow}
\usepackage{graphicx}
\usepackage[most]{tcolorbox}
\usepackage{amsmath}
\usepackage{amssymb}
\usepackage{mathtools}
\usepackage{amsthm}
\usepackage{booktabs}
\usepackage{wrapfig}
\usepackage{subcaption}
\usepackage{mathtools}

\theoremstyle{plain}
\newtheorem{theorem}{Theorem}[section]
\newtheorem{proposition}[theorem]{Proposition}
\newtheorem{lemma}[theorem]{Lemma}
\newtheorem{corollary}[theorem]{Corollary}
\theoremstyle{definition}
\newtheorem{definition}[theorem]{Definition}
\newtheorem{assumption}[theorem]{Assumption}
\theoremstyle{remark}

\title{A Geometric View for Understanding Concept Learning and Neuron Interpretation in Sparse Autoencoders}

\author{Chenhao Zhang\thanks{Equal contribution.}, Chris Lin\footnotemark[1], Xiaojian Chen, Hanqi Feng \& Su-In Lee\thanks{Corresponding author.} \\
Paul G. Allen School of Computer Science \& Engineering \\
University of Washington \\
\texttt{\{cnhzhang,clin25,xchen172,hanqif2,suinlee\}@cs.washington.edu}
}

\iclrfinalcopy 
\begin{document}

\maketitle
\ificlrfinal\lhead{Preprint}\fi 

\begin{abstract}
   We propose a unified mathematical framework for a geometric and structural understanding of concept learning and neuron interpretation in sparse autoencoders (SAEs). While SAEs improve interpretability of neural networks by learning sparse feature representations, a principled definition of ''concept'' and ''learning'' remains unclear. We formalize concepts as sets of data points and cast concept learning as a set-alignment problem between human-defined and model-induced concepts. This formulation distinguishes three increasingly strong notions of learning (i.e., detection, separation, and approximation) and yields geometric conditions, error bounds, and capacity constraints for when concepts can be represented by individual neurons or multi-neuron units. It also provides a set-theoretic account for common SAE phenomena, including feature splitting, feature absorption, and hierarchical concepts. Finally, we connect concept learning and neuron interpretation through formal concept analysis, showing that the two directions need not agree and that their many-to-many structure can be organized by concept lattices. Experiments with ReLU and Top-$K$ SAEs illustrate the theory and reveal the effects of SAE hyperparameters on concept learning.
\end{abstract}


\section{Introduction}


Large neural networks achieve remarkable performance, yet their internal mechanisms remain opaque, limiting scientific understanding~\citep{science,interplm} and reliable deployment~\citep{auditing}. Mechanistic interpretability seeks to understand how information is represented and transformed in model activations~\citep{mechinterp1,mechinterp2}. However, interpretability can be hindered by polysemanticity~\citep{monosemantic}, commonly attributed to superposition~\citep{superposition,originalsae2}: a neuron may activate for multiple concepts, and a concept may be distributed among multiple neurons. Sparse autoencoders (SAEs) aim to mitigate this problem with overcomplete sparse representations that are more monosemantic and interpretable~\citep{andrewsae, originalsae}.

SAEs are commonly motivated by the linear representation hypothesis~\citep{linearrepresentationhypothesis}, under which concepts correspond to directions in activation space that combine approximately linearly. Yet a direction by itself does not specify an interpretable concept. In practice, SAE neurons are interpreted through the examples that activate them most strongly, by identifying patterns shared across those examples~\citep{autointerp}. Practically, then, the interpretation of an SAE neuron is determined by a \emph{set of examples} rather than a vector.

This distinction points to a more fundamental ambiguity: the literature lacks a precise definition of \emph{concept}, reflecting a long-standing divide in philosophy~\citep{philosophysae}. A Platonic view treats a concept as an abstract entity that exists independently of its instances. A class-nominalist view instead identifies a concept with its \emph{extension}, the set of instances to which it applies. Because interpretability methods typically infer and validate interpretations from activating examples, their operational notion of concept is closer to the class-nominalist view. Accordingly, here we define a concept by a set of examples. This leads to our central formulation: 
\begin{tcolorbox}
Concept learning and neuron interpretation are \emph{set-alignment} problems between human-defined and model-induced concepts, where a concept is a set, and learning is alignment.
\end{tcolorbox}

{The two families of sets need not coincide. \citet{understandvocab} argue that machines hold concepts that the human vocabulary cannot name, and \citet{cammarata2026clean} argues that seemingly polysemantic neurons often detect a single coherent pattern. In our framework, a neuron's activation set is itself a concept, so polysemanticity is relative to the human vocabulary (Figure~\ref{fig:overview}d,e), and treating concepts as sets of inputs makes this mismatch measurable.}

Under this formulation, learning a concept amounts to aligning a neuron-induced set with a human-defined set of examples. For instance, a neuron learns a concept when it selects examples that agree with the human-defined set. Interpreting a neuron amounts to aligning human-defined sets to the neuron-induced set. For instance, a neuron is polysemantic when it represents multiple sets. 

Our contributions are the following. \textbf{(1)} We propose a unified geometric and set-theoretic framework that casts concept learning and neuron interpretation as set alignments. \textbf{(2)} We distinguish between three modes of concept learning (detection, separation, and approximation), deriving the necessary and sufficient conditions for each mode. \textbf{(3)} We show that concept learning and neuron interpretation are related but distinct, formally relating them through concept lattices~\citep{conceptlattice}.

\begin{figure}[t]
\centering
\vspace{-25pt}
\includegraphics[width=\linewidth]{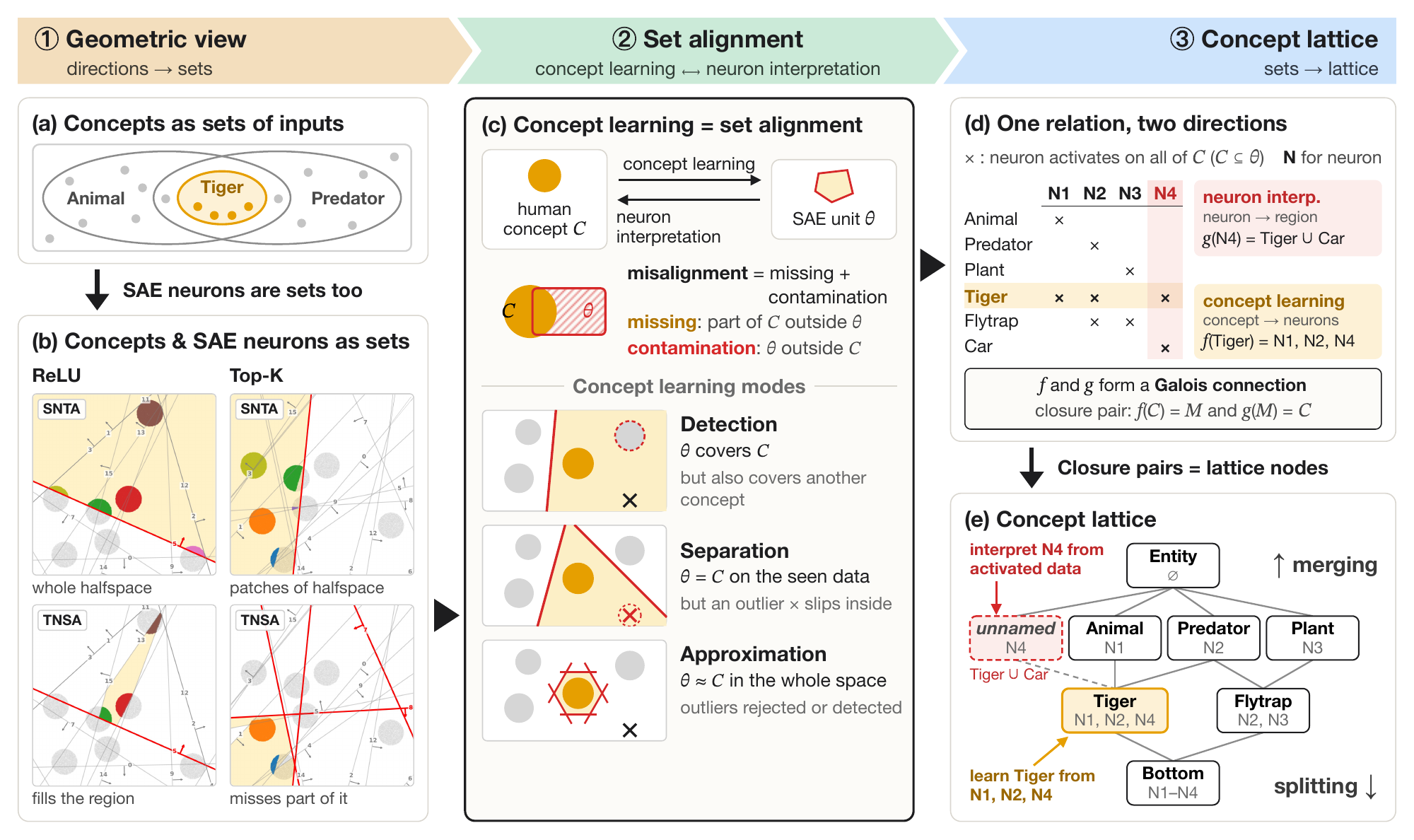}
\vspace{-25pt}
\caption{Overview of our framework.
\textbf{(a,b)} Human concepts and SAE neurons are both sets of input data; (b) shows ReLU and Top-$K$ SAE's activation pattern for single-neurons-total-activation (SNTA) and total-neuron-single-activation (TNSA).
\textbf{(c)} Concept learning maps a concept $C$ to a unit $\theta$, three concept learning modes with increasing strength are shown; neuron interpretation maps $\theta$ back.
\textbf{(d,e)} The two directions of concept learning and neuron interpretation are the rows ($f$) and columns ($g$) of one incidence relation, organized by a concept lattice, and need not agree: $f(\text{Tiger}) = \{N_1, N_2, N_4\}$, but $g(N_4) = \text{Tiger} \cup \text{Car}$.}
\vspace{-15pt}
\label{fig:overview}
\end{figure}

\section{Preliminaries}\label{sec:preliminary}
\textbf{Due to space limit, we delay the related work section to Appendix.~\ref{app:related_work}}. We first review SAEs, focusing on ReLU SAE~\citep{originalsae} and Top-$K$ SAE~\citep{topksae}. Let $x\in\mathbb R^n$, $n\ge1$, be an activation vector from a neural network. An SAE maps $x$ to a sparse activation vector $a\in\mathbb R^d$, where $d\ge1$ is the number of SAE neurons, usually larger than $n$, and then reconstructs $x$:
\begin{align*}
 z&=\operatorname{Enc}(x)=W_{\mathrm{enc}}(x-b_{\mathrm{pre}})+b_{\mathrm{enc}},\\
 a&=\operatorname{Act}(z),\qquad \hat x=\operatorname{Dec}(a)=W_{\mathrm{dec}}a+b_{\mathrm{dec}}.
\end{align*}
Here $z$ contains the pre-activation scores, $a_i$ is neuron $i$'s activation, and $\hat x$ is the reconstruction. The $W$ terms are weight matrices, $b_{\mathrm{pre}}$ centers $x$, $b_{\mathrm{enc}},b_{\mathrm{dec}}$ are biases, and $\operatorname{Act}$ is the activation function. ReLU uses $a_i=\max\{z_i,0\}$ and loss $\|x-\hat x\|_2^2+\lambda\|a\|_1$, where $\lambda\ge0$ controls the sparsity penalty. Top-$K$ keeps the $k$ largest scores, applies ReLU to them, and sets the others to zero; its loss is $\|x-\hat x\|_2^2$. We take $1\le k\le d$ and resolve ties by a fixed ordering of neuron indices, giving at most $k$ positive activations. Writing $W=W_{\mathrm{enc}}$ and $b=b_{\mathrm{enc}}-W_{\mathrm{enc}}b_{\mathrm{pre}}$, we have $z_i(x)=\langle w_i,x\rangle+b_i$, where $w_i$ is row $i$ of $W$. Thus a linear direction, together with its bias and activation, determines a set of activating inputs.

\section{Main Framework}\label{sec:framework}
\subsection{Notation and Definitions}
\textbf{Concepts.} Let $X\subseteq\mathbb R^n$ be a Borel data domain in the SAE input space. A human concept is a measurable set $C\in\mathcal B(X)$, and $\mathcal C\subseteq\mathcal B(X)$ is the family of concepts under study. We use Borel probability measures $\mu,\nu$ on $\mathbb R^n$, with $\mu(X)=1$. The measure $\mu$ evaluates alignment on the data domain, whereas $\nu$ evaluates alignment on the ambient space and may assign mass outside $X$.

\textbf{SAE neurons.} Write $[d]=\{1,\ldots,d\}$. For neuron $i\in[d]$, define its positive- and nonpositive-score regions, using threshold zero, by
\[
 H_i^+=\{x\in X:z_i(x)>0\},\qquad H_i^-=\{x\in X:z_i(x)\le0\}.
\]
Its actual \emph{activation region} is $N_i=\{x\in X:a_i(x)>0\}$, with $N_i^+=N_i$ and $N_i^-=X\setminus N_i$. For ReLU, $N_i=H_i^+$; for Top-$K$, a positive score may not enter the top-$k$ set, so $N_i\subseteq H_i^+$.

\textbf{SAE activations.} An \emph{activation pattern} $s$ assigns a sign $\sigma_{s;i}\in\{+,-\}$ to every neuron, recording whether it is active. Its \emph{total neuron single activation} (TNSA) region is $R_s=\bigcap_{i\in[d]}N_i^{\sigma_{s;i}}$. Let $\mathcal A$ be the collection of all patterns, including those with empty regions. The \emph{single-neuron total activation} (SNTA) is $N_i=\bigcup_{s\in\mathcal A:\,\sigma_{s;i}=+}R_s$: all inputs that activate neuron $i$, regardless of the other neurons.

For a neuron set $M\subseteq[d]$, define the \emph{multi-neuron activation}, or \emph{unit}, as $\theta_M=\bigcap_{i\in M}N_i$. It requires all selected neurons to activate, but leaves the others unrestricted. Write $\Theta=\{\theta_M:M\subseteq[d]\}$ for the model-learned concepts and set $\theta_\emptyset=X$. A single neuron $N_i$ can cover several unrelated concepts, leading to \emph{polysemanticity}, while $R_s$ can be too small because it specifies every neuron's activity. By combining selected neurons, $\theta_M$ provides an intermediate level for concept learning and interpretation. Different concepts may use different numbers $|M|$ of neurons.

A tilde denotes the same activation evaluated on all of $\mathbb R^n$, rather than only $X$. Thus $N_i=X\cap\widetilde N_i$, $\theta_M=X\cap\widetilde\theta_M$, and $\widetilde\theta_\emptyset=\mathbb R^n$; the same convention applies to $H_i^\pm$.

\subsection{Sparse Autoencoder Architectures}\label{section:sae_arch}
We categorize SAE architectures into \emph{absolute  activation}, where each neuron's activity is determined independently of other neurons, and \emph{relative  activation}, where it also depends on other neurons. ReLU is absolute  activation: for $w_i\ne0$, $\widetilde N_i$ is an open halfspace, one side of a flat boundary with the boundary excluded; for $w_i=0$, it is empty or all of $\mathbb R^n$. A ReLU unit intersects these regions, so its ambient region is \emph{convex}: it contains the line segment joining any two of its points. Its restriction to $X$ need not be convex.

Top-$K$ is relative  activation: $N_i=\{x\in X:z_i(x)>0,\ i\in\operatorname{TopK}_k(z(x))\}$, where $\operatorname{TopK}_k$ returns the selected indices. Its activation region can consist of several pieces inside the positive-score halfspace and need not be convex~\citep{projectingassumption}. Therefore, our geometric results focus on ReLU. Appendix~\ref{app:topk} analyzes \emph{rank interference}: the removal of otherwise positive activations because competing neurons have higher-ranked scores.

\subsection{Concept Learning}\label{sec:conceptlearning}
For simplicity, we write $\theta$ for $\theta_M$. \textbf{The goal of concept learning is to align human concepts $\mathcal C$ with model-learned concepts $\Theta$.} We consider three modes: concept detection, concept separation, and concept approximation. For concept representations we require $M\ne\emptyset$; otherwise $\theta_\emptyset=X$ would make coverage automatic. Unless otherwise specified, we study ReLU SAE in the main text.

\subsubsection{Concept Detection}
Concept detection asks whether a unit covers the target concept. 
\begin{definition}[Concept detection]\label{def:concept_detection}
 $C$ is detected if some nonempty $M$ satisfies $\mu(C\setminus\theta_M)=0$.
\end{definition}
This allows missed points of zero probability under $\mu$. Detection alone permits many-to-many mappings: one concept may be covered by several units, and one unit may cover several concepts.

\subsubsection{Concept Separation}
\begin{definition}[Concept separation]\label{def:concept_separation}
A concept $C$ is said to be separated by $\theta_M$ if
(i) $x\in N_i$ for all $x\in C$ and $i\in M$, and
(ii) for every $x'\in X\setminus C$, $x'\in N_j^-$ for some $j\in M$.
\end{definition}
Equivalently, every concept point activates all selected neurons, and every non-concept point fails to activate at least one. Imperfect separation can be measured by $e_{\mathrm{sep}}(C,\theta)=\mu(C\Delta\theta)$, where $C\Delta\theta=(C\setminus\theta)\cup(\theta\setminus C)$ is the \emph{symmetric difference}: missed concept points together with incorrectly included points. Appendix~\ref{app:separation_error} gives details. Separation evaluates only $X$; the same unit may still include large blank regions outside it. This limits what separation alone tells us about previously unseen concepts.

\subsubsection{Concept Approximation}
\begin{definition}[Concept approximation]\label{def:concept_approximation}
For $\varepsilon\ge 0$, a concept $C$ is said to be $\varepsilon$-approximated by $\theta_M$ if the total $\nu$-measure of points violating the following
conditions is at most $\varepsilon$:
(i) $x\in \widetilde N_i$ for all $x\in C$ and $i\in M$, and
(ii) for every $x'\in\mathbb R^n\setminus C$,
$x'\in \widetilde N_j^-$ for some $j\in M$.
Equivalently, $e_{app}(\Delta, \theta_M)\coloneqq\nu(C\Delta\widetilde\theta_M)\le\varepsilon$.
\end{definition}
The key difference between concept separation and concept approximation is the choice of space: the former is evaluated on $X$, whereas the latter is evaluated in $\mathbb{R}^n$. Intuitively, approximation asks the unit to fit the concept more tightly, penalizing included blank regions when $\nu$ gives them positive probability. This is relevant to concept discovery and anomaly detection~\citep{anomalydetection}. The measures may differ, so small approximation error does not automatically imply small separation error (Appendix~\ref{app:separation_error}).

\subsection{Neuron Interpretation}\label{mainframework:neuron_interp}
Concept learning and neuron interpretation are two directions of the same relation between concepts and units. Concept learning starts with a human concept and seeks a matching unit; neuron interpretation starts with a unit and seeks a human concept that describes its activating inputs. A neuron may cover several concepts, and a concept may be covered by several neurons. Section~\ref{sec:neuron_interp} uses formal concept analysis (FCA)~\citep{conceptlattice} to organize these relations by the neurons that inputs activate in common. Appendix~\ref{app:sae_phenomena} describes common SAE phenomena in this framework.

\section{Theoretical Results}\label{section:theory}
Here we present some theoretical results, which are existence results that allow the affine ReLU encoder parameters to be chosen freely and with enough neurons for each construction. They characterize representability, not whether training recovers the required neurons. Arbitrarily accurate approximation may require increasing width. 

\begin{assumption}\label{assumption}
The concept family $\mathcal C$ is finite and nonempty; each $C\in\mathcal C$ is nonempty and compact; and $X=\bigcup_{C\in\mathcal C}C$.
\end{assumption}
This includes finite samples and compact continuous concept regions, without requiring concepts to be convex or disjoint. It makes $X$ compact, but $X\setminus C$ need not be closed.

\subsection{Concept Separation}\label{sec:theory_concept_separation}
Write $D_C=X\setminus C$ for the non-concept region, $\operatorname{Conv}$ for convex hull, and $\overline A$ for ambient closure. Proofs are in Appendix~\ref{app:concept_separation}.

\begin{theorem}[Concept separation with one neuron]\label{theorem:neuron_separation_condition}
A concept $C$ can be separated from $D_C$ by one neuron if and only if
\[
 \operatorname{Conv}(C)\cap\overline{\operatorname{Conv}(D_C)}=\varnothing.
\]
\end{theorem}
Intuitively, one affine boundary must place the entire target on one side and all non-target points on the other; taking convex hulls preserves these constraints.

\begin{corollary}\label{corollary:neuron_separation}
All concepts can be separated by one neuron per concept if and only if the condition in Theorem~\ref{theorem:neuron_separation_condition} holds for every $C\in\mathcal C$. When it holds, the minimum number of representing neurons is $|\mathcal C|$.
\end{corollary}
A neuron has only one activation region, so it cannot exactly represent two distinct concept sets.

\begin{theorem}[Concept separation with multiple neurons]\label{theorem:unit_separation_condition}
A concept $C$ can be separated from $D_C$ by finite unit if and only if
\[
 \operatorname{Conv}(C)\cap\overline{D_C}=\varnothing.
\]
\end{theorem}
{\textit{Remark}: $\overline{D_C}$ is compact}. To better understand, different neurons can exclude different non-target regions. However, any convex unit covering $C$ must also cover its convex hull. The closure matters because a finite ReLU unit is open and cannot exclude non-target points that accumulate inside it.

\begin{corollary}\label{corollary:unit_separation}
All concepts can be separated by a finite units if and only if
\[
 \operatorname{Conv}(C_i)\cap\overline{(C_j\setminus C_i)}=\varnothing
 \quad\text{for all distinct }C_i,C_j\in\mathcal C.
\]
\end{corollary}
Each target concept's convex hull must avoid the non-target portion of every other concept, including its limit points.

These results show that combining units can turn individually polysemantic neurons into a more selective unit. Increasing width provides more candidate boundaries, but the corollary also reveals a geometric limit: some concepts remain inseparable even with finitely many ReLU units. Thus, feature splitting can mitigate polysemanticity, but is not a universal solution, especially for hierarchical concepts that require separating a child from the remainder of its parent.

\subsection{Concept Approximation}\label{sec:theory_concept_approximation}
We now evaluate the ambient unit under $e_{\mathrm{app}}$, rather than requiring exact equality on $X$.
\begin{theorem}[Concept approximation condition]\label{theorem:approximation_condition}
A concept $C$ can be arbitrarily well approximated by a finite unit if and only if there exists a Borel convex set $K\subseteq\mathbb R^n$ such that $\nu(C\Delta K)=0$.
\end{theorem}
ReLU units can approximate convex regions, but cannot remove nonconvexity that remains visible to the evaluation measure. The condition allows differences on sets of zero $\nu$-mass and holds even for atomic measures (Appendix~\ref{app:approximation_condition}).

\begin{proposition}[Concept approximation error rate]\label{theorem:approximation_rate}
Let $L\subseteq\mathbb R^n$ be an $r$-dimensional affine subspace containing $C$, with $r\ge2$. Suppose $\nu$ is supported on $L$ and has a continuous, strictly positive density relative to $r$-dimensional Lebesgue measure. Let $K\subseteq L$ be a convex body (a compact convex set with nonempty relative interior) with $C^2$ boundary and positive Gaussian curvature, relative to $L$. For all sufficiently large $m$, some ReLU unit with $|M|\le m$ satisfies
\[
 e_{\mathrm{app}}(C,\widetilde\theta_M)
 \le \nu(C\Delta K)+A_Km^{-2/(r-1)},
\]
where $A_K$ depends on $K$, the density, and $r$, but not on $m$.
\end{proposition}
Intuitively, each selected neuron supplies a flat boundary; more boundaries approximate a smooth convex region more closely, with a slower rate in higher dimension. This follows from weighted polytope approximation~\citep{convexbodyapprox2}; the proof and the interval case $r=1$ are in Appendix~\ref{app:approximation_rate}.

The irreducible error is $e_{\mathrm{irr}}:=\inf_{K\text{ Borel convex}}\nu(C\Delta K)$, equal to the infimum over finite ReLU units. It need not equal $\nu(\operatorname{Conv}(C)\setminus C)$: a better convex region may omit part of $C$. The displayed rate applies to a chosen smooth $K$, and has residual $e_{\mathrm{irr}}$ only if an optimal convex region meets the same regularity conditions. For $K=C$, the residual vanishes. Separately constructed units for multiple concepts use at most the sum of their sizes; neuron reuse can reduce this construction budget.

\subsection{Concept Learning Capacity}\label{sec:theory_capacity}
Geometry constrains which concepts are representable; width also limits the number of candidate units. Let $k_c\ge1$ be an integer budget on the number of neurons selected per concept, distinct from the mean activation sparsity $L_0$.
\begin{theorem}[Concept learning capacity]\label{theorem:capacity}
Suppose distinct concepts are assigned distinct nonempty neuron sets $M_C\subseteq[d]$, with $|M_C|\le k_c$. Then $|\mathcal C|\le\sum_{j=1}^{\min(k_c,d)}\binom dj$. When $k_c/d\to0$, this implies
$d\ge(k_c!\,|\mathcal C|)^{1/k_c}$.
\end{theorem}
Intuitively, combining neurons creates more candidate descriptions than using one neuron per concept. Exact separation forces this injective assignment, but counting subsets does not guarantee distinct activation regions or monosemantic individual neurons. Top-$K$ additionally requires $|M_C|\le k$ for a nonempty region: a larger dictionary cannot replace the need for the selected neurons to coactivate. The proof is in Appendix~\ref{app:capacity}.

\subsection{Concept Learning and Neuron Interpretation}\label{sec:neuron_interp}
For a fixed SAE, the incidence relation $x\mathcal I i\iff x\in N_i$ defines a formal context with data points as objects and neurons as attributes~\citep{conceptlattice}. For $A\subseteq X$ and $M\subseteq[d]$, define
\[
 f(A)=\{i\in[d]:A\subseteq N_i\},\qquad
 g(M)=\bigcap_{i\in M}N_i=\theta_M,
 \quad g(\emptyset)=X.
\]
Thus $f$ collects all neurons covering a region, while $g$ returns their common activation region.
\begin{theorem}[Galois connection]\label{theorem:galois_connection}
The maps $f:\mathcal P(X)\to\mathcal P([d])$ and $g:\mathcal P([d])\to\mathcal P(X)$ form an inclusion-reversing Galois connection:
\[
 A\subseteq g(M)\iff M\subseteq f(A).
\]
\end{theorem}
Intuitively, a region lies inside a unit exactly when every selected neuron covers the whole region. Enlarging the region can only remove common neurons; adding neurons can only shrink the unit.

The composition $g\circ f$ is a closure operator: $E_C:=g(f(C))$ is the smallest available unit region containing $C$. An exact nonempty-unit representation exists precisely when $E_C=C$ and $f(C)\ne\emptyset$. This is the fixed-dictionary counterpart of the existence results: they ask whether suitable neurons can exist, whereas closure tests what the available neurons distinguish. A \textit{formal concept} is a pair $(E,M)$ with $E=g(M)$ and $M=f(E)$, called its \textit{extent} and \textit{intent}. Ordered by extent inclusion, these pairs form a concept lattice. Here ``closed'' means unchanged by $g\circ f$, not topologically closed. Full details of lattice operations are in Appendix~\ref{app:neuron_interp}.

Different human concepts may share one closed region, and inclusion between regions reverses inclusion between their intents. The lattice therefore records sharing and hierarchy induced by the learned activation geometry, not necessarily by human semantics. Full-coverage $f$ is distinct from selecting a best-F1 unit, and the label-closedness used in our experiments tests whole-label inclusion rather than equality $E_C=C$. Together, the geometric and structural views connect the encoder's activation regions, the arrangement of data concepts, their representation by intersections, and the shared or hierarchical interpretations of those representations.

\section{Experiments}
Our experiments ask how trained SAEs realize the concept regions characterized by our theory. We first study how data geometry, gating, and sparsity affect concept learning, then use FCA to connect these effects to concept learning and neuron interpretation.

\subsection{Setup}
\textbf{Data and models.}
We use two-dimensional synthetic concepts represented by disjoint or partially overlapping clusters. For real-world data, we use the ten Imagenette classes with 768-dimensional \mbox{DINOv3} embeddings~\citep{Howard_Imagenette_2019,dinov3}, and Yelp review sentences with five star-rating labels and 1536-dimensional \texttt{text-embedding-3-small} embeddings~\citep{yelp,openai_text_embedding}. We vary expansion factor (EF) and sparsity in ReLU and Top-$K$ SAEs. Larger $L_0$, the mean number of active neurons per input, means denser activation. See Appendix~\ref{app:exp_detail} for training and aggregation details.

\textbf{Concept learning.}
Following \citet{projectingassumption}, we measure alignment with F1. We rank neurons by individual F1 and intersect the top $N$ activation regions: membership requires all selected neurons to activate. We sweep the permitted unit sizes, capped at 20 for real-world data. We report test score $F_1(N)$, the smallest best-scoring size $N^*$ within the sweep, and gain $G=F_1(N^*)-F_1(1)$.

\textbf{Concept lattice.}
We connect these findings to the concept lattice (Section~\ref{sec:neuron_interp}) using the overlapping synthetic dataset. For each concept $C$, its intent $M_C=f(C)$ contains all neurons active on every training example of $C$. For nonempty intents, we map $M_C$ back to $E_C=g(M_C)$, the region where all neurons in $M_C$ are jointly active, and call it the \emph{description} of $C$. We evaluate $E_C$ on test points and record which whole concepts it contains. Concepts are \emph{merged} when their descriptions coincide; $C$ is a \emph{parent} if $E_C$ contains another whole concept, and a \emph{child} if $C$ lies inside another concept's description while $E_C$ contains none. A concept is \emph{label-closed} if its description contains no other whole concept; this includes children and \emph{isolated} concepts, which no other description contains. Label-closedness need not mean $E_C=C$: partial overlap can remain. An \emph{empty} intent means that no neuron is active on all of $C$, not that the SAE contains no information about it.

\subsection{Synthetic data: concept separation}\label{sec:exp_synthetic}
\begin{figure}[!t]
\centering
\vspace{-33pt}
\includegraphics[width=\linewidth]{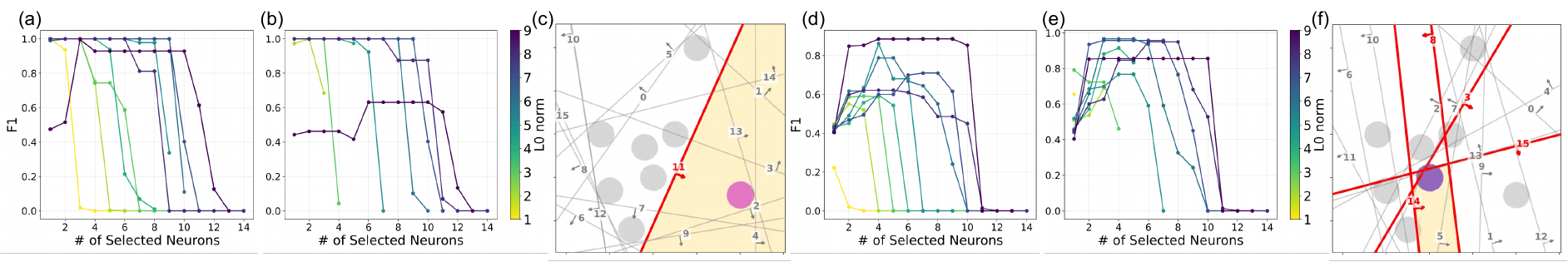}
\vspace{-20pt}
\caption{Concept learning on disjoint synthetic data at EF$=8$. \textbf{(a,b)} F1 curves for an easy-to-separate concept using ReLU and Top-$K$ SAEs; \textbf{(c)} its ReLU activation geometry. \textbf{(d--f)} Corresponding results for a harder, surrounded concept.}
\label{fig: main_exp_disjoint}
\vspace{-15pt}
\end{figure}
Figure~\ref{fig: main_exp_disjoint} contrasts a peripheral concept separable by one halfspace with a surrounded concept requiring several boundaries. For ReLU, this illustrates Theorems~\ref{theorem:neuron_separation_condition} and~\ref{theorem:unit_separation_condition}: individually nonselective halfspaces can jointly exclude different neighbours. For the harder concept, ReLU F1 at $L_0=9$ rises from $0.405$ with one neuron to $0.887$ with seven ($119\%$ relative gain). Top-$K$ F1 at $L_0=7$ rises from $0.441$ to $0.958$ with three; its best score across sparsities is $0.966$.

One neuron already reaches F1$=1$ for the easier concept, but combinations can compensate for a poorly aligned dictionary: at $L_0=9$, three ReLU neurons improve F1 from $0.475$ to $0.996$. Thus, needing multiple neurons can reflect either concept geometry or the learned dictionary. A polysemantic neuron may still supply an essential boundary of a selective unit: useful concept representations need not coincide with individual dictionary entries. Additional details are in Appendix.~\ref{app:exp}.

\subsection{Real-world data: concept separation}
\begin{figure}[t]
\centering
\vspace{-33pt}
\includegraphics[width=0.99\linewidth]{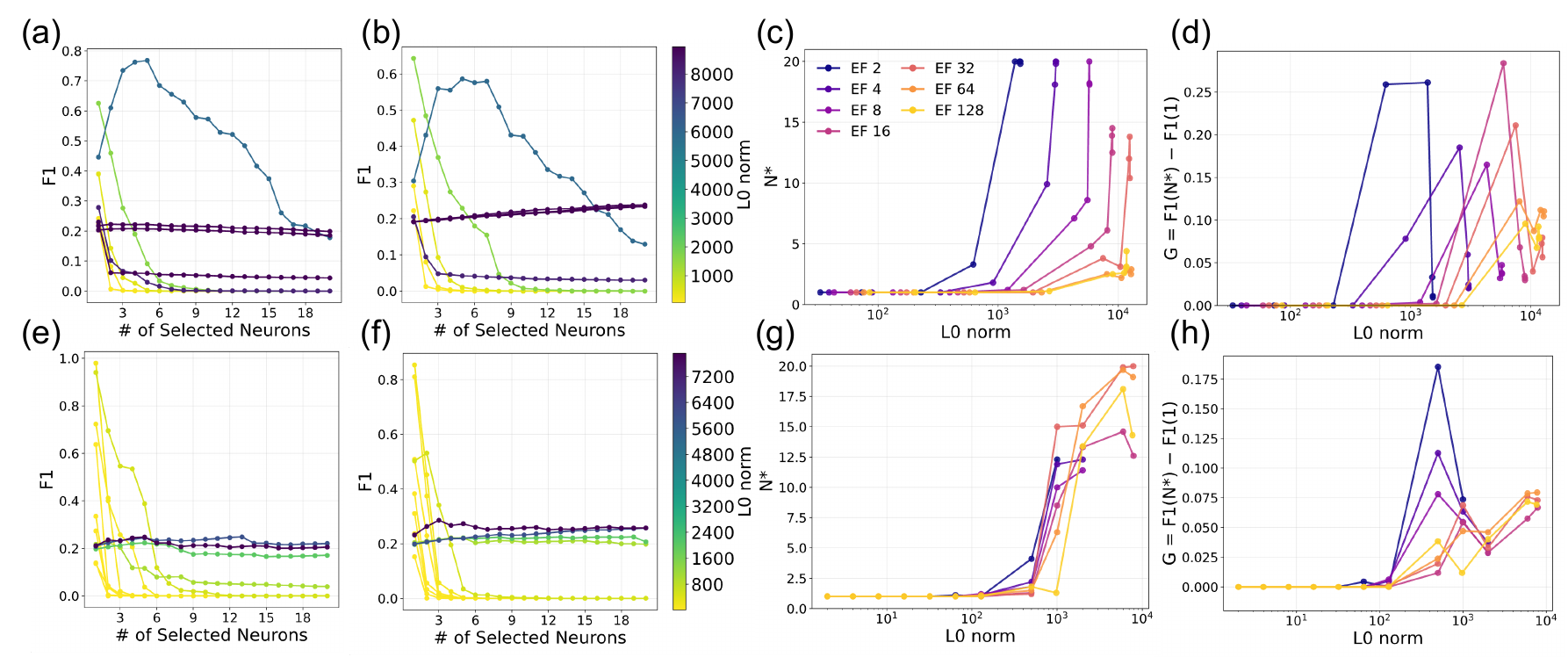}
\vspace{-13pt}
\caption{Concept learning on Imagenette. \textbf{(a,b)} ReLU F1 curves for tench and chain saw at EF$=16$. \textbf{(c,d)} Mean $N^*$ and $G$ across concepts, varying $L_0$ and EF. \textbf{(e--h)} Corresponding Top-$K$ results.}
\label{fig: main_exp_image}
\vspace{4pt}
\includegraphics[width=\linewidth]{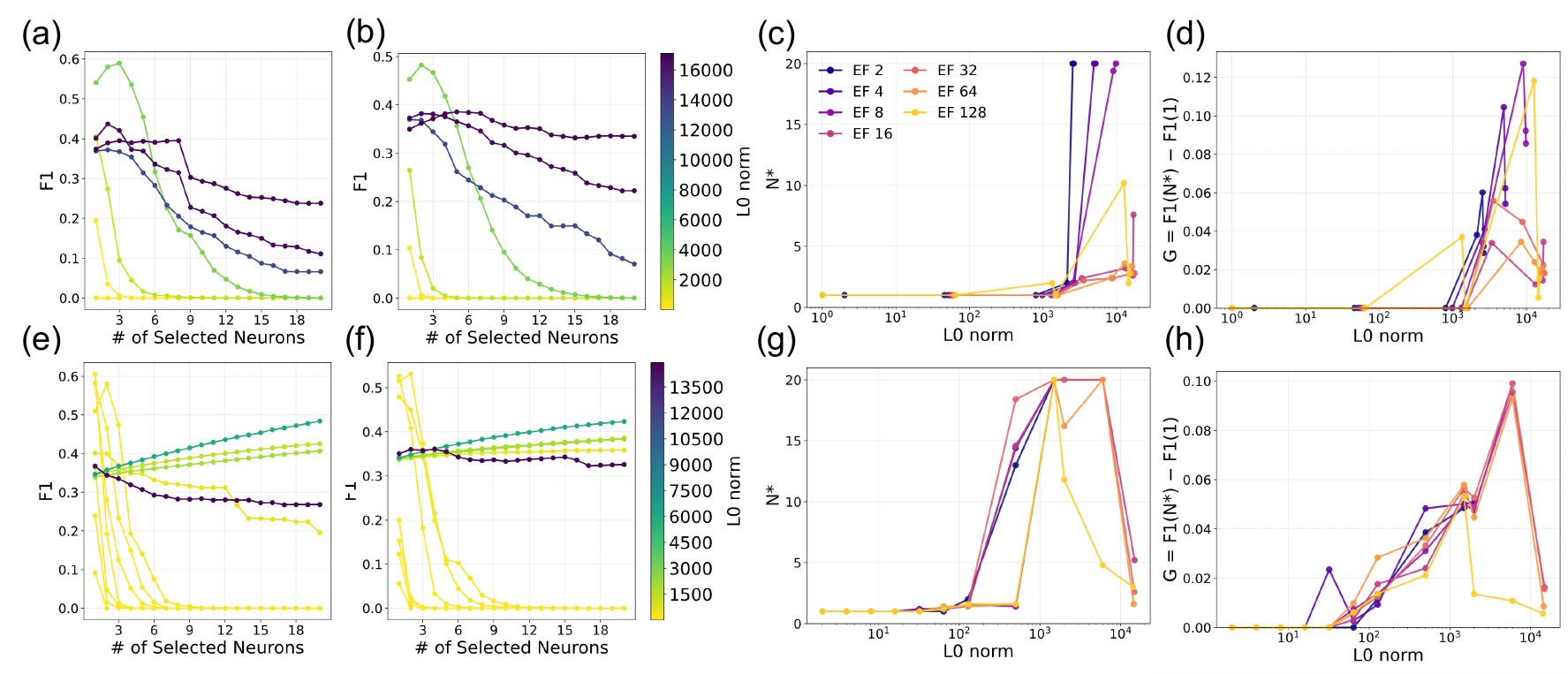}
\vspace{-16pt}
\caption{Concept learning on Yelp, with the same organization as Figure~\ref{fig: main_exp_image}. The case studies use 1-star and 4-star sentences.}
\label{fig: main_exp_text}
\end{figure}
Figures~\ref{fig: main_exp_image} and~\ref{fig: main_exp_text} extend the analysis to image and text embeddings. Guided by the one-neuron separation criterion, we select easier concepts (tench and 1-star) and harder concepts (chain saw and 4-star) by their lowest and highest one-vs-rest linear SVM errors.

The curves typically fall in sparse SAEs, become hump-shaped at intermediate sparsity, and rise more slowly in denser SAEs. Within a fixed dictionary, each additional neuron can remove both contamination and target examples. Sparse features may cover different portions of a concept, making their intersection lose recall; at intermediate sparsity, several boundaries can remove contamination before recall deteriorates; broad, strongly coactive features may require many intersections to become selective. For example, five ReLU neurons improve tench F1 from $0.446$ to $0.768$, while three improve Yelp 1-star F1 from $0.541$ to $0.590$ at EF$=16$.

Across concepts, the mean within-EF Spearman correlation $\rho_s$ between $L_0$ and $N^*$ is $0.76$--$0.82$. At comparable $L_0$, wider ReLU SAEs tend to need fewer neurons ($\rho_s=-0.68$ on Imagenette and $-0.40$ on Yelp). For Top-$K$, the mean bias-corrected $\eta^2$ values for $N^*$, $G$, and curve shape are $0.47$--$0.81$ when grouped by $L_0$, but near or below zero when grouped by EF. ReLU transitions also shift with width, suggesting that active fraction $L_0/d$ is useful alongside $L_0$ when comparing dictionaries. These results distinguish combinatorial capacity (Theorem~\ref{theorem:capacity}) from usable concept representations: width supplies candidates, but their geometry and coactivation determine whether they combine usefully. Additional details including the full statistics table are in Appendix.~\ref{app:exp}.

Importantly, the largest gain need not produce the highest final F1. Across our sweep, $13/15$ concepts attain their best ReLU score with multiple neurons, whereas $12/15$ attain their best score across both architectures with a single Top-$K$ neuron. This is consistent with our relative-gating analysis (Appendix.~\ref{app:topk}): one Top-$K$ activation region can exploit boundaries imposed by competing neurons. A single-neuron interpretation therefore need not imply geometrically simple or independent computation. Architecture and sparsity should be chosen jointly with interpretation granularity, evaluating final F1 for both single neurons and units, not just $G$.

\subsection{Concept learning on unseen concepts: concept approximation}\label{sec:exp_unseen}
\begin{figure}[t]
\centering
\vspace{-10pt}
\includegraphics[width=\linewidth]{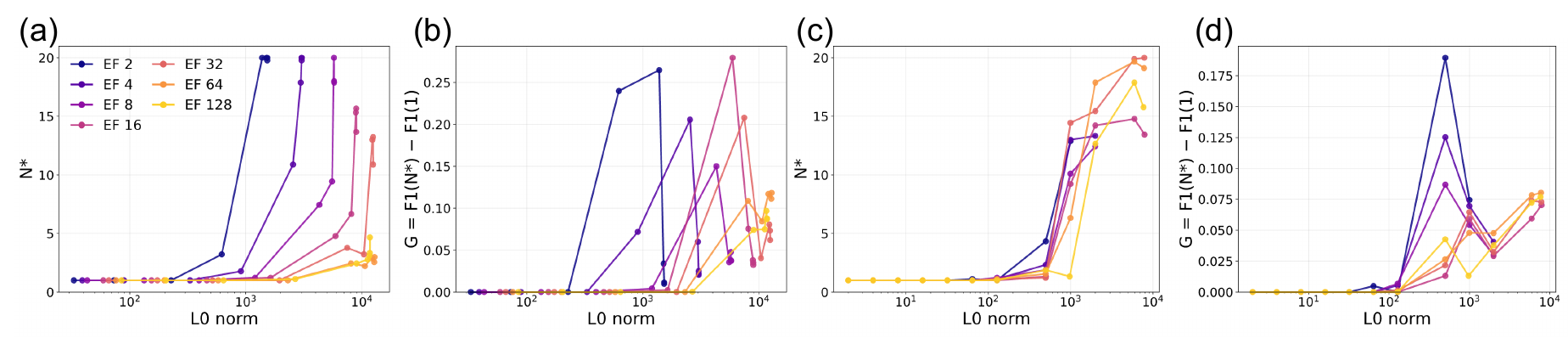}
\vspace{-20pt}
\caption{Concept learning on Imagenette without tench. (ab) $N^*$ and $G$ for ReLU SAEs aggregated over other concepts. (cd) similar to (ab) but use Top-$K$ SAEs.}
\vspace{-20pt}
\label{fig:main_exp_ad_without_tench}
\end{figure}
{
To probe beyond the training support (Section~\ref{sec:theory_concept_approximation}), we delay full analysis of concept approximation to Appendix.~\ref{app:exp}. In summary, SAE features exclude most unseen negatives, yet stronger exclusion can degrade seen-concept detection, demonstrating that OOD exclusion does not guarantee OOD robustness. Conversely, a held-out concept remains partially detectable through combinations of previously learned features, revealing \textbf{compositional generalization} and suggesting potential for \textbf{novel concept discovery}; together, these findings highlight the role of learned feature geometry in unseen-data behavior and motivate practitioners to evaluate OOD exclusion, seen-concept retention, and novel-concept coverage separately. 
}

\subsection{FCA: from activation geometry to concept structure}\label{sec:exp_fca}
\begin{wrapfigure}{r}{0.56\textwidth}
\centering
\vspace{-15pt}
\includegraphics[width=\linewidth]{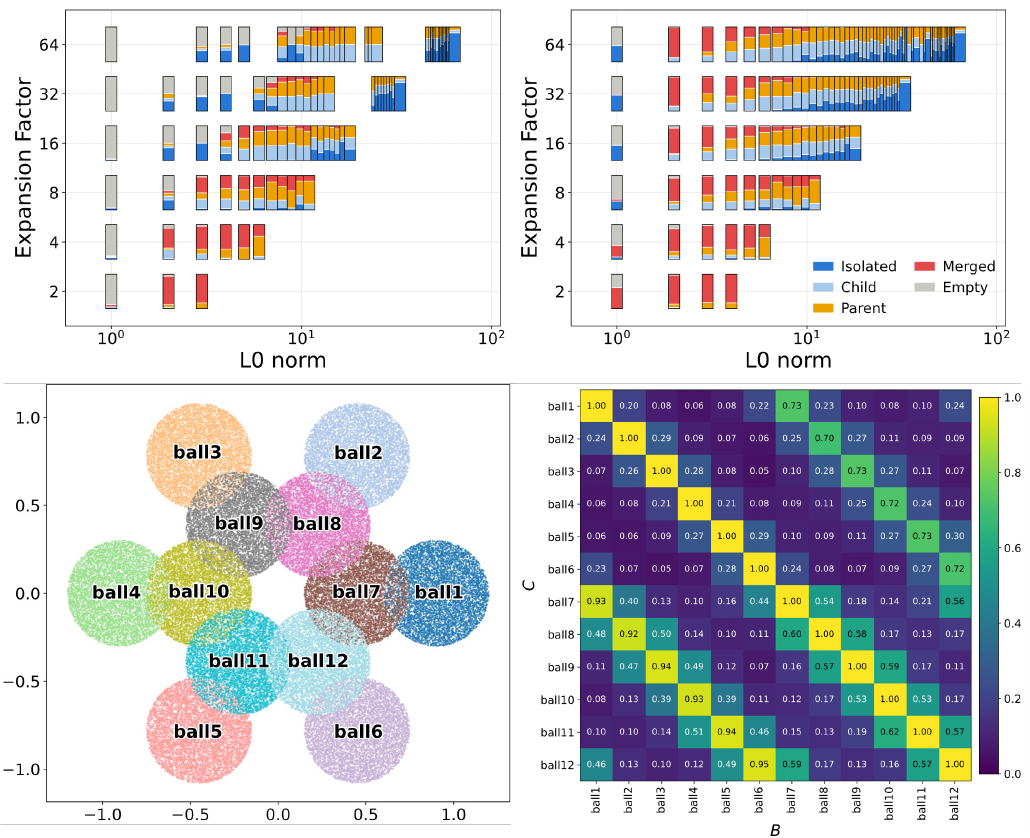}
\vspace{-20pt}
\caption{FCA on overlapping concepts. Top: states for ReLU (left) and Top-$K$ (right). Bottom: data geometry and mean directional inclusion $\ell(C,B)$ over SAEs with nonempty intents.}
\vspace{-15pt}
\label{fig:main_exp_fca}
\end{wrapfigure}

\textbf{Concept descriptions sharpen as $L_0$ grows.} Descriptions move from empty or merged at low $L_0$ to closed at higher $L_0$ (Figure~\ref{fig:main_exp_fca}): near $L_0=2$, $59\%$ (ReLU) and $76\%$ (Top-$K$) are merged; near $L_0=64$, $75\%$ and $73\%$ are closed. Denser coactivation thus yields finer \emph{joint} descriptions even when single neurons stay broad.

\textbf{The hierarchy comes from geometry, not labels.} The labels contain no parent--child relations, yet a hierarchy emerges. Let $\ell(C,B)=|B\cap E_C|/|B|$. Inner concepts' regions cover $92$--$95\%$ of their outward radial neighbours, versus $70$--$73\%$ in reverse, even in dense SAEs, suggesting a directional bias in the learned gates. Apparent hierarchies or feature families can thus reflect gate geometry rather than semantics.

\textbf{FCA bridges concept-learning F1 and neuron interpretation.} For a concept $C$ with nonempty intent, let $B^\star=\arg\max_{B\neq C}\ell(C,B)$ be the other concept with the largest fraction inside $E_C$; this single quantity predicts both directions. \emph{Concept learning:} larger $\ell(C,B^\star)$ goes with lower $F_1(N^*)$ ($\rho_s\approx-0.74$ for both architectures). \emph{Neuron interpretation:} labelling the unit selected for $C$ by its best-F1 concept disagrees with $C$ for $83\%$ (ReLU) and $80\%$ (Top-$K$) of parents but only $2.5\%$ and $4.8\%$ of children, and the label is then $B^\star$ in $70$--$74\%$ of cases. Low F1 and misinterpretation are thus two symptoms of one FCA structure: the lattice shows when a learned unit will be misinterpreted and as which concept.

{
\subsection{FCA on an LLM: human-defined and model-induced concepts}\label{sec:exp_fca_llm}
We apply FCA to 36 SAEBench SAEs~\citep{saebench} of Gemma-2-2B on 168,964 tokens from the first-letter task of \citet{featureabsorption}, whose labels are exactly nested (st $\subset$ s). The node $\theta_C$ is the region of the smallest neuron firing on $70\%$ of $C$ (Appendix~\ref{app:fca_llm}).

\begin{figure}[t]
\centering
\vspace{-33pt}
\includegraphics[width=\linewidth]{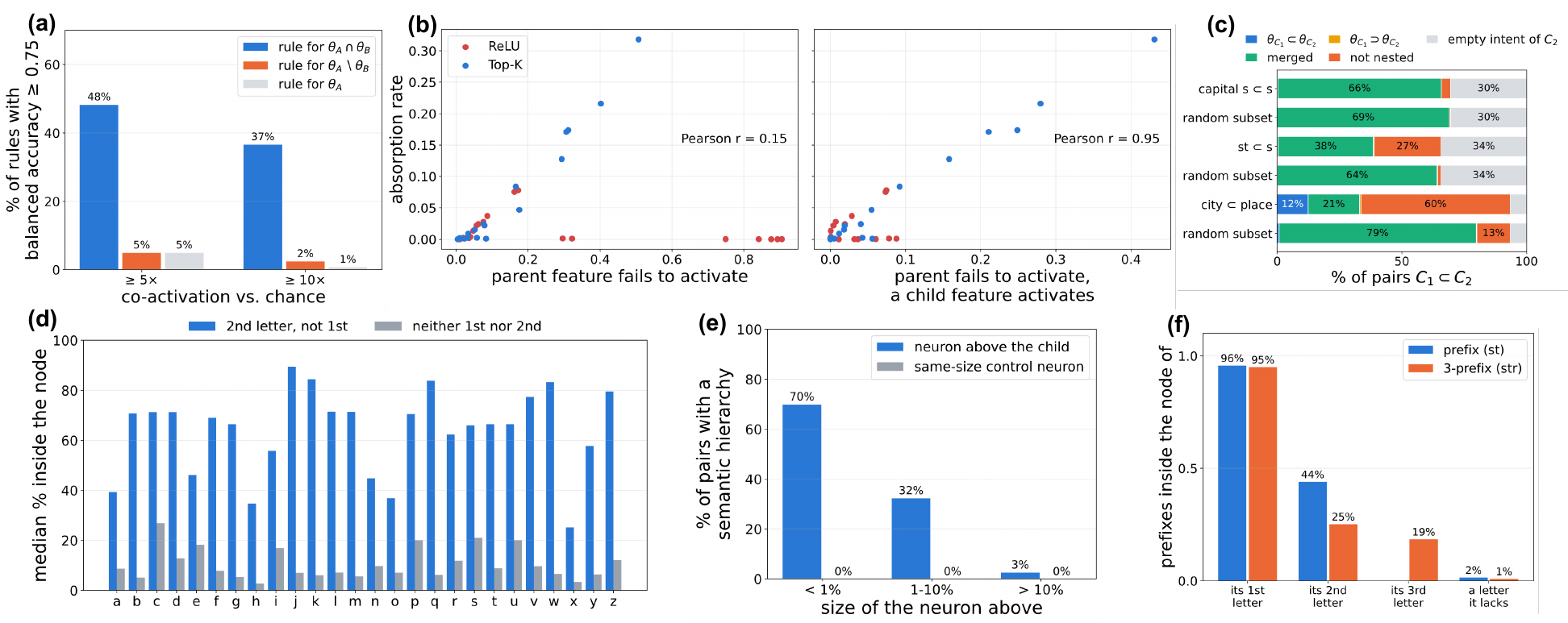}
\vspace{-18pt}
\caption{{Gemma-2-2B SAEs. \textbf{(a)} Rules with F1 $>0.7$ for $\theta_A\cap\theta_B$, $\theta_A\setminus\theta_B$, $\theta_A$. \textbf{(b)} Absorption rate vs.\ miss rate and vs.\ ``missed and taken''. \textbf{(c)} Nested pairs vs.\ random subsets. \textbf{(d)} Second-place vs.\ other tokens in nodes. \textbf{(e)} Hypernyms by upper-neuron size. \textbf{(f)} Prefixes in their letters' nodes.}}
\label{fig:main_fca_llm}
\vspace{-14pt}
\end{figure}

\textbf{Human-defined concept mismatch.} A letter's node rarely matches the letter: only $5\%$ of (SAE, letter) pairs have $\theta_C\approx C$, although $61\%$ are label-closed. Most nodes are wider: the node of b also contains \texttt{getBean} and \texttt{listBox}. Nodes favour tokens with the letter in second place (Figure~\ref{fig:main_fca_llm}d), yet no description captures them: a rule for b scores $0.93$ on held-out tokens but only $0.55$ F1 on the whole vocabulary. The SAE thus learns a nearby concept, roughly ``the letter near the start'', which label-closed rates and held-out scores hide.

\textbf{Model-induced concept discovery.} We call the smallest neuron containing $70\%$ of a neuron's tokens its upper neuron. When it is specific (under $1\%$ of tokens), $70\%$ of such pairs are hypernyms in a held-out test (e.g., Reuters, CNN under news media), giving a label-free hierarchy, but $87\%$ of children lie under a wide one (Figure~\ref{fig:main_fca_llm}e). FCA meets can also create new concepts (\textbf{\textit{combinatorial expressivity}}): neurons may carry \textbf{\textit{implicit semantic fragments}} that a meet makes explicit, yielding knowledge neither neuron shows. Such meets are rarely existing neurons, yet a rule fits $31\%$ of them well, against $1$--$4\%$ for $\theta_A\setminus\theta_B$ (Figure~\ref{fig:main_fca_llm}a). Some are true conjunctions (ancient names $\cap$ outdated words $=$ ``ancient, fossil''), some are novel (forms of \emph{add} $\cap$ sorting verbs $=$ \emph{add, insert, delete}; F1 $0.94$), and some mean nothing despite $18\times$ co-activation. This fits Theorem~\ref{theorem:unit_separation_condition} and may explain how the unseen tench in Section~\ref{sec:exp_unseen} emerges from fragments in other concepts' neurons.

\textbf{Structure-preserving mapping.} We test if $C\mapsto\theta_C$ is close to an isomorphism between human and model semantics, keeping order, meets and round trips. Order holds only weakly: $78$--$98\%$ of children still lie in the parent's node, yet only $0.8\%$ of 58,104 nested pairs keep $\theta_{C_1}\subset\theta_{C_2}$; most are merged ($39\%$, e.g., wa and w) or not nested ($29\%$). Not nesting is real ($27\%$ vs.\ $1.4\%$ for random sets; Figure~\ref{fig:main_fca_llm}c) and follows from node meaning: a prefix lies in the nodes of its first and second letters (Figure~\ref{fig:main_fca_llm}f) and takes the smaller one, so tw $\subset$ t ends up under w. Semantic hierarchies break alike: weekdays have their own neurons, but outside the node of ``time'' ($60\%$ not nested vs.\ $13\%$). Meets and round trips also fail: a human conjunction (leading space $\wedge$ capital) has no single neuron, at best the meet of two ($36\%$), and $0.3\%$ of concepts map to a node and back to themselves. Conversely, $70\%$ of pairs with a specific upper neuron are hypernyms. The SAE thus re-organizes rather than copies the human tree, supporting lattice- or graph-structured SAEs (Appendix~\ref{app:neuron_interp}).

\textbf{Feature absorption.} As a case study, we revisit absorption~\citep{featureabsorption}, usually seen as a child absorbing its parent. In Top-$K$ 4k ($L_0{=}40$), the node of s is its parent neuron, roughly ``s near the start'', and misses 682 s tokens; 306 go to an ``st''-fragment neuron that also covers \texttt{estanding}. In the lattice, $99.99\%$ of absorbing neurons are incomparable with their parent, with half of their tokens outside the letter. Absorption can also be spread out: absorbing neurons are mostly words or stems, often with translations (\emph{home, hogar, hjem}); each takes about $1\%$ of the parent's misses, together $50$--$71\%$. So absorption can be an overlap of model concepts, not a hierarchy, and one stage of a single structure: as $L_0$ grows, a parent--child pair passes from absorption to nesting to polysemanticity ($P(\text{parent}\mid\text{child})$: $0.41\to0.91\to0.99$). It is measurable without a probe ($r=0.975$) and tracked by misses another neuron takes ($r=0.95$), not by all misses ($r=0.15$; Figure~\ref{fig:main_fca_llm}b). In Top-$K$, the parent's gate is mostly open but the top $K$ outrank it (Appendix~\ref{app:fca_llm_encoder}).
}

Taken together, the experiments suggest viewing an SAE as a system of reusable activation constraints, rather than a one-to-one dictionary of human concepts. Data geometry and SAE hyperparameters shape those constraints; their intersections determine concept learning; and the FCA lattice organizes the resulting shared, hierarchical, and distinct interpretations. {On a real LLM, it shows that SAE concepts lie near but rarely on human ones and form a lattice.} Evaluating this structure alongside F1 can reveal whether an apparently interpretable feature captures the intended concept or an architecture-dependent grouping of concepts.

\vspace{-5pt}
\section{Discussion}
\vspace{-10pt}
We present a set-theoretic framework for concept learning and neuron interpretation in SAEs. Geometrically, it characterizes when individual neurons or multi-neuron units can represent concepts; structurally, FCA organizes the many-to-many relation between concepts and activation regions. {Experiments show how data geometry and SAE architecture shape concept learning, why combining features helps alignment without a unique neuron interpretation, and that in a real LLM this many-to-many relation is the norm. Our analysis focuses on ReLU SAEs and a lexical LLM task; other architectures, training dynamics, and lattice- or graph-structured SAEs remain open.}

\bibliography{iclr2027_conference}
\bibliographystyle{iclr2027_conference}

\appendix
\section{Appendix}
\subsection{List of Contents}
We organize the appendix as follows.
\begin{itemize}
    \item \textbf{Related Work}~\ref{app:related_work}
    \item \textbf{Sparse Autoencoder Phenomena}~\ref{app:sae_phenomena} describes how to use our framework to explain and formulate several major SAE phenomena.
    \item \textbf{Experiment Details}~\ref{app:exp_detail} describes datasets and SAE training hyperparameters.
    \item \textbf{A Complete View of Concept Learning and Neuron Interpretation}~\ref{app:neuron_interp} includes proofs for Section~\ref{sec:neuron_interp}.
    \item \textbf{Proofs for Concept Separation}~\ref{app:concept_separation} contains proofs for Section~\ref{sec:theory_concept_separation}.
    \item \textbf{Proofs for Concept Approximation}~\ref{app:concept_approximation} contains proofs for Section~\ref{sec:theory_concept_approximation}.
    \item \textbf{Proofs for Concept Learning Capacity}~\ref{app:capacity} contains the proof for Section~\ref{sec:theory_capacity}.
    \item \textbf{Additional Discussion on Top-$K$ SAE}~\ref{app:topk} discusses the additional rank-interference effect induced by Top-$K$ gating.
    \item \textbf{Additional Experiments}~\ref{app:exp} contains more experiment results about concept learning on both synthetic and real-world datasets and unseen concept learning.
\item {\textbf{FCA on an LLM: Details}~\ref{app:fca_llm} gives the setup and all results of the FCA experiments on Gemma-2-2B in Section~\ref{sec:exp_fca_llm}.}
\end{itemize}

\subsection{Related Work}\label{app:related_work}
    \textbf{Mechanistic Interpretability and Sparse Autoencoders.}
    Mechanistic interpretability studies models' internal computations \citep{mechinterp1,mechinterp2,mechinterp3}. SAEs address superposition by learning sparse dictionaries that reconstruct activations and expose interpretable features \citep{superposition,dictionarylearning,andrewsae,originalsae,originalsae2,monosemantic}. Findings on random models and scaling motivate explicit criteria for concept alignment \citep{randomsae,featuremanifold}.
    
    \textbf{SAE Architectures and Phenomena.}
    Feature splitting, families, and absorption reveal relationships among SAE features \citep{everyphenomenon,featureabsorption}. Gated, JumpReLU, Top-$K$, matching-pursuit, ensemble, and hierarchical variants modify sparsity and feature structure \citep{gatedsae,jumprelusae,topksae,matchingpursuitsae,ensemblingsae,hierarchicalsae,matryoshkasae,hierarchyforest,treesae}. \citet{projectingassumption} relate architecture to receptive-field geometry and design geometry-adapted encoders. Building on this duality, we characterize when intersections of ReLU activation regions can separate concepts that no single ReLU neuron can, and distinguish separation on the data support from ambient approximation.
    
    \textbf{Concept Learning and Neuron Interpretation.}
    Interpretation infers descriptions from activations \citep{autointerp,blackboxneuroninterp,millionsautointerp}, whereas concept-based analysis identifies features for supplied concepts \citep{fromdatatoneuron,cbm,networkdissect,mu2020compositional}. Evaluating these associations \citep{sanitycheck,saebench,synthsaebench,saebenchreliability} motivates a common set-alignment formulation. Geometric accounts study directions and manifolds \citep{linearrepresentationhypothesis,categoricalgeometry,multidimensionalfeatures,llmactivationgeometry}. \citet{localgeometry} learn Gaussian-region decompositions; we analyze regions induced by SAE gates. In concurrent work, \citet{conceptmanifold} study manifold reconstruction by compact decoder groups, as well as tiling by features whose active subsets vary across samples. We instead analyze the target coverage and non-target exclusion of a fixed activation conjunction; reconstruction alone does not establish these membership properties.
    
    \emph{Generator cones and membership intersections differ.} \citet{geometricconcept} use nonnegative spans of dictionary directions and evaluate directional and activation agreement. Our units intersect encoder membership regions. For affine-threshold gates, their non-strict homogeneous ambient counterparts are positive duals of cones generated by encoder normals, which need not equal decoder atoms \citep{rockafellarconvexanalysis}. With biases, these gates yield affine polyhedra; relative gates need not yield halfspaces. This distinction motivates direct coverage, separation, and approximation criteria.
    
    \textbf{Formal Concept Analysis and Lattice Representations.}
    These many-to-many associations motivate formal concept analysis (FCA), whose extent--intent construction \citep{conceptlattice} has been applied to neural responses and language models \citep{endres2008neuralfca,conceptualviews,tokenslattices}. Concurrent work by \citet{latticerepresentation} derives lattices from thresholded attribute directions, overlapping with our affine-gating case. We adopt the same established algebra for SAE activation predicates, including relative gates. That work studies semantic hierarchy and concept composition; we study SAE membership alignment, its geometric feasibility, approximation error, and neuron budgets.
    
    {\textbf{Human and Model Concepts.}
    \citet{understandvocab} argue that humans and machines hold different concepts: some are shared, some are held only by the machine, and some only by humans. They propose neologisms, new words learned as token embeddings, to name concepts in either direction. Their concepts are described informally and their sets are sets of concepts, whereas in our framework each concept is itself a set of inputs, so the mismatch can be measured as set alignment and its geometry analyzed; closed lattice nodes without a human label (Section~\ref{sec:neuron_interp}) are natural candidates for such neologisms. Concurrently, \citet{cammarata2026clean} argues that neurons regarded as polysemantic often detect one coherent pattern, for example interpreting the cat-and-car neuron of InceptionV1 as a detector of a shared shape, and that neurons can be studied directly, without decomposition methods such as SAEs. In our lattice, such a neuron is a single closed node that the human vocabulary does not name, i.e., a machine-only concept in the sense of \citet{understandvocab}. Polysemanticity is therefore a statement relative to the chosen concept family rather than a claim about superposition, and the two interpretations of the same neuron do not conflict. The ``recursive generalization'' described there, in which a neuron grows related offshoots from a core pattern, corresponds to parent--child relations in the lattice. Our lattice is built from activation sets alone and does not assume superposition, and we do not argue that SAE decomposition is necessary; it applies to native neurons as well, although our experiments cover only SAE neurons. Both works are complementary to ours: they argue the position conceptually and empirically, and we supply definitions, representability conditions, and error measures under which it can be stated and tested.}
    
    \textbf{Network Capacity and Hyperplane Arrangements.}
    This geometric view also connects to capacity analysis: hyperplane arrangements relate activation patterns to expressivity \citep{regionnumber,regionnumber2,stanleyhyoperplane,itskov2020hyperplanecodes}, including Top-$K$ expert selection \citep{tropicalgeometrymoe}. Our positive-support units correspond to preimages of the neural-code trunks studied by \citet{jeffs2020neuralcodes}, providing a combinatorial basis for neuron budgets. Classical polytope approximation results \citep{convexbodyapprox,convexbodyapprox2} supply the smooth-convex approximation rate; we translate facet counts into neuron budgets for ambient concept approximation.

\subsection{Experiment Details}\label{app:exp_detail}
\paragraph{Synthetic data.} For each cluster, we sample 10,000 points with uniform density. The disjoint-concept dataset contains 8 clusters, and the overlapping-concept dataset contains 12 clusters. In the overlapping dataset, the density in an overlap region is kept the same as the density in a non-overlap region. This mimics realistic concept densities: for example, the density of ``red car'' should not be double-counted simply because a point belongs to both ``red'' and ``car.'' We approximate the data-supported measure $\mu$ by the empirical measure on observed data. For concept approximation, we additionally sample 10,000 probe points from the blank region inside the convex hull of all observed data, i.e., $\operatorname{Conv}(X)\setminus X$, and include these points when empirically estimating the ambient/probe measure $\nu$. For ReLU SAEs, we use expansion factors $\{1,2,4,8,16,32,64\}$ and $L_1$ regularization coefficients $10^{-5},\, 5\times10^{-5},\, 10^{-4},\, 5\times10^{-4},\, 10^{-3},\, 5\times10^{-3},\, 10^{-2},\, 5\times10^{-2},\, 10^{-1},\, 5\times10^{-1},\, 1,\, 3,\, 5,\, 7,$ and $10$. For Top-$K$ SAEs, $K$ must not exceed the SAE width. Since the input dimension is two, the SAE width is $2\cdot\mathrm{EF}$. We therefore use $K\in\{1,\ldots,2\cdot\mathrm{EF}\}$ for $\mathrm{EF}\in\{1,2,4,8,16\}$; $K\in\{1,\ldots,35,40,45,50,55,60,64\}$ for $\mathrm{EF}=32$; and $K\in\{1,\ldots,35,40,45,50,55,60,64,80,100,128\}$ for $\mathrm{EF}=64$.

For both architectures, we use learning rates in $\{0.1,0.01,0.001,0.0003\}$ and train for 200 epochs, which is sufficient for convergence in this synthetic setting. Since our theory is constructive, we run 10 random seeds, $\{0,1,42,3407,8347,19285,657306,482910,915673,2746089\}$, and report the best concept-learning metric across seeds. {Specifically, the concept-separation results in Section~\ref{sec:exp_synthetic} use the seed with the highest F1. In contrast, the FCA analysis in Section~\ref{sec:exp_fca} does not select the best seed; instead, we pool the SAEs from all seeds to best capture their diversity.} Each run is trained on one NVIDIA A40 GPU and takes approximately 5--20 minutes, depending on the expansion factor.

\paragraph{Realistic data.} Imagenette dataset~\cite{Howard_Imagenette_2019} is a subset of the ImageNet dataset~\cite{imagenet}. Imagenette contains 10 classes with around 13k images in each class. We use the 320px version to do embedding. DINOv3~\cite{dinov3} is an image embedding model, which tokenizes the input image to 1 CLS token, 4 register tokens, 196 patch tokens. DINOv3 first downsamples 320px input images to 224px and produces 201 tokens for our SAE training. The Yelp sentence dataset is constructed from the Yelp Open Dataset~\cite{yelp}, which contains user reviews and their associated 1--5 star ratings. We split the reviews into individual sentences and assign each sentence the star rating of its source review, resulting in five sentiment classes. We train SAEs with expansion factors $\mathrm{EF}\in{2,4,8,16,32,64,128}$. For ReLU SAEs, we use $L_1$ regularization coefficients ${10^{-6},10^{-5},10^{-4},10^{-3},10^{-2},10^{-1},0.3,0.6,1.0}$. For Top-$K$ SAEs, we use $K\in{2,4,8,16,32,64,128}$. For Imagenette, we additionally consider $K\in{500,1000,2000,6000,8000}$, and train all SAEs for 50 epochs to ensure convergence. For the Yelp sentence dataset, we additionally use $K\in{500,1500}$ for $\mathrm{EF}=2$; $K\in{500,1500,2000}$ for $\mathrm{EF}=4$; $K\in{500,1500,2000,6000}$ for $\mathrm{EF}=8$; and $K\in{500,1500,2000,6000,15000}$ for $\mathrm{EF}\in{16,32,64,128}$. All SAEs on Yelp are trained for 200 epochs to ensure convergence. The reason that we train denser Top-$K$ SAEs is to match the sparsity with ReLU SAEs in each expansion factor for a more complete comparison between SAE architectures. Each run is trained on one NVIDIA H200 GPU and takes approximately 2-5 hours, depending on the expansion factor.

\subsection{Sparse Autoencoder Phenomena}\label{app:sae_phenomena}
    We now formulate and explain several common SAE phenomena in our framework. 
    
    \textbf{Polysemanticity and monosemanticity.}
    Polysemanticity means that a neuron is related to multiple concepts. In terms of $\mathcal{R}$, this means that a neuron is associated with multiple unrelated concepts. Monosemanticity means that this reverse relation is concentrated on one concept, and thus requires $f$ to be injective: different concepts should not be assigned to the same neuron. Therefore, the number of available neurons must be at least the number of concepts when using single neuron. This is made precise in Section~\ref{sec:theory_capacity}.
    
    \textbf{Feature splitting.}
    Feature splitting means that a broad neuron in a smaller SAE is split into several more specific neurons in a larger SAE. Formally, if $\theta$ is a broad neuron and $\theta_1,\ldots,\theta_r$ are more specific neurons, then feature splitting can be expressed as $ \theta \approx \bigcup_{j=1}^r\theta_j$ with $\theta_j\cap\theta_l\approx\varnothing$ for $j\neq l$.
    The approximate disjointness is consistent with sparsity: if the split neurons overlapped heavily, then data in the overlap would activate many neurons at once and violate the sparsity constraint.
    
    \textbf{Feature absorption.}
    Feature absorption is a failure mode of hierarchical learning. Suppose $C_i\subset C_j$, where $C_i$ is the child concept and $C_j$ is the parent concept. Ideally, data in $C_i$ should also activate the model concept for $C_j$. Absorption occurs when $\mu(C_i\cap \theta_{C_j}^{c})>0$. That is, the parent feature fails to activate on part of its child concept. This can happen because activating both parent and child features increases the sparsity cost.
    
    \textbf{Feature family.}
    A feature family consists of several neurons that tend to activate together. Formally, a family $\theta_1,\ldots,\theta_r$ has nontrivial co-activation if $\bigcap_{l=1}^r\theta_l\neq\varnothing$. These neurons may represent different aspects of the same concept or nearby concepts in a semantic family.
    
    \textbf{Hierarchical concepts.}
    Hierarchical concepts correspond to set inclusion. If $C_i\subset C_j$, then $C_i$ is more specific than $C_j$. For example, ``mammal'' is contained in ``animal''. Ideally, the learned model concepts should satisfy $\theta_{C_i}\subset \theta_{C_j}$. However, maintaining this hierarchy can increase the sparsity cost because data in the child concept may need to activate both child and parent neurons.

\subsection{A Complete View of Concept Learning and Neuron Interpretation}\label{app:neuron_interp}
Although neuron interpretation has not been specified in the previous sections, the theoretical results above still provide insight into it. In particular, using individual neurons or multi-neuron units for concept learning leads to different levels of interpretation quality, such as different degrees of monosemanticity. We now state the link between concept learning and neuron interpretation in an algebraic way.

Define $U$ as the power set of all data, $U=\mathcal{P}(X)$, so that the human concept family $\mathcal{C}$ is a finite subset of $U$. Define $\mathcal{N}=[d]$ as the set of neurons. Recall that $M\in\mathcal{P}(\mathcal{N})$ is a set of neurons and $\theta_M\subseteq X$ is the activation region of $M$. For a single neuron $N\in\mathcal{N}$, we write $\theta_N:=\theta_{\{N\}}$. Since each $\theta_M$ is a subset of $X$, we also regard it as an element of $U$.

We define a binary relation $R\subseteq U\times\mathcal{N}$ by
\begin{align*}
    C R N \Longleftrightarrow C\subseteq \theta_N.
\end{align*}
Intuitively, $C R N$ means that neuron $N$ is active on the whole data region $C$. When $C\in\mathcal{C}$, this says that $N$ covers the whole human concept region $C$. This relation is not assumed to be a function: a concept may be represented by multiple neurons, and a neuron may be related to multiple concepts.

The relation $R$ induces two maps
\begin{align*}
    &f:U\rightarrow \mathcal{P}(\mathcal{N}),
    \qquad f(C)=\{N\in\mathcal{N}: C R N\},\\
    &g:\mathcal{P}(\mathcal{N})\rightarrow U,
    \qquad g(M)=\bigcap_{N\in M} \theta_N,
\end{align*}
with the convention that $g(\emptyset)=X$. The map $f$ sends a data region to the set of neurons that are active on the entire region. Thus, $f$ corresponds to the concept-to-neuron direction. The map $g$ sends a set of neurons to their common activation region. Thus, $g$ corresponds to the neuron-to-region direction used in neuron interpretation.

Both $U$ and $\mathcal{P}(\mathcal{N})$ are partially ordered by set inclusion. We write these posets as $(U,\subseteq)$ and $(\mathcal{P}(\mathcal{N}),\subseteq)$. This order-theoretic formulation is useful because it allows us to study hierarchical structure among concepts and neuron-defined regions. The larger the data region $C$ is, the fewer neurons can be active on all of it; the larger the neuron set $M$ is, the smaller its common activation region becomes. This is captured by the following Galois connection.

\begin{theorem}[Galois Connection Between Human Concepts and Model Neurons]\label{theorem:gc}
The maps $f$ and $g$ form a contravariant Galois connection between $U$ and $\mathcal{P}(\mathcal{N})$. That is, for every $C\in U$ and every $M\in\mathcal{P}(\mathcal{N})$,
\begin{align*}
    C\subseteq g(M)
    \Longleftrightarrow
    M\subseteq f(C).
\end{align*}
\end{theorem}

\begin{proof}
Fix $C\in U$ and $M\in\mathcal{P}(\mathcal{N})$. By the definition of $g$,
\begin{align*}
    C\subseteq g(M)
    \Longleftrightarrow
    C\subseteq \bigcap_{N\in M}\theta_N.
\end{align*}
This holds if and only if $C\subseteq \theta_N$ for every $N\in M$. By the definition of $R$, this is equivalent to saying that $C R N$ for every $N\in M$. By the definition of $f(C)$, this holds if and only if every $N\in M$ belongs to $f(C)$, or equivalently $M\subseteq f(C)$. Therefore, $C\subseteq g(M)\Longleftrightarrow M\subseteq f(C)$.
\end{proof}

Intuitively, the theorem says that a data region $C$ lies inside the common activation region of a neuron set $M$ if and only if every neuron in $M$ is active on the whole region $C$. When $C\in\mathcal{C}$, we interpret $C$ as a human concept. The word ``contravariant'' reflects the reversal of inclusion: adding more neurons to $M$ shrinks $g(M)$, while enlarging $C$ can only remove neurons from $f(C)$. Therefore, the concept-learning direction $f$ and the neuron-interpretation direction $g$ are two sides of the same order-theoretic structure.

This Galois connection naturally induces closure operators $g\circ f$ on $U$ and $f\circ g$ on $\mathcal{P}(\mathcal{N})$. For a data region $C\in U$, the region $g(f(C))$ is the smallest neuron-closed region, relative to the current neuron family, that contains $C$. If $C=g(f(C))$, then $C$ is closed under the concept-to-neuron-to-concept operation and is called a \textit{fixed point}. Similarly, a neuron set $M$ is a fixed point if $M=f(g(M))$. Studying fixed points is useful for two reasons. First, different regions or neuron sets with the same closure collapse to the same canonical object, which makes redundancy visible. Second, fixed points carry lattice operations, giving a clean way to compare broader and narrower model-learned regions.

From fixed points, we define \textit{formal concepts} as follows: a formal concept is a pair $(C,M)$ satisfying $f(C)=M$ and $g(M)=C$. Equip formal concepts with the order
\begin{align*}
    (C_1,M_1)\le (C_2,M_2)
    \Longleftrightarrow
    C_1\subseteq C_2.
\end{align*}
Equivalently, the neuron-side order is reversed: $(C_1,M_1)\le (C_2,M_2)$ if and only if $M_2\subseteq M_1$. The meet and join operations are
\begin{align*}
    &(C_1,M_1)\land (C_2,M_2)\\
    &\qquad = (C_1\cap C_2, f(C_1\cap C_2)),\\
    &(C_1,M_1)\vee (C_2,M_2)\\
    &\qquad = (g(M_1\cap M_2), M_1\cap M_2).
\end{align*}
Intuitively, the meet gives a more specific region, $C_1\cap C_2$, and then collects the neurons active on all of that smaller region. The join keeps the neurons shared by the two descriptions, $M_1\cap M_2$, and returns the broader closed region described by those shared neurons. The collection of formal concepts with this order is the \textit{concept lattice}. The $U$-side components $C$ are extents, and the $\mathcal{P}(\mathcal{N})$-side components $M$ are intents. Extents that are not in $\mathcal{C}$ are not necessarily meaningless; rather, they are closed regions that are not named by the chosen human concept family, and may therefore appear uninterpretable under that vocabulary. Here ``closed'' refers to Galois closure. There are at most $2^d$ formal concepts, with top $(X,f(X))$ and bottom $(g(\mathcal N),\mathcal N)$, so the finite lattice is complete.

This algebraic framework is an abstraction of the geometric framework. It gives a more general view of SAE phenomena and helps explain why concept learning can be cast as a set alignment problem between human concepts and model-learned regions.

\textbf{Polysemanticity/monosemanticity.}
One geometric signature of polysemanticity is that a neuron covers disjoint nonempty $C_1,C_2\in\mathcal{C}$ regarded as unrelated at the chosen semantic granularity, with $C_1\subseteq g(\{N\})$ and $C_2\subseteq g(\{N\})$. If $\{N\}$ is already closed, the corresponding lattice node has extent $g(\{N\})$ and intent $\{N\}$; in general, its intent is $f(g(\{N\}))$. Its extent contains $C_1\cup C_2$, so this neuron alone does not distinguish the two concepts. Lattice operations inside this fixed context can only recombine regions induced by the available neurons; they cannot create a new boundary inside $\theta_N$ unless some neuron already encodes that boundary. At the chosen semantic granularity, ideal monosemantic alignment with a target $C\in\mathcal C$ means $g(\{N\})=C$, up to the chosen notion of approximation. If the existing dictionary contains no regions that distinguish the target concepts, resolving them requires changing the dictionary or training objective. This is what feature splitting attempts to do.

\textbf{Feature splitting.}
Enlarging the SAE changes the neuron family and may provide new activation regions; separately trained dictionaries need not be nested. The training objective, through reconstruction under sparsity, may learn new neurons $N_1,N_2,\dots$ with $\theta_{N_i}\approx C_i$. Once these neurons are present, a coarse region $\theta_N$ that previously covered several concepts can be resolved into purer regions. The actual split from coarse to fine is performed by optimization in the larger dictionary. After training, the refined lattice records the result through nodes whose extents approximate the $C_i$; a human concept $C_i$ itself is an extent only if it is closed. If the old coarse feature remains, or if the refined dictionary contains shared neurons describing it, the broad region can appear above the finer ones when their extents are contained in it. Otherwise, the relation $\theta_N\approx\bigcup_i\theta_{N_i}$ should be read as an approximate geometric relation, not necessarily as an exact join inside a single intersection-based lattice.

\textbf{Concept Hierarchy.}
This is natural in the poset. If $C_{\text{large}}$ is more general than $C_{\text{small}}$, then $C_{\text{small}}\subseteq C_{\text{large}}$, and by antitonicity the neuron sets are reverse-ordered:
\begin{align*}
    f(C_{\text{large}})\subseteq f(C_{\text{small}}).
\end{align*}
Thus, moving to a more specific concept can only add neuron constraints, while moving to a more general concept removes them. However, under a {strong} sparsity constraint{ (small $L_0$)}, a standard SAE tends to learn a relatively flat set of features and does not explicitly represent these subconcept/superconcept containments. {At weaker sparsity ($L_0\ge160$ on Gemma-2-2B), a parent neuron fires with its child in most cases (Appendix~\ref{app:fca_llm_phenomena}).} Capturing such hierarchy may require lattice- or graph-structured SAEs.

\textbf{Feature family.}
A feature family can be viewed as a collection of lattice nodes whose extents have nonempty intersection. For a neuron family $M$, this common extent is $g(M)$, and $(g(M),f(g(M)))$ is the meet of its members' nodes. This captures related features without forcing each member of the family to be represented by a single neuron.

\textbf{Feature absorption.}
In the exact lattice, the hierarchy $C_{\text{small}}\subseteq C_{\text{large}}$ implies
\begin{align*}
    f(C_{\text{large}})\subseteq f(C_{\text{small}}).
\end{align*}
That is, any neuron $N$ that is active on the whole parent concept $C_{\text{large}}$ must also be active on the child concept $C_{\text{small}}$. This gives the ideal implication
\begin{align*}
    &x\in C_{\text{small}}
    \Longrightarrow
    N\text{ activates on }x,\\
    &\text{for }N\in f(C_{\text{large}}).
\end{align*}
Feature absorption is a mismatch between a detector's intended parent label and its actual coverage: a detector intended for $C_{\text{large}}$ fails to activate on a non-negligible part of $C_{\text{small}}$. In that case $N\notin f(C_{\text{small}})$ and therefore also $N\notin f(C_{\text{large}})$. The exact Galois implication remains valid; the intended parent label does not establish its premise. This suggests a diagnostic: for neurons assigned a parent label, check the expected child-to-neuron implications against empirical activation data and flag failures of coverage.

\subsection{Proofs for Concept Separation}\label{app:concept_separation}
Throughout this subsection, $C\in\mathcal C$ is a nonempty compact concept and
\[
    D_C:=X\setminus C=\bigcup_{C'\in\mathcal C,\,C'\neq C}(C'\setminus C)
\]
is the non-target region. We assume $|\mathcal{C}|$ is finite and $X\subset\mathbb R^n$ is bounded. All closures are taken in the ambient input space, and $\operatorname{Conv}(\cdot)$ denotes convex hull, with $\operatorname{Conv}(\varnothing)=\varnothing$. The neurons may be freely chosen, and a unit may use any finite, nonempty set of neurons.

\subsubsection{Proof of Theorem~\ref{theorem:neuron_separation_condition}}\label{app:neuron_separation_condition}
We first recall two standard convexity facts.
\begin{lemma}[Carath\'{e}odory's theorem~\cite{infiniteanalysis}]\label{lemma:caratheodory}
In an $n$-dimensional vector space, every vector in the convex hull of a nonempty set can be written as a convex combination of at most $n+1$ points from that set.
\end{lemma}

\begin{corollary}~\cite{infiniteanalysis}\label{corollary:convex_hull_compact}
The convex hull of a compact subset of a finite-dimensional vector space is compact.
\end{corollary}
\begin{proof}
For nonempty compact $C\subseteq\mathbb R^n$, Carath\'{e}odory's theorem writes its convex hull as the image of the compact set $C^{n+1}\times\Delta_n$ under the continuous map $((x_j),(\lambda_j))\mapsto\sum_{j=1}^{n+1}\lambda_jx_j$, where $\Delta_n=\{\lambda\ge0:\sum_j\lambda_j=1\}$. The image is therefore compact. The empty case is immediate.
\end{proof}

\paragraph{Theorem~\ref{theorem:neuron_separation_condition}.}
A concept $C$ can be separated from $D_C$ with one neuron if and only if
\[
    \operatorname{Conv}(C)\cap \overline{\operatorname{Conv}(D_C)}=\varnothing.
\]

\begin{proof}
If $D_C=\varnothing$, the condition is immediate and a sufficiently large bias makes any nonzero neuron positive on the bounded set $C$. We therefore assume $D_C\neq\varnothing$.

By Corollary~\ref{corollary:convex_hull_compact}, $K:=\operatorname{Conv}(C)$ is compact and convex. Since $D_C\subset X$ and $X$ is bounded, $\operatorname{Conv}(D_C)$ is bounded; hence $L:=\overline{\operatorname{Conv}(D_C)}$ is compact by the Heine--Borel theorem. The set $L$ is also convex because the closure of a convex set is convex.

Recall that separation by one neuron means that there are $w\in\mathbb R^n\setminus\{0\}$ and $b\in\mathbb R$ such that
\begin{equation}\label{eq:neuron_sep}
    w^\top x+b>0\quad (x\in C),
    \qquad
    w^\top y+b\le 0\quad (y\in D_C).
\end{equation}
Equivalently, with $\beta:=-b$, we require $w^\top x>\beta$ on $C$ and $w^\top y\le\beta$ on $D_C$. The usual hyperplane-separation theorem almost gives this result, but here the positive side is strict while the negative side is weak. We therefore give the direct proof.

\noindent\textbf{($\Leftarrow$)} Assume $K\cap L=\varnothing$. Since $K$ and $L$ are nonempty compact sets, the continuous map $(u,v)\mapsto\|u-v\|$ attains its minimum over $K\times L$ at some $(u^\ast,v^\ast)$. Let
\[
    \delta:=\|u^\ast-v^\ast\|,
    \qquad
    w:=u^\ast-v^\ast.
\]
Because $K\cap L=\varnothing$, we have $\delta>0$ and $w\neq0$.

For any $u\in K$ and $t\in[0,1]$, convexity gives $u^\ast+t(u-u^\ast)\in K$. By minimality,
\[
    \|u^\ast+t(u-u^\ast)-v^\ast\|^2\ge \|u^\ast-v^\ast\|^2.
\]
Expanding, dividing by $t>0$, and sending $t\to0^+$ gives
\begin{equation}\label{eq:var_C}
    w^\top u\ge w^\top u^\ast\qquad (u\in K).
\end{equation}
Similarly, varying $v\in L$ while keeping $u^\ast$ fixed yields
\begin{equation}\label{eq:var_N}
    w^\top v\le w^\top v^\ast\qquad (v\in L).
\end{equation}
Moreover,
\begin{equation}\label{eq:gap}
    w^\top u^\ast-w^\top v^\ast=w^\top(u^\ast-v^\ast)=\|w\|^2=\delta^2>0.
\end{equation}
Set $\beta:=\frac12(w^\top u^\ast+w^\top v^\ast)$ and $b:=-\beta$. For $x\in C\subseteq K$, ref~\ref{eq:var_C} and~\ref{eq:gap} imply
\[
    w^\top x+b\ge w^\top u^\ast-\beta=\tfrac12\delta^2>0.
\]
For $y\in D_C\subseteq \operatorname{Conv}(D_C)\subseteq L$, equations~\ref{eq:var_N} and~\ref{eq:gap} imply
\[
    w^\top y+b\le w^\top v^\ast-\beta=-\tfrac12\delta^2<0.
\]
Thus $(w,b)$ satisfies~\ref{eq:neuron_sep}.

\noindent\textbf{($\Rightarrow$)} Conversely, suppose $(w,b)$ satisfies~\ref{eq:neuron_sep}, and set $\beta:=-b$. Then $w^\top x>\beta$ for all $x\in C$ and $w^\top y\le\beta$ for all $y\in D_C$.

By Lemma~\ref{lemma:caratheodory}, each $u\in\operatorname{Conv}(C)$ is a finite convex combination $u=\sum_j\lambda_j x_j$ with $x_j\in C$, $\lambda_j\ge0$, and $\sum_j\lambda_j=1$. Therefore
\[
    w^\top u=\sum_j\lambda_j w^\top x_j\ge \min_j w^\top x_j>\beta,
\]
where the final strict inequality is valid because the minimum is over finitely many terms. Hence
\begin{equation}\label{eq:posC}
    w^\top u>\beta\qquad (u\in\operatorname{Conv}(C)).
\end{equation}
The same finite-convex-combination argument gives $w^\top v\le\beta$ for all $v\in\operatorname{Conv}(D_C)$. By continuity, this extends to
\begin{equation}\label{eq:negN}
    w^\top v\le\beta\qquad (v\in\overline{\operatorname{Conv}(D_C)}).
\end{equation}
If a point $p$ belonged to both $\operatorname{Conv}(C)$ and $\overline{\operatorname{Conv}(D_C)}$, equation.~\ref{eq:posC} and~\ref{eq:negN} would give the contradiction $w^\top p>\beta$ and $w^\top p\le\beta$. Thus the two sets are disjoint.
\end{proof}

\subsubsection{Proof of Corollary~\ref{corollary:neuron_separation}}\label{app:neuron_separation}
\paragraph{Corollary~\ref{corollary:neuron_separation}.}
\begin{enumerate}
    \item All concepts in $\mathcal C$ can be separated from each other by one neuron per concept if and only if
    \[
        \operatorname{Conv}(C_i)\cap\overline{\operatorname{Conv}(D_i)}=\varnothing
        \qquad (C_i\in\mathcal C),
    \]
    where $D_i=X\setminus C_i$ denotes the non-target region for $C_i$.
    \item If perfect one-neuron separation is possible for all distinct, nonempty concepts in a finite $\mathcal C$, then at least $|\mathcal C|$ neurons are necessary and $|\mathcal C|$ neurons are sufficient.
\end{enumerate}

\begin{proof}
The first claim follows by applying Theorem~\ref{theorem:neuron_separation_condition} to each concept $C_i$ with $D_i=X\setminus C_i$.

For the second claim, sufficiency follows from the first claim: each concept can be assigned one separating neuron. For necessity, suppose a neuron with activation region $H^+$ perfectly represents two concepts $C_i$ and $C_j$. Perfect separation gives $H^+\cap X=C_i$ and also $H^+\cap X=C_j$, so $C_i=C_j$. Thus distinct concepts cannot share the same one-neuron representation, and at least $|\mathcal C|$ neurons are required.
\end{proof}

\subsubsection{Proof of Theorem~\ref{theorem:unit_separation_condition}}\label{app:unit_separation_condition}
We first record a strict point-versus-convex-set separation lemma.
\begin{lemma}[Strict separation of a point]\label{lemma:point_separation}
Let $K\subseteq\mathbb R^n$ be nonempty, compact, and convex, and let $y\notin K$. Then there exist $w\neq0$ and $\beta\in\mathbb R$ such that
\[
    w^\top y<\beta,
    \qquad
    w^\top u>\beta\quad (u\in K).
\]
\end{lemma}

\begin{proof}
Since $K$ is compact, the map $u\mapsto\|u-y\|$ attains its minimum at some $u^\ast\in K$. Because $y\notin K$, $\delta:=\|u^\ast-y\|>0$. Let $w:=u^\ast-y$. For $u\in K$ and $t\in[0,1]$, convexity gives $u^\ast+t(u-u^\ast)\in K$, and minimality gives
\[
    \|u^\ast+t(u-u^\ast)-y\|^2\ge \|u^\ast-y\|^2.
\]
Expanding and letting $t\to0^+$ yields $w^\top(u-u^\ast)\ge0$, so
\[
    w^\top u\ge w^\top u^\ast=w^\top y+\|w\|^2\qquad (u\in K).
\]
Choosing $\beta:=w^\top y+\frac12\|w\|^2$ proves the claim.
\end{proof}

\paragraph{Theorem~\ref{theorem:unit_separation_condition}.}
A concept $C$ can be separated from $D_C$ by a unit if and only if
\[
    \operatorname{Conv}(C)\cap\overline{D_C}=\varnothing.
\]
\begin{proof}
A unit is a finite intersection of neuron activation regions. We use $\widetilde\theta$ for the ambient region and $\theta=X\cap\widetilde\theta$ for its restriction to the data space. Given neurons $(w_i,b_i)_{i=1}^m$, $m\ge1$, write
\[
    \widetilde\theta:=\bigcap_{i=1}^m\{x\in\mathbb R^n: w_i^\top x+b_i>0\}.
\]
Separating $C$ from $D_C$ by this unit means
\begin{equation}\label{eq:unit_sep}
    C\subseteq\widetilde\theta,
    \qquad
    D_C\cap\widetilde\theta=\varnothing,
\end{equation}
or equivalently, every point in $C$ activates all selected neurons, while every point in $D_C$ fails to activate at least one selected neuron. If $D_C=\varnothing$, one neuron with sufficiently large bias covers the bounded set $C=X$. Henceforth assume $D_C\neq\varnothing$.

\noindent\textbf{($\Leftarrow$)} Assume $\operatorname{Conv}(C)\cap\overline{D_C}=\varnothing$. Let $K:=\operatorname{Conv}(C)$, which is nonempty, compact, and convex by Corollary~\ref{corollary:convex_hull_compact}. The set $\overline{D_C}$ is compact because $D_C\subseteq X$ and $X$ is bounded.

For each $y\in\overline{D_C}$, Lemma~\ref{lemma:point_separation} gives $w_y\neq0$ and $\beta_y$ such that
\begin{equation}\label{eq:unit_point_strict}
    w_y^\top y<\beta_y,
    \qquad
    w_y^\top u>\beta_y\quad (u\in K).
\end{equation}
The open halfspaces $U_y:=\{x:w_y^\top x<\beta_y\}$ form an open cover of $\overline{D_C}$. By compactness, choose a finite subcover $U_{y_1},\ldots,U_{y_m}$. Define neurons by $(w_i,b_i):=(w_{y_i},-\beta_{y_i})$ and set
\[
    \widetilde\theta=\bigcap_{i=1}^m\{x:w_i^\top x+b_i>0\}.
\]
For every $x\in C\subseteq K$, the second inequality in~\ref{eq:unit_point_strict} gives $w_i^\top x+b_i>0$ for all $i$, so $C\subseteq\widetilde\theta$. For every $y\in D_C\subseteq\overline{D_C}$, the finite subcover gives some $i$ with $y\in U_{y_i}$, hence $w_i^\top y+b_i<0$ and $y\notin\widetilde\theta$. Therefore $D_C\cap\widetilde\theta=\varnothing$.

\noindent\textbf{($\Rightarrow$)} Conversely, suppose a finite unit $\widetilde\theta$ satisfies~\ref{eq:unit_sep}. We first show that $\operatorname{Conv}(C)\subseteq\widetilde\theta$. For any $u\in\operatorname{Conv}(C)$, write $u=\sum_j\lambda_jx_j$ with $x_j\in C$, $\lambda_j\ge0$, and $\sum_j\lambda_j=1$. For each selected neuron $i$,
\[
    w_i^\top u+b_i
    =\sum_j\lambda_j(w_i^\top x_j+b_i)>0,
\]
because every score is positive and the nonnegative coefficients sum to one. Hence $u\in\widetilde\theta$.

If $p\in\operatorname{Conv}(C)\cap\overline{D_C}$, then $p\in\widetilde\theta$. Since $\widetilde\theta$ is open and $p\in\overline{D_C}$, there exists a sequence $y_k\in D_C$ with $y_k\to p$ and eventually $y_k\in\widetilde\theta$, contradicting $D_C\cap\widetilde\theta=\varnothing$. Thus $\operatorname{Conv}(C)\cap\overline{D_C}=\varnothing$.
\end{proof}
{\textit{Remark}: $\overline{D_c}$ is compact. By Assumption.~\ref{assumption}, $X = \bigcup_{C \in \mathcal{C}} C$ is a finite union of compact sets, so $X$ is compact, i.e., closed and bounded in $\mathbb{R}^n$. Hence $D_C = X \setminus C$ is bounded, although it need not be closed, and its closure $\overline{D_C}$ is closed and bounded. Since we work in Euclidean space, $\overline{D_C}$ is compact by the Heine–Borel theorem. The finite-cover argument in Appendix A.6.3 is applied to $\overline{D_C}$, not to $D_C$, so it does not require the non-concept region to be closed. The closure is needed because a finite ReLU unit is an open set and cannot exclude non-concept points that accumulate at a point inside it.}

\subsubsection{Proof of Corollary~\ref{corollary:unit_separation}}\label{app:unit_separation}
\paragraph{Corollary~\ref{corollary:unit_separation}.}
Assume $\mathcal C$ is finite. All concepts in $\mathcal C$ can be separated from each other by units if and only if
\[
    \operatorname{Conv}(C_i)\cap\overline{(C_j\setminus C_i)}=\varnothing
    \qquad (C_i,C_j\in\mathcal C,
    \ C_j\neq C_i).
\]

\begin{proof}
Fix $C_i\in\mathcal C$ and write
\[
    D_i:=X\setminus C_i=\bigcup_{j\neq i}(C_j\setminus C_i).
\]
By Theorem~\ref{theorem:unit_separation_condition}, $C_i$ can be separated from $D_i$ by a unit if and only if
\[
    \operatorname{Conv}(C_i)\cap\overline{D_i}=\varnothing.
\]
Since $\mathcal C$ is finite,
\[
    \overline{D_i}
    =\overline{\bigcup_{j\neq i}(C_j\setminus C_i)}
    =\bigcup_{j\neq i}\overline{(C_j\setminus C_i)}.
\]
Therefore
\begin{align*}
    &\operatorname{Conv}(C_i)\cap\overline{D_i}=\varnothing\\
    &\Longleftrightarrow\\
    &\operatorname{Conv}(C_i)\cap\overline{(C_j\setminus C_i)}=\varnothing
    \quad\text{for all }j\neq i.
\end{align*}

Taking the conjunction over all $i$ proves the equivalence.
\end{proof}

\paragraph{Neuron-budget remark.}
The equivalence above guarantees that each concept can be separated by some finite unit, but the number of halfspaces obtained from the compact-cover proof depends on the geometry of the sets. When $|\mathcal C|\ge2$, a geometry-independent bound of $|\mathcal C|(|\mathcal C|-1)$ neurons follows under the stronger pairwise one-neuron condition
\[
    \operatorname{Conv}(C_i)\cap
    \overline{\operatorname{Conv}(C_j\setminus C_i)}=\varnothing
    \qquad (i\neq j).
\]
Indeed, Theorem~\ref{theorem:neuron_separation_condition} then gives one neuron that is positive on $C_i$ and non-positive on $C_j\setminus C_i$ for each ordered pair $(i,j)$. Intersecting the $|\mathcal C|-1$ neurons associated with a fixed $i$ separates $C_i$ from $X\setminus C_i$. Across all concepts this uses at most $|\mathcal C|(|\mathcal C|-1)$ neurons before any reuse. Under the weaker condition of Corollary~\ref{corollary:unit_separation}, the correct general statement is finiteness, not a universal $|\mathcal C|(|\mathcal C|-1)$ bound.

When $|\mathcal C|=1$, one neuron suffices to cover the bounded set $X$. These constructions explain why units can be more monosemantic than individual neurons: neurons can be reused across several units, and each unit can combine several halfspaces to exclude different non-target regions.

\subsubsection{Separation error and evaluation measures}\label{app:separation_error}
\begin{definition}[Separation error]\label{def:separation_error}
For $C,\theta\subseteq X$ measurable, define
\[
 e_{\mathrm{sep}}(C,\theta)=\mu(C\Delta\theta)
 =\underbrace{\mu(\theta\setminus C)}_{\text{contamination mass}}
  +\underbrace{\mu(C\setminus\theta)}_{\text{missing mass}}.
\]
\end{definition}
The equality follows because the two differences are disjoint. Under Assumption~\ref{assumption}, $\theta\setminus C=\theta\cap\bigcup_{C'\ne C}(C'\setminus C)$, recovering the union-of-non-targets expression. With $\mu(C)+\mu(\theta)>0$, the measure-based F1 score satisfies
\[
 \mathrm{F1}(C,\theta)=\frac{2\mu(C\cap\theta)}{\mu(C)+\mu(\theta)}
 =1-\frac{e_{\mathrm{sep}}(C,\theta)}{\mu(C)+\mu(\theta)}.
\]
Thus F1 equals one precisely at zero separation error, but F1 and raw symmetric-difference error need not rank imperfect units identically because the denominator varies with $\theta$.

Zero separation error means agreement up to a $\mu$-null set, not necessarily pointwise equality. For example, let $X=[0,2]$, $C=[0,1]$, and let $\mu=\nu$ be uniform on $X$. Exact ReLU-unit separation fails because $\overline{X\setminus C}=[1,2]$ intersects $C$. Nevertheless, the open interval $\widetilde\theta=(-1,1)$ has zero approximation and separation error, missing only the null endpoint $1$ on $X$. This also explains why nonempty compact concepts cannot equal open ambient ReLU units exactly, while zero measure-based error can still occur.

For arbitrary evaluation measures, approximation does not imply detection or separation error control. If $\mu\ll\nu$, then $\nu(C\Delta\widetilde\theta)=0$ implies $\mu(C\Delta\theta)=0$, because $C\Delta\theta=X\cap(C\Delta\widetilde\theta)$ and $\mu(X)=1$. More quantitatively, if $\mu(A)\le B\nu(A)$ for all Borel $A$, then $e_{\mathrm{sep}}\le B e_{\mathrm{app}}$. No such relation between the two measures is otherwise assumed.

The choice of evaluation measure also determines whether the approximation problem is nontrivial. If $\nu(C)=0$, the empty activation region has zero approximation error. In particular, a purely continuous probe measure gives zero mass to a finite set of target samples. To evaluate both observed examples and blank-region probes, one possible design is $\nu=\alpha\mu+(1-\alpha)\pi$ with $0<\alpha\le1$ and a probe distribution $\pi$. This is an illustrative design, not a claim about the implementation of our experiments. It gives $e_{\mathrm{sep}}\le e_{\mathrm{app}}/\alpha$. Such empirical mixtures are allowed by Theorem~\ref{theorem:approximation_condition}, but need not satisfy the density assumptions of Proposition~\ref{theorem:approximation_rate}.

\subsection{Proofs for Concept Approximation}\label{app:concept_approximation}
Approximation is evaluated using the ambient activation region $\widetilde\theta_M$, an intersection of finitely many open halfspaces. The encoder parameters and the finite number of selected neurons may vary with the desired accuracy. The approximation condition below holds for every Borel probability measure $\nu$ on $\mathbb R^n$, including atomic measures; density and smoothness assumptions are needed only for the quantitative rate.

\subsubsection{Proof of Theorem~\ref{theorem:approximation_condition}}\label{app:approximation_condition}
\paragraph{Theorem~\ref{theorem:approximation_condition}.}
A concept $C\in\mathcal C$ can be arbitrarily well approximated under the approximation error by units if and only if $C$ is convex up to a $\nu$-null set.

\begin{proof}
Here ``convex up to a $\nu$-null set'' means that there exists a convex Borel set $K$ such that $\nu(C\Delta K)=0$.

\noindent\textbf{($\Leftarrow$)} Suppose such a set $K$ exists. If $\nu(K)=0$, an empty activation region, realized by $x_1>1$ and $x_1<0$, has zero error. Otherwise, fix $\varepsilon>0$. By inner regularity of finite Borel measures on $\mathbb R^n$, there is a nonempty compact set $F\subseteq K$ with $\nu(K\setminus F)<\varepsilon/2$. Its convex hull $Q:=\operatorname{Conv}(F)$ is compact and satisfies $F\subseteq Q\subseteq K$.

We next approximate $Q$ from outside by open polyhedra. Take a closed box $B$ whose interior contains $Q$. For each integer $j\ge1$, the set
\[
    E_j:=\{y\in B:\operatorname{dist}(y,Q)\ge1/j\}
\]
is compact and disjoint from $Q$. Lemma~\ref{lemma:point_separation} and a finite subcover give finitely many open halfspaces whose intersection contains $Q$ and excludes $E_j$. Intersecting with the interior of $B$ gives an open polyhedron $O_j$ such that
\[
    Q\subseteq O_j\subseteq\{x:\operatorname{dist}(x,Q)<1/j\}.
\]
If $E_j$ is empty, take $O_j$ to be the interior of $B$. Set $P_j=\bigcap_{\ell=1}^j O_\ell$. These open polyhedra decrease to $Q$, so continuity from above gives $\nu(P_j\setminus Q)\to0$. Since each $P_j$ is an intersection of finitely many strict affine inequalities, it is an ambient unit, and for all sufficiently large $j$ with $\nu(P_j\setminus Q)<\varepsilon/2$,
\[
    \nu(C\Delta P_j)
    \le \nu(C\Delta K)+\nu(K\setminus Q)+\nu(P_j\setminus Q)
    <\varepsilon.
\]
Thus $C$ can be arbitrarily well approximated.

\noindent\textbf{($\Rightarrow$)} Conversely, choose ambient units $\widetilde\theta_j$ such that $\nu(C\Delta\widetilde\theta_j)<2^{-j}$. For each integer $q\ge1$, the union bound gives
\[
    \nu\!\left(\bigcup_{j\ge q}(C\Delta\widetilde\theta_j)\right)
    \le\sum_{j\ge q}2^{-j}\longrightarrow0.
\]
Thus, outside a $\nu$-null set, membership in $\widetilde\theta_j$ eventually agrees with membership in $C$. Define
\[
    K:=\liminf_{j\to\infty}\widetilde\theta_j
      =\bigcup_{q=1}^{\infty}\bigcap_{j\ge q}\widetilde\theta_j.
\]
Each intersection on the right is convex, and these intersections increase with $q$, so their union is convex. The set $K$ is Borel and agrees with $C$ outside the null set above. Hence $\nu(C\Delta K)=0$.
\end{proof}

\subsubsection{Proof of Proposition~\ref{theorem:approximation_rate}}\label{app:approximation_rate}
\paragraph{Proposition~\ref{theorem:approximation_rate}.}
Let $L\subseteq\mathbb R^n$ be an $r$-dimensional affine subspace containing $C$, where $r\ge2$, and suppose $\nu$ is supported on $L$ with continuous strictly positive density $\rho_\nu$ relative to $r$-dimensional Lebesgue measure. Let $K\subset L$ be a convex body with $C^2$ boundary and positive Gaussian curvature, all relative to $L$. There are constants $A_K$ and $m_0$ such that, for every $m\ge m_0$, some unit with $|M|\le m$ satisfies
\[
    e_{\mathrm{app}}(C,\widetilde\theta_M)
    \le \nu(C\Delta K)+A_K m^{-\frac{2}{r-1}}.
\]
In particular, the residual vanishes when $\nu(C\Delta K)=0$. The constant $A_K$ depends on $K$, $\nu$, and $r$.

\begin{proof}
Identify $L$ isometrically with $\mathbb R^r$. Theorem 3 of \citet{convexbodyapprox2}, equation (1.8), bounds the infimum of weighted symmetric-difference error over polytopes with at most $m$ facets. Choosing a polytope within $m^{-2/(r-1)}$ of that infimum and enlarging the constant gives a constant $A_K$ such that, for all sufficiently large $m$, there is a bounded, full-dimensional polytope $P_m\subset L$ with at most $m$ facets such that
\[
    \nu(K\Delta P_m)\le A_K m^{-\frac{2}{r-1}}.
\]
For large $m$ this error is smaller than $\nu(K)>0$, so $P_m$ cannot be lower-dimensional, since any lower-dimensional polytope is $\nu$-null.
The relative interior of $P_m$ is defined by at most $m$ strict affine inequalities on $L$. To extend them, write $L=x_0+V\mathbb R^r$ with $V^\top V=I_r$. An inequality $u^\top y+c>0$ in the coordinates $x=x_0+Vy$ extends to $u^\top V^\top(x-x_0)+c>0$ on $\mathbb R^n$. These extensions realize the relative interior as $L\cap\widetilde\theta_M$. Since $\nu$ is supported on $L$ and the polytope boundary has zero $r$-dimensional volume, it is $\nu$-null. Therefore
\[
    \nu(K\Delta\widetilde\theta_M)=\nu(K\Delta P_m)
    \le A_K m^{-\frac{2}{r-1}}.
\]
The triangle inequality now gives
\[
    \nu(C\Delta\widetilde\theta_M)
    \le \nu(C\Delta K)+\nu(K\Delta\widetilde\theta_M)
    \le \nu(C\Delta K)+A_K m^{-\frac{2}{r-1}}.
\]
\end{proof}

For $r=1$, under the same density assumption relative to the line $L$, a nondegenerate compact interval $K\subset L$ is represented, up to its $\nu$-null endpoints, by two strict affine inequalities. Thus the analogous upper bound is $e_{\mathrm{app}}(C,\widetilde\theta_M)\le\nu(C\Delta K)$ with $|M|\le2$.

The residual $\nu(C\Delta K)$ is the error of the chosen convex surrogate, not generally the least attainable error. The true irreducible error is
\[
    e_{\mathrm{irr}}:=\inf_{K\text{ convex Borel}}\nu(C\Delta K).
\]
Every finite ambient ReLU unit is Borel and convex, so $e_{\mathrm{irr}}$ is no larger than the infimum over such units. Conversely, the construction in the proof of Theorem~\ref{theorem:approximation_condition} approximates any Borel convex $K$ by a finite ambient unit $\widetilde\theta$ with $\nu(K\Delta\widetilde\theta)<\varepsilon$. Hence
\[
    \inf_{\text{finite ReLU units }\theta}\nu(C\Delta\widetilde\theta)
    \le\nu(C\Delta K)+\varepsilon.
\]
Taking the infimum over $K$ and then letting $\varepsilon\downarrow0$ proves equality of the two infima. Theorem~\ref{theorem:approximation_condition} then shows that this common value vanishes exactly when $C$ agrees with a convex Borel set up to a $\nu$-null set. If a minimizing convex set satisfies the smoothness hypotheses above, the rate holds with residual $e_{\mathrm{irr}}$. Without that additional hypothesis, the displayed rate is only asserted for each fixed admissible $K$. When $K=\operatorname{Conv}(C)$ is admissible, its residual is $\nu(\operatorname{Conv}(C)\setminus C)$, an upper bound on $e_{\mathrm{irr}}$ that can be strict.

\subsection{Proofs for Concept Learning Capacity}\label{app:capacity}
\paragraph{Theorem~\ref{theorem:capacity}.}
Let $d\ge1$ be the number of available neurons and let $k_c\ge1$ be an integer. Suppose distinct concepts have distinct nonempty representing sets $M_C\subseteq[d]$, with $|M_C|\le k_c$. Then
\[
    |\mathcal C|\le\sum_{r=1}^{\min(k_c,d)}\binom dr.
\]
In particular, along any sequence with $k_c/d\to0$,
\[
    d\ge (1-o(1))(k_c!\,|\mathcal C|)^{1/k_c}.
\]

\begin{proof}
The selected assignment $h(C):=M_C$ is injective by assumption.

There are $\sum_{r=1}^{\min(k_c,d)}\binom dr$ nonempty subsets of $[d]$ with at most $k_c$ elements, proving the exact bound. For $1\le k_c<d$, use $\binom dr\le d^r/r!$ and compare successive terms to obtain
\begin{align*}
    |\mathcal C|
    &\le\sum_{r=1}^{k_c}\binom dr
    \le\sum_{r=0}^{k_c}\frac{d^r}{r!}\\
    &\le\frac{d^{k_c}}{k_c!}
       \sum_{j=0}^{k_c}\left(\frac{k_c}{d}\right)^j
    \le\frac{d^{k_c}}{k_c!(1-k_c/d)}.
\end{align*}
Indeed, the term indexed by $r=k_c-j$ divided by $d^{k_c}/k_c!$ is a product of $j$ factors, each at most $k_c/d$. Rearranging gives
\[
    d\ge(1-k_c/d)^{1/k_c}
       (k_c!\,|\mathcal C|)^{1/k_c}.
\]
Since $(1-k_c/d)^{1/k_c}\to1$ whenever $k_c/d\to0$, the stated asymptotic bound follows.
\end{proof}

\subsection{Additional Discussion on Top-K SAE}\label{app:topk}
For Top-$K$ SAE, write $z_i(x)=\langle w_i,x\rangle+b_i$, and let $\mathrm{TopK}_k$ select exactly $k$ indices with the fixed tie-breaking convention from Section~\ref{sec:preliminary}. The activation region of neuron $i$ is
\[
    N_i^{\mathrm{topk}}
    =
    \{x\in X:z_i(x)>0\ \text{and}\ i\in\mathrm{TopK}_k(z(x))\}.
\]
Thus $N_i^{\mathrm{topk}}\subseteq H_i^+$. For a set of neurons $M$,
\[
    \theta_M^{\mathrm{topk}}=\bigcap_{j\in M}N_j^{\mathrm{topk}}.
\]
If $|M|>k$, then $\theta_M^{\mathrm{topk}}=\varnothing$. If $|M|\le k$, then $\theta_M^{\mathrm{topk}}$ equals the absolute-gating region $\bigcap_{j\in M}H_j^+$ after removing points where at least one neuron in $M$ fails to enter the top-$k$ set. This is the relative-gating effect.

Let
\[
    P_M:=\bigcap_{j\in M}\widetilde H_j^+,
    \qquad
    Q_M:=\widetilde\theta_M^{\mathrm{topk}}.
\]
These are ambient regions, so the following identity applies both to $\mu$ (supported on $X$) and to $\nu$. Then $Q_M\subseteq P_M$. Define the rank-interference region
\[
    I_M:=P_M\setminus Q_M.
\]
For any concept $C$ and any measure $\eta$ equal to either $\mu$ or $\nu$,
\begin{align*}
    \eta(C\Delta Q_M)
    &=\eta(C\Delta P_M)
      +\eta(C\cap I_M)
      -\eta((P_M\setminus C)\cap I_M).
\end{align*}
Indeed, $Q_M=P_M\setminus I_M$, so replacing $P_M$ by $Q_M$ adds false negatives on $C\cap I_M$ and removes false positives on $(P_M\setminus C)\cap I_M$.

Therefore, Top-$K$ gating can hurt by removing true target points, but it can also help by removing false positives outside the target. The identity gives $\eta(C\Delta Q_M)\le\eta(C\Delta P_M)+\eta(C\cap I_M)$, so error upper bounds transfer when target rank loss is controlled. Convexity-based necessary conditions do not transfer: Top-K units need not be convex.

\paragraph{Geometry of the Top-$K$ pieces.}
Fix the index ordering used to resolve ties, and say that $j$ outranks $i$ when $z_j>z_i$, or when $z_j=z_i$ and $j$ has higher priority. The condition that $i$ is selected is that at most $k-1$ indices outrank it. Fixing the set of outranking indices gives an intersection of affine strict or weak comparisons. Taking the union over such sets of size at most $k-1$ therefore gives a finite union of these regions; intersecting with $z_i>0$ gives $\widetilde N_i$. These pieces are measurable and convex individually, but their union need not be convex. For example, in one dimension take $d=3$, $k=2$, and scores $z_1(x)=1$, $z_2(x)=x+2$, $z_3(x)=-x+2$. Neuron $1$ is active at $x=-2$ and $x=2$, but inactive at $x=0$, regardless of the tie convention. Thus its ambient activation region is not convex. Comparisons with the absolute-gating regions above keep the same scores fixed; they do not identify the encoders learned by separately trained ReLU and Top-$K$ SAEs.

\subsection{Additional Experiments}\label{app:exp}
\subsubsection{Synthetic Data}
We present additional results on the two synthetic datasets: the disjoint-concept dataset in Fig.~\ref{fig: app_exp_disjoint} and the partially overlapping dataset in Fig.~\ref{fig: app_exp_overlap}. On the partially overlapping dataset, we evaluate concept approximation.
\begin{figure}[t]
\centering
\begin{subfigure}{\linewidth}
\centering
\vspace{-10pt}
\includegraphics[width=\linewidth]{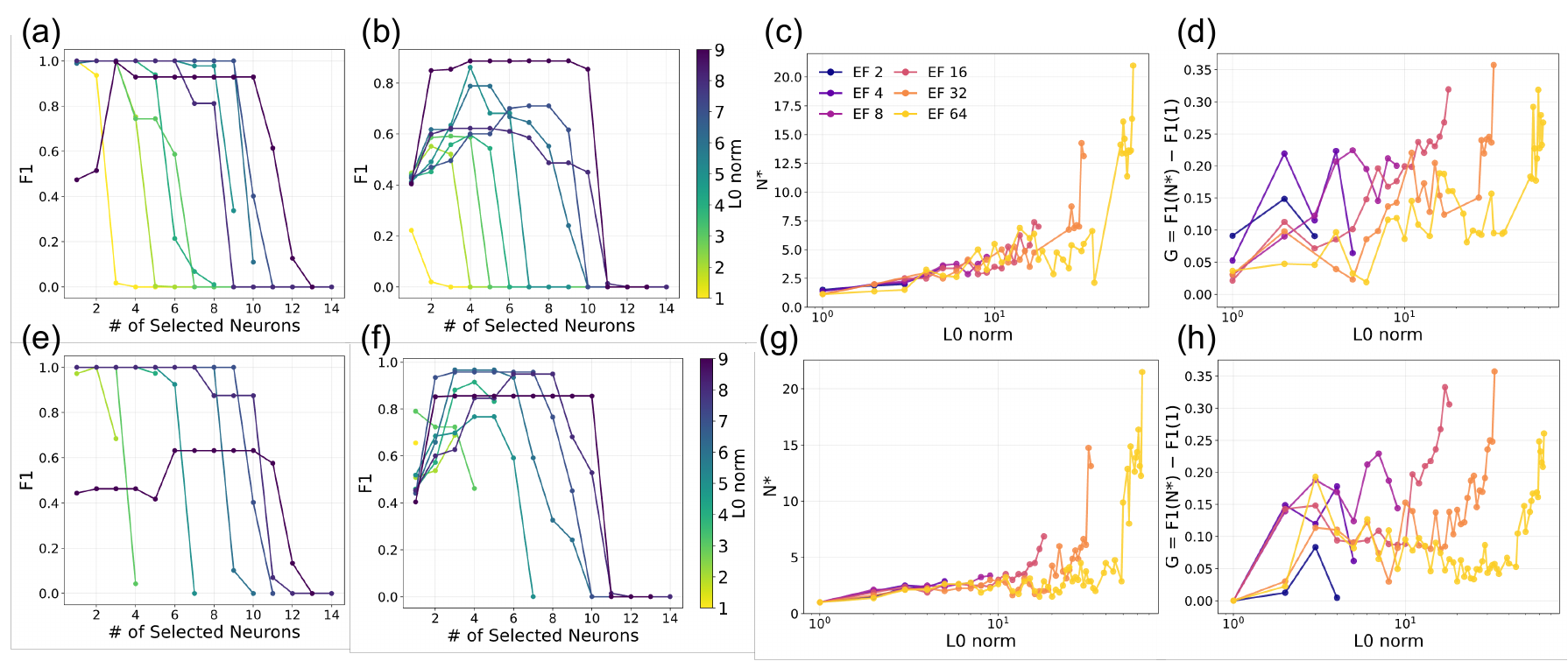}
\caption{Synthetic disjoint-concept dataset.}
\label{fig: app_exp_disjoint}
\end{subfigure}
\vspace{4pt}
\begin{subfigure}{\linewidth}
\centering
\includegraphics[width=\linewidth]{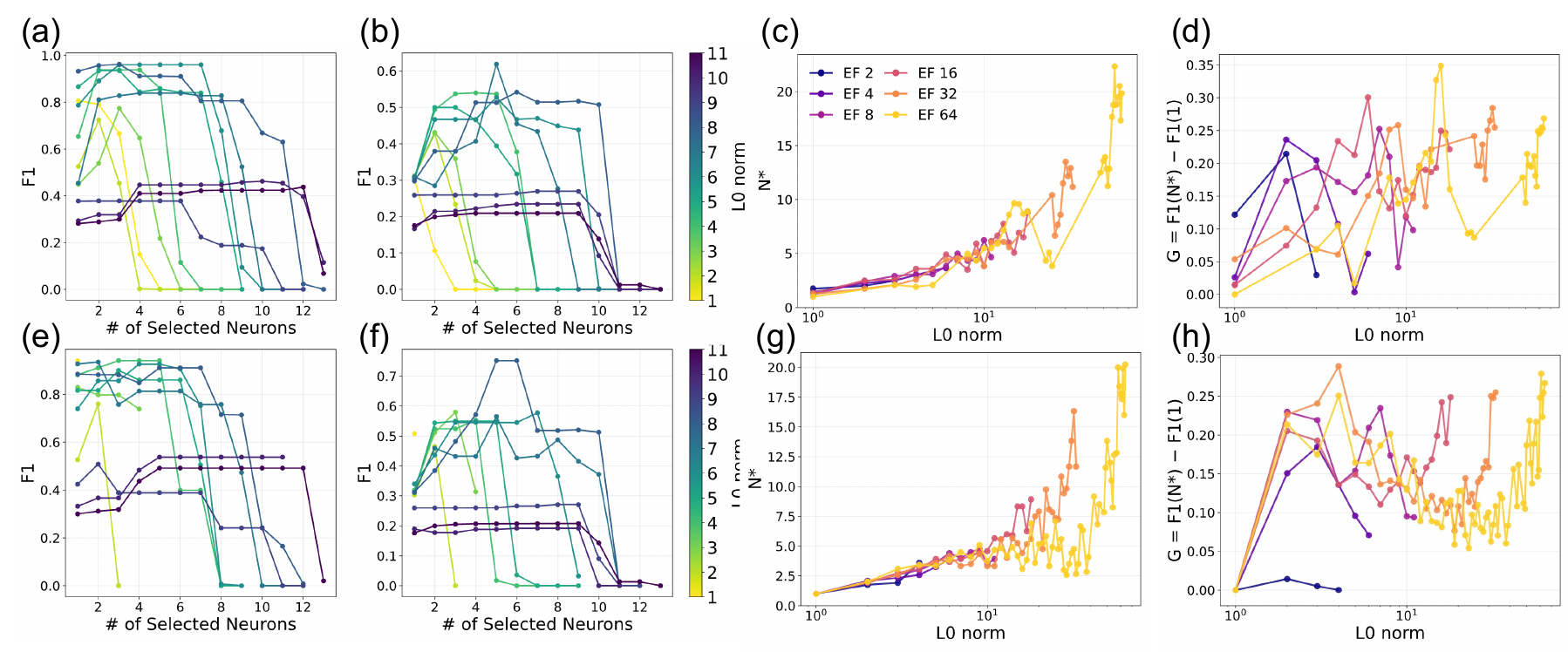}
\caption{Synthetic partially overlapping dataset.}
\label{fig: app_exp_overlap}
\end{subfigure}
\vspace{-20pt}
\caption{Concept learning on the synthetic datasets. Within each subfigure, panels (a,b) show ReLU F1 curves for an easy and a hard concept at EF$=8$, panels (c,d) show mean $N^*$ and $G$ across concepts as $L_0$ and EF vary, and panels (e--h) show the corresponding Top-$K$ results.}
\label{fig: app_exp_synthetic}
\vspace{-10pt}
\end{figure}

The blue region is the part of the convex hull of all concepts that lies outside the seen concept regions. We use a soft boundary: points within the 95th-percentile distance to a concept cluster are treated as seen. Panels (a) and (e) show an easy concept that overlaps with only one other concept. For the ReLU SAE at $L_0=6$, F1 increases from 0.710 with a single neuron to 0.9281 with a 4-neuron unit; for the Top-$K$ SAE at $L_0=4$, it increases from 0.7815 to 0.9089 with a 5-neuron unit. Panels (b) and (f) show a harder-to-approximate concept that overlaps with three other concepts. For the ReLU SAE at $L_0=7$, F1 increases from 0.2829 with a single neuron to 0.6261 with a 5-neuron unit. For the Top-$K$ SAE, the best multi-neuron unit achieves a substantially higher F1 than any single neuron, but it still falls short of the F1 achieved on the easy concept. These results show that a single neuron is generally insufficient for concept approximation, and that multi-neuron units improve approximation but do not eliminate the difficulty caused by heavy overlap.
\begin{wrapfigure}{r}{0.48\textwidth}
\centering
\includegraphics[width=\linewidth]{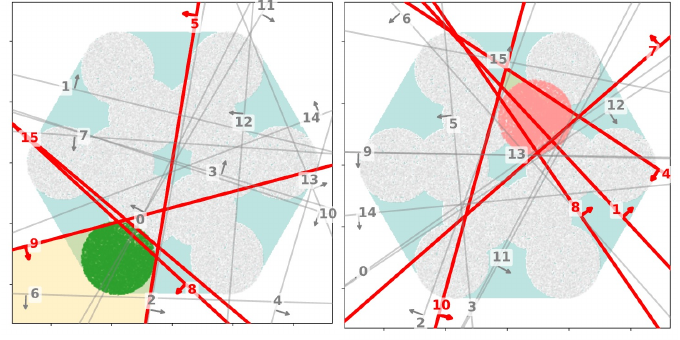}
\caption{Visualization of case studies on how SAEs approximate concepts in the partially overlapping dataset. The easy concept lies in the outer circle and the hard concept in the inner circle, matching the concepts in Fig.~\ref{fig: app_exp_overlap}. Light blue points denote unseen data.}
\label{fig:app_exp_overlap_vis}
\vspace{-40pt}
\end{wrapfigure}
 
The concept learning capacity results in panels (c,d,g,h) of Fig.~\ref{fig: app_exp_disjoint} and Fig.~\ref{fig: app_exp_overlap} are similar to those on the real-world datasets, which we analyze next.
 
To illustrate how SAEs exclude unseen data in the partially overlapping dataset, we provide a visualization in Fig.~\ref{fig:app_exp_overlap_vis}.
 
\subsubsection{Real-World Data}
The left two panels of Fig.~\ref{fig:app_exp_real_vis} visualize the real-world datasets via 2D PCA projections.
\begin{figure*}[t]
    \centering
    \includegraphics[width=0.49\textwidth]{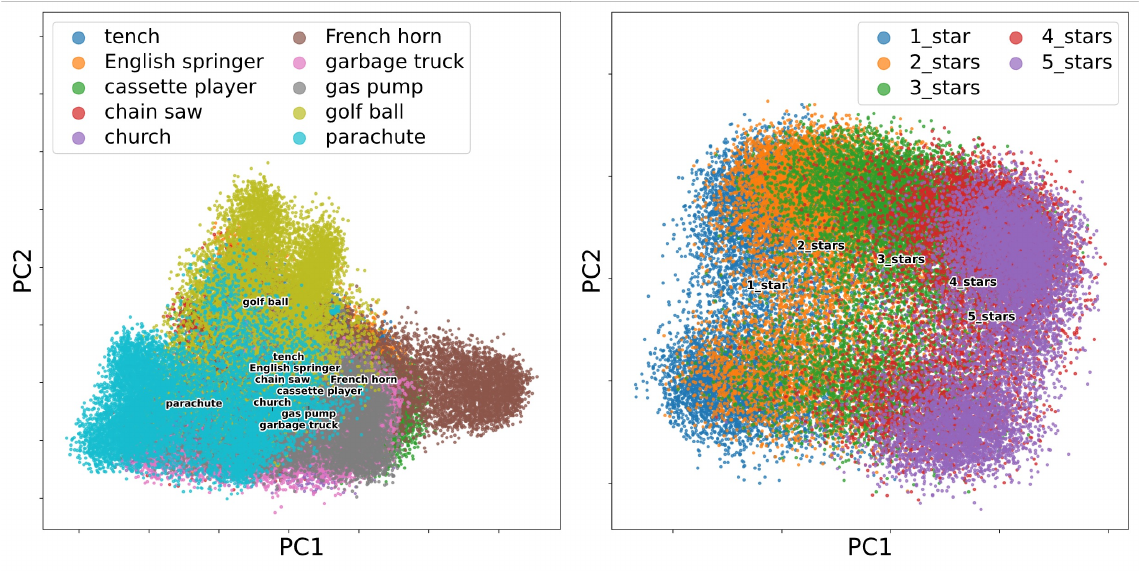}
    \hfill
    \includegraphics[width=0.49\textwidth]{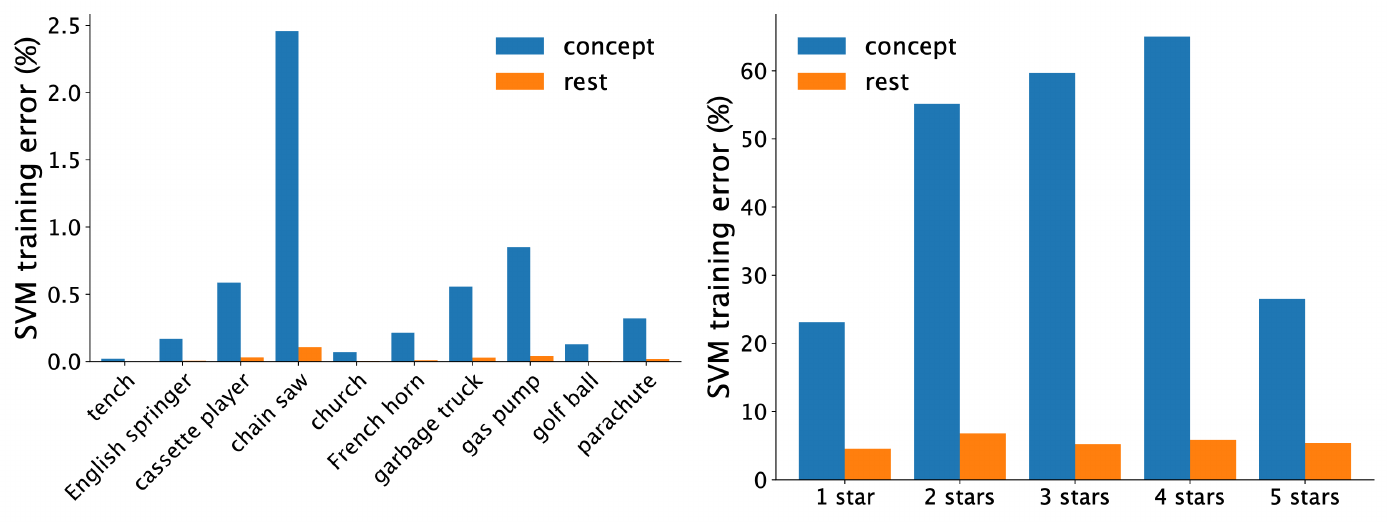}
    \caption{
    Visualization and linear separability of the real-world datasets.
    Left two panels: 2D PCA projections of Imagenette and Yelp.
    Right two panels: one-vs-rest linear SVM training error on Imagenette and Yelp.
    }
    \label{fig:app_exp_real_vis}
    \vspace{-20pt}
\end{figure*}
We test whether the convex hulls of two concepts' activations intersect. Since two finite point sets have disjoint convex hulls if and only if some hyperplane strictly separates them, this is equivalent to testing whether the two concepts are linearly separable. For every pair of concepts, we fit a linear SVM with a large regularization constant $C$, which approximates a hard-margin SVM. Before fitting, we center the activations of each pair at the midpoint of the two class means and divide them by the pair's median inter-class distance. This similarity transform preserves separability and relative margins, and it makes $C$ scale-free. We start from $C = 10^4$ and multiply $C$ by $100$ as long as training errors remain and their number keeps changing, for at most four fits. A pair with zero training error is declared \emph{disjoint}, and we report its separation $2/\lVert w \rVert$. To ensure that this separation does not depend on the choice of $C$, we solve the hard-margin dual on an active set of near-boundary points and certify the solution with a primal--dual bound. A pair that still has training errors after the final fit is declared \emph{overlapping}, and we summarize the overlap by the fraction of each concept's points that the SVM misclassifies. We repeat the same test for each concept against the union of all other concepts. One-vs-rest separability is strictly stronger than pairwise separability, because it requires a single separating hyperplane rather than an intersection of halfspaces. We report the separability results in the right two panels of Fig.~\ref{fig:app_exp_real_vis}.

Tables~\ref{tab:app_exp_image_stats} and~\ref{tab:app_exp_text_stats} quantify the concept learning curves in Fig.~\ref{fig: main_exp_image} and Fig.~\ref{fig: main_exp_text}, respectively.

\begin{table}[hbt]
\centering
\caption{Statistical quantification (mean $\pm$ standard deviation) of the phenomena observed in Fig.~\ref{fig: main_exp_image}: the behavior of the unit concept learning curves (Fig.~\ref{fig: main_exp_image}(a,b,e,f)) and the relationship of $N^*$ and $G$ to the expansion factor (EF) and $L_0$. Shape trend encodes each curve's shape as 0 (fall), 1 (hump), or 2 (rise).}

\resizebox{1.0\textwidth}{!}{
\begin{tabular}{llrrrrrr}
\toprule
Quantity & SAE & $\eta^2$ $L_0$ & adj $\eta^2$ $L_0$ & $\rho_s$ ($\log L_0$) & $\eta^2$ EF & adj $\eta^2$ EF & $\rho_s$ ($\log$ EF) \\
\midrule
\multirow{2}{*}{N*} & ReLU & $0.46 \pm 0.07$ & $0.37 \pm 0.09$ & $0.76 \pm 0.24$ & $0.22 \pm 0.05$ & $0.13 \pm 0.06$ & $-0.68 \pm 0.35$ \\
 & Top-$K$ & $0.78 \pm 0.17$ & $0.75 \pm 0.19$ & $0.76 \pm 0.12$ & $0.06 \pm 0.04$ & $-0.02 \pm 0.04$ & $-0.24 \pm 0.52$ \\
\midrule
\multirow{2}{*}{Shape trend} & ReLU & $0.53 \pm 0.08$ & $0.44 \pm 0.10$ & $0.74 \pm 0.28$ & $0.13 \pm 0.04$ & $0.03 \pm 0.04$ & $-0.60 \pm 0.44$ \\
 & Top-$K$ & $0.78 \pm 0.16$ & $0.75 \pm 0.19$ & $0.74 \pm 0.13$ & $0.04 \pm 0.03$ & $-0.04 \pm 0.03$ & $-0.18 \pm 0.52$ \\
\midrule
\multirow{2}{*}{G} & ReLU & $0.27 \pm 0.09$ & $0.15 \pm 0.10$ & $0.56 \pm 0.19$ & $0.06 \pm 0.03$ & $-0.05 \pm 0.04$ & $-0.47 \pm 0.50$ \\
 & Top-$K$ & $0.54 \pm 0.12$ & $0.47 \pm 0.14$ & $0.75 \pm 0.12$ & $0.03 \pm 0.02$ & $-0.05 \pm 0.02$ & $-0.07 \pm 0.53$ \\
\bottomrule
\end{tabular}
}
\label{tab:app_exp_image_stats}
\end{table}
 
\begin{table}[hbt]
\centering
\small
\caption{Statistical quantification (mean $\pm$ standard deviation) of the phenomena observed in Fig.~\ref{fig: main_exp_text}: the behavior of the unit concept learning curves (Fig.~\ref{fig: main_exp_text}(a,b,e,f)) and the relationship of $N^*$ and $G$ to the expansion factor (EF) and $L_0$. Shape trend encodes each curve's shape as 0 (fall), 1 (hump), or 2 (rise).}
\resizebox{1.0\textwidth}{!}{
\begin{tabular}{llrrrrrr}
\toprule
Quantity & SAE & $\eta^2$ $L_0$ & adj $\eta^2$ $L_0$ & $\rho_s$ ($\log L_0$) & $\eta^2$ EF & adj $\eta^2$ EF & $\rho_s$ ($\log$ EF) \\
\midrule
\multirow{2}{*}{N*} & ReLU & $0.38 \pm 0.02$ & $0.27 \pm 0.03$ & $0.82 \pm 0.13$ & $0.21 \pm 0.02$ & $0.10 \pm 0.02$ & $-0.40 \pm 0.70$ \\
 & Top-$K$ & $0.83 \pm 0.07$ & $0.81 \pm 0.09$ & $0.76 \pm 0.13$ & $0.03 \pm 0.01$ & $-0.05 \pm 0.02$ & $-0.28 \pm 0.46$ \\
\midrule
\multirow{2}{*}{Shape trend} & ReLU & $0.44 \pm 0.08$ & $0.34 \pm 0.10$ & $0.69 \pm 0.33$ & $0.12 \pm 0.04$ & $-0.01 \pm 0.04$ & $-0.45 \pm 0.70$ \\
 & Top-$K$ & $0.79 \pm 0.05$ & $0.76 \pm 0.06$ & $0.71 \pm 0.16$ & $0.04 \pm 0.02$ & $-0.05 \pm 0.02$ & $-0.23 \pm 0.46$ \\
\midrule
\multirow{2}{*}{G} & ReLU & $0.48 \pm 0.08$ & $0.39 \pm 0.09$ & $0.70 \pm 0.14$ & $0.14 \pm 0.04$ & $0.02 \pm 0.05$ & $-0.25 \pm 0.65$ \\
 & Top-$K$ & $0.75 \pm 0.06$ & $0.71 \pm 0.07$ & $0.74 \pm 0.14$ & $0.04 \pm 0.02$ & $-0.04 \pm 0.02$ & $-0.17 \pm 0.52$ \\
\bottomrule
\end{tabular}
}

\label{tab:app_exp_text_stats}
\end{table}
 
Fig.~\ref{fig:app_exp_shape} provides evidence for the shape trend reported in the main text, where we label the curve shapes fall, hump, and rise as 0, 1, and 2, respectively. The curve shape shifts from fall at high sparsity to a mixture of hump and rise at lower sparsity, where rise is more frequent than hump.

{
\subsubsection{Concept approximation bound in Proposition~\ref{theorem:approximation_rate}}\label{app: approximation_rate}
In this section, we empirically study the error rate in Proposition~\ref{theorem:approximation_rate}. We decompose this error bound into two parts: the irreducible error $\nu(C\Delta K)$ and the rate $m^{-2/(r-1)}$, where $m$ is the number of selected neurons and $r$ is the ambient dimension. To study the rate $m^{-2/(r-1)}$, we generate mutually disjoint concepts in 4-d, 8-d, and 16-d spaces using configurations similar to those in App.~\ref{app:exp_detail}, and train ReLU and Top-$K$ SAEs with 10 seeds. For each number of selected neurons, we report the best result across SAE configurations and seeds. Since Proposition~\ref{theorem:approximation_rate} is constructive, we verify this existence result by explicitly constructing a set of hyperplanes for concept learning; the construction algorithm is shown below.

\paragraph{Construction of the polytope.} For each concept and each $m$, we build a polytope from $m$ half-spaces without using any SAE:
\begin{enumerate}
    \item \textbf{Whitening.} We map the training points of the concept to $q = \Lambda^{-1/2}V^\top(x-\mu)$, where $\mu$ and $V\Lambda V^\top$ are their mean and covariance restricted to the $r$ nonzero eigendirections, so that a ball or ellipsoid concept becomes a ball centered at the origin.
    \item \textbf{Directions.} We choose $m$ unit directions $u_1,\dots,u_m$ that cover $\mathbb{S}^{r-1}$ nearly uniformly: equally spaced angles for $r=2$, and Riesz-energy repulsion from a random initialization for $r>2$.
    \item \textbf{Polytope.} We take $P_a = \{q : \max_i u_i^\top q \le a\}$, i.e., the polytope circumscribing the ball of radius $a$.
    \item \textbf{Fitting and evaluation.} The only free parameter $a$ is chosen to minimize $\mathrm{FP}+\mathrm{FN}$ on the training set, and the approximation error and F-1 are computed on the test set and averaged over concepts.
\end{enumerate}
Replacing the score $\max_i u_i^\top q$ with $\|q\|$ gives the fitted ellipsoid $K$ (the limit $m\to\infty$), from which we compute the irreducible error $\nu(C\Delta K)$. We vary $m$ from $r+1$ to $1024$ and perform the construction only for $r\le 16$, so it is not shown for DINOv3 and Yelp. With nearly uniform directions, neighboring directions are separated by an angle $\theta\sim m^{-1/(r-1)}$, so the vertices of the polytope lie at distance $a(1/\cos\theta-1)\sim m^{-2/(r-1)}$ outside the ball, which yields the rate in Proposition~\ref{theorem:approximation_rate}. Accordingly, the dashed lines show $\nu(C\Delta K)+A\,m^{-2/(r-1)}$ with the smallest $A$ that upper-bounds the constructed curve. Unlike the proof, $a$ is fitted on data and the directions for different $m$ are not nested, which causes small fluctuations in the constructed curves.

We also investigate the effect of the unseen-data weight and report results for three settings: no unseen data, 20\% unseen data, and 80\% unseen data, while keeping the total number of data points (seen+unseen) fixed (e.g., in the 80\% unseen data setting, we sample only 20\% of the points from the seen concepts so that the total number of points remains the same). For synthetic data experiments, we directly compute the convex hull of the input data and sample points in the non-concept region. For realistic data experiments, since the dimension is high, we employ the following strategy. We first compute the mean direction $\bar{u}$ of all real points, together with the 10\%--90\% quantile ranges of their norms and of their angles to $\bar{u}$. Each candidate is then sampled as $x = \rho\,(\cos t\,\bar{u} + \sin t\,v)$, where $\rho$ and $t$ are drawn uniformly from these two ranges and $v$ is a random unit vector orthogonal to $\bar{u}$, so that its norm and angle to $\bar{u}$ match those of the real data. We keep a candidate only if a separating hyperplane (found via Frank--Wolfe projection) certifies that it lies outside the convex hull of the data, and hence outside every concept. Unseen points for the training and test splits are sampled independently with the same unseen-to-seen ratio.

We present the results for all four datasets in Figs.~\ref{fig:app_approx_rate_disjoint_error}--\ref{fig:app_approx_yelp_topk}. We first study the rate $m^{-2/(r-1)}$ in Fig.~\ref{fig:app_approx_rate_disjoint_error} (from top to bottom: no unseen data, 20\% unseen data, and 80\% unseen data). The constructed hyperplanes agree with the rate and sometimes decay even faster, which is reasonable since the rate is an upper bound. The SAE results also generally agree with the rate and the constructed hyperplanes, as they all show a fast decay followed by stabilization. Across ambient dimensions, for both SAEs and the constructed hyperplanes, smaller dimensions lead to sharper decay, while larger dimensions lead to slower decay and higher errors. To better understand how the approximation error relates to F-1, the metric we previously used to quantify concept learning quality, we show the F-1 results in Fig.~\ref{fig:app_approx_rate_disjoint_f1}. We observe that F-1 follows a trend similar to that of the approximation error.

To study the irreducible error $\nu(C\Delta K)$, we present the results on the partially overlapping dataset in Figs.~\ref{fig:app_approx_overlap_relu} and~\ref{fig:app_approx_overlap_topk} for ReLU and Top-$K$ SAEs, respectively. The constructed solution and the theoretical rate both approach the irreducible error while remaining above it; the neurons selected from SAEs also approach it, but perform worse than the constructed solution. To better understand the components of the error, we plot a histogram in the upper-right corner (note that the bar heights are not normalized by the concept weight, so comparing heights can be misleading; this plot is only used to illustrate the composition of the error). We observe a trend of first excluding the unseen data when the unseen-data weight is moderate, and then sacrificing seen data when the unseen-data weight is high. In the F-1 bar plot in the lower-right corner, we evaluate F-1 in two cases: selecting neurons on seen+unseen data and evaluating on seen+unseen data, and selecting neurons on seen+unseen data and evaluating on seen data only. This is to study whether introducing unseen data encourages selecting neurons that bound the seen concept more tightly, which turns out not to be the case. We defer a more detailed analysis to the next section (Section~\ref{app: novel_discovery}).
}

\begin{figure}[htbp]
\centering
\includegraphics[width=\linewidth]{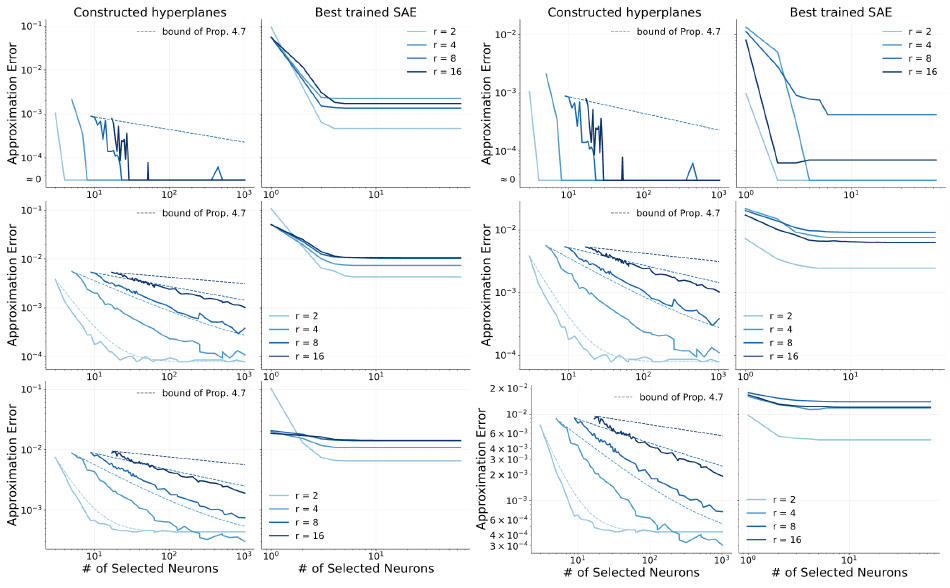}
\caption{{Approximation error versus the number of selected neurons $m$ on the synthetic mutually-disjoint concept dataset. The left two columns are for ReLU SAEs and the right two columns are for Top-$K$ SAEs; within each pair, the first column shows the constructed hyperplanes and the second column shows the best trained SAE across configurations and seeds. From top to bottom: no unseen data, 20\% unseen data, and 80\% unseen data. Colors denote the ambient dimension $r$, and dashed lines show the bound of Proposition~\ref{theorem:approximation_rate}, which decays at the rate $m^{-2/(r-1)}$.}}
\label{fig:app_approx_rate_disjoint_error}
\end{figure}

\begin{figure}[htbp]
\centering
\includegraphics[width=\linewidth]{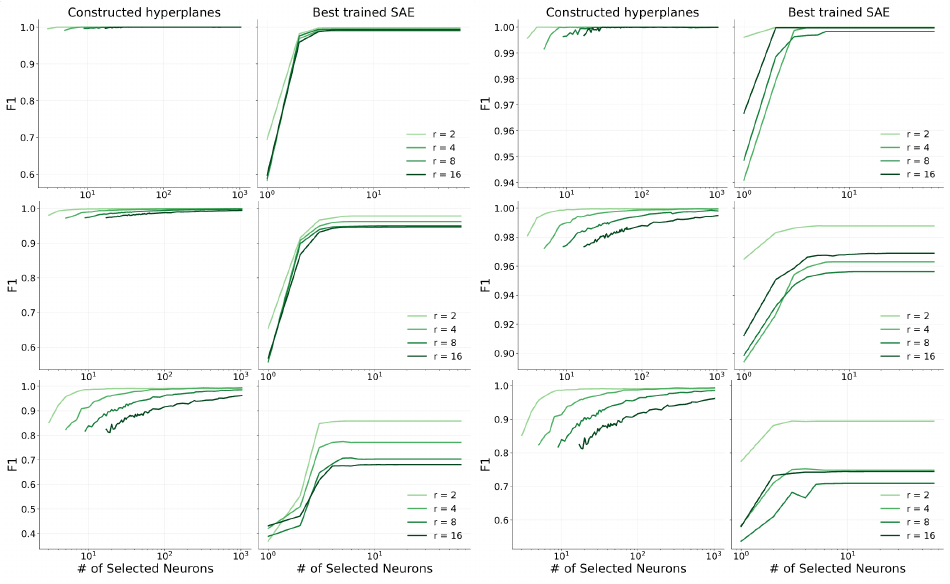}
\caption{{F-1 score versus the number of selected neurons $m$ on the synthetic mutually-disjoint concept dataset, with the same layout as Fig.~\ref{fig:app_approx_rate_disjoint_error}: ReLU SAEs in the left two columns and Top-$K$ SAEs in the right two columns, constructed hyperplanes and the best trained SAE within each pair, and no, 20\%, and 80\% unseen data from top to bottom. The F-1 score follows a trend similar to the approximation error.}}
\label{fig:app_approx_rate_disjoint_f1}
\end{figure}

\begin{figure}[htbp]
\centering
\includegraphics[width=\linewidth]{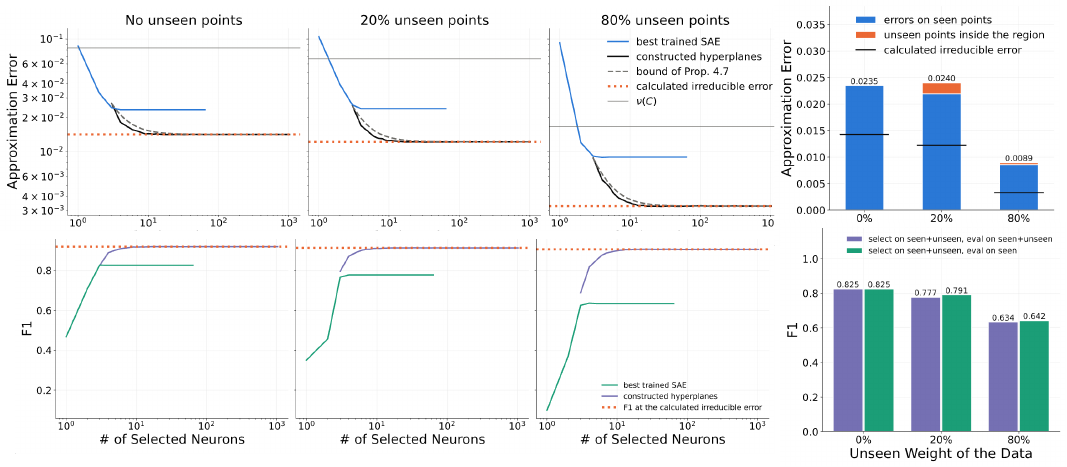}
\caption{{Approximation error on the synthetic partially overlapping concept dataset for ReLU SAEs, where the irreducible error $\nu(C\Delta K)$ is nonzero. The left three columns correspond to no unseen data, 20\% unseen data, and 80\% unseen data. \textbf{Top row:} approximation error versus the number of selected neurons for the best trained SAE and the constructed hyperplanes, together with the bound of Proposition~\ref{theorem:approximation_rate}, the calculated irreducible error, and $\nu(C)$ as a reference. \textbf{Bottom row:} the corresponding F-1 score, where the dotted line is the F-1 score at the calculated irreducible error. \textbf{Right column:} (top) the error of the best trained SAE decomposed into errors on seen points and unseen points inside the region, with black lines marking the calculated irreducible error; bar heights are not normalized by the concept weight and only illustrate the composition of the error; (bottom) the F-1 score of neurons selected on seen+unseen data, evaluated on seen+unseen data or on seen data only.}}
\label{fig:app_approx_overlap_relu}
\end{figure}

\begin{figure}[htbp]
\centering
\includegraphics[width=\linewidth]{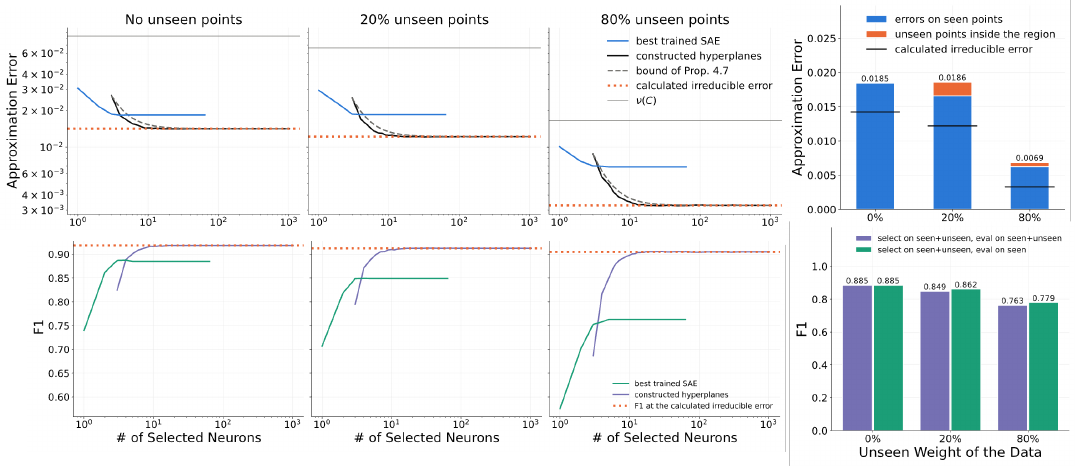}
\caption{{Same as Fig.~\ref{fig:app_approx_overlap_relu}, but for Top-$K$ SAEs.}}
\label{fig:app_approx_overlap_topk}
\end{figure}

\begin{figure}[htbp]
\centering
\includegraphics[width=\linewidth]{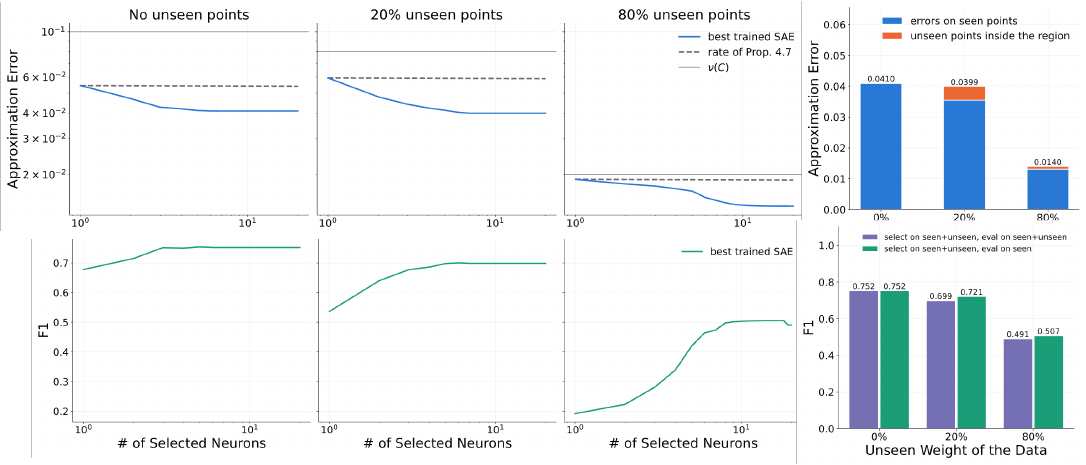}
\caption{{Approximation error on the Imagenette dataset with DINOv3 embeddings for ReLU SAEs. The layout follows Fig.~\ref{fig:app_approx_overlap_relu}, except that the constructed hyperplanes and the irreducible error are omitted for this high-dimensional dataset; dashed lines show the rate of Proposition~\ref{theorem:approximation_rate}.}}
\label{fig:app_approx_dino_relu}
\end{figure}

\begin{figure}[htbp]
\centering
\includegraphics[width=\linewidth]{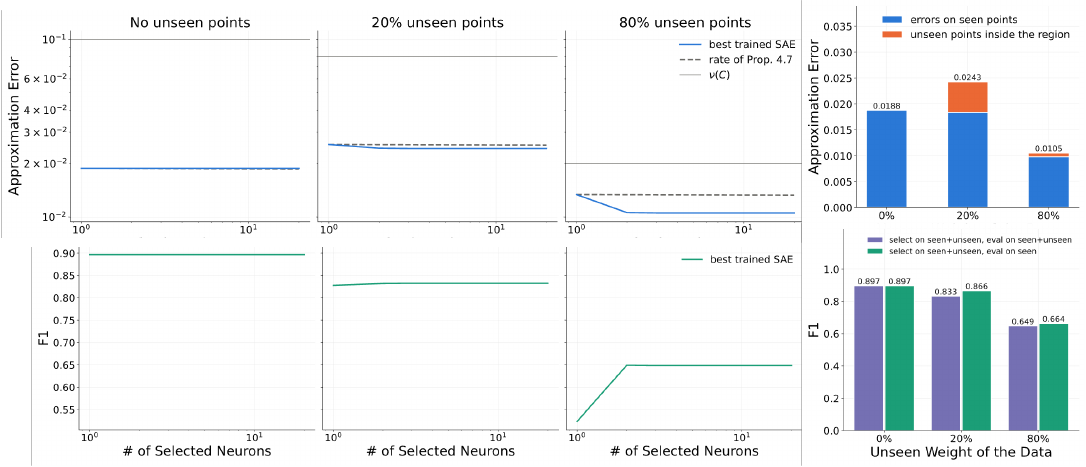}
\caption{{Same as Fig.~\ref{fig:app_approx_dino_relu}, but for Top-$K$ SAEs.}}
\label{fig:app_approx_dino_topk}
\end{figure}

\begin{figure}[htbp]
\centering
\includegraphics[width=\linewidth]{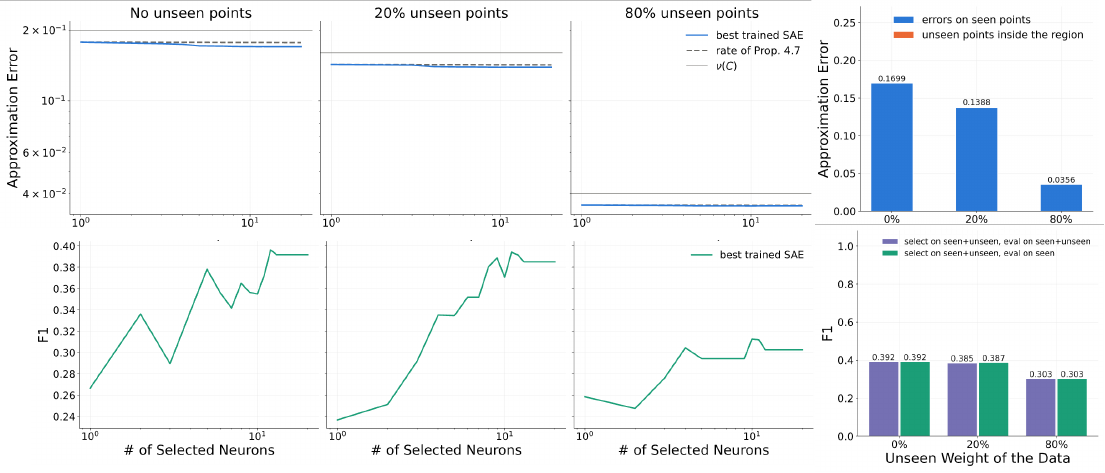}
\caption{{Approximation error on the Yelp sentence dataset for ReLU SAEs. The layout is the same as Fig.~\ref{fig:app_approx_dino_relu}.}}
\label{fig:app_approx_yelp_relu}
\end{figure}

\begin{figure}[htbp]
\centering
\includegraphics[width=\linewidth]{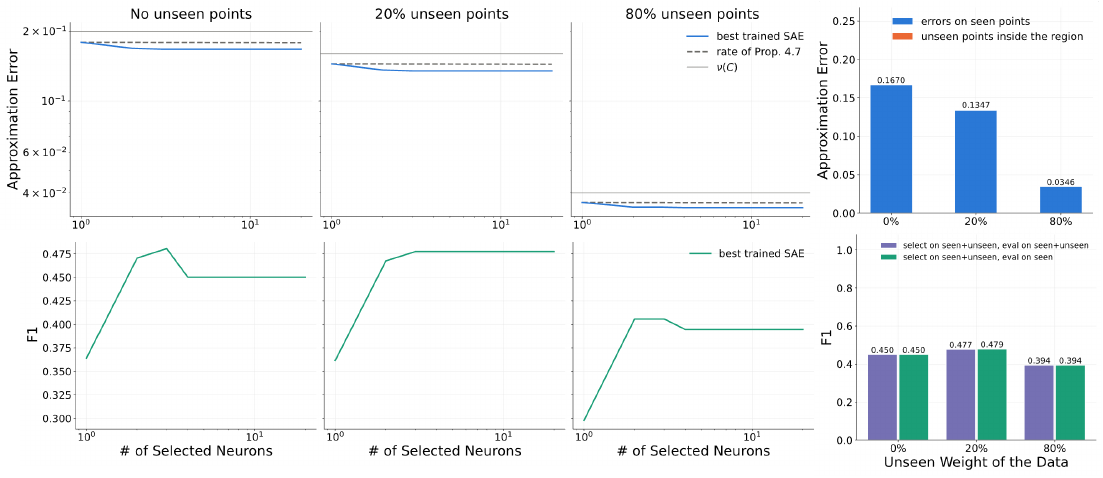}
\caption{{Same as Fig.~\ref{fig:app_approx_yelp_relu}, but for Top-$K$ SAEs.}}
\label{fig:app_approx_yelp_topk}
\end{figure}

\subsubsection{Concept learning with novel concept}\label{app: novel_discovery}
\begin{figure}[t]
\centering
\includegraphics[width=\linewidth]{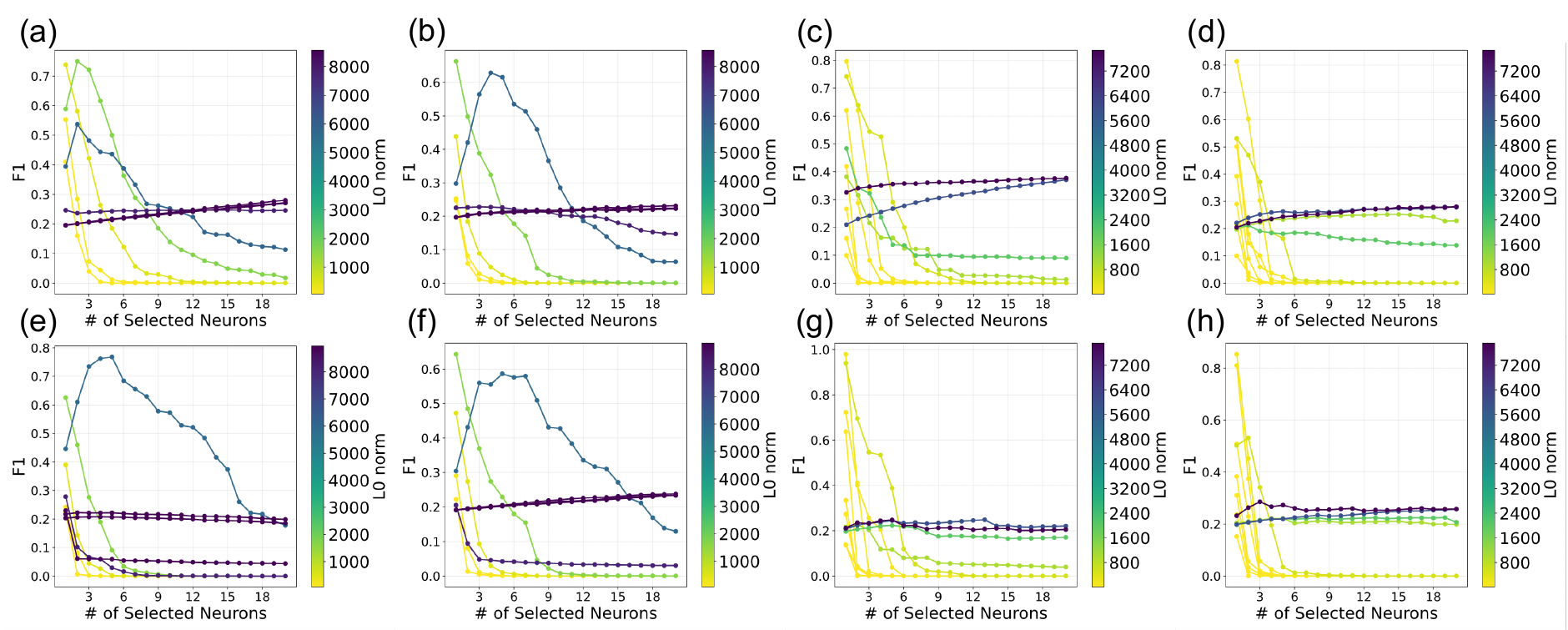}
\caption{Case study of concept learning quality on chain saw and tench for SAEs trained with or without tench held out. \textbf{(a,b)} Concept learning quality of tench and chain saw for ReLU SAEs trained on all classes. \textbf{(c,d)} The same for Top-$K$ SAEs. \textbf{(e--h)} Same as (a--d), but for SAEs trained with tench held out.}
\label{fig:app_exp_ad_tench_only}
\end{figure}
 
\begin{figure}[t]
\centering
\includegraphics[width=\linewidth]{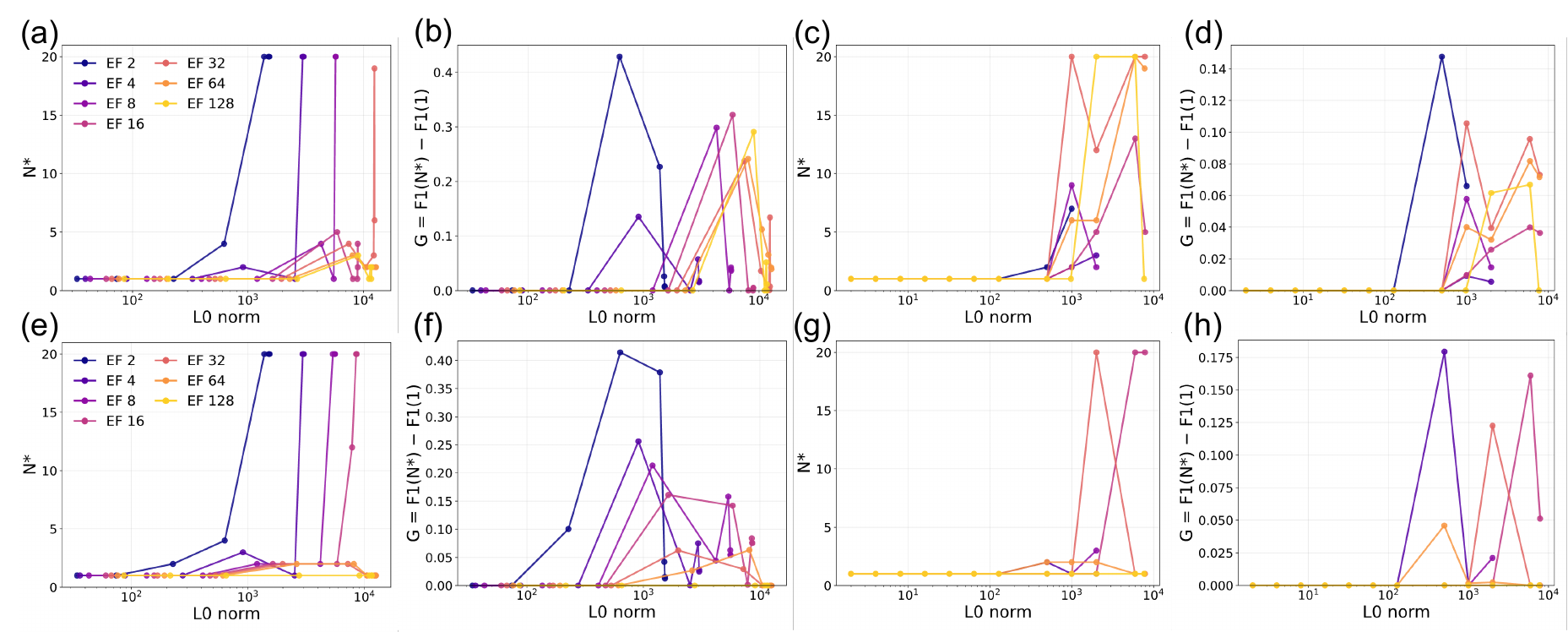}
\caption{Concept learning quality and capacity over all concepts (including tench) for SAEs trained on Imagenette with or without tench held out; when held out, tench is added back for neuron selection. \textbf{(a,b)} $N^*$ and $G$ for ReLU SAEs trained on all classes. \textbf{(c,d)} $N^*$ and $G$ for Top-$K$ SAEs trained on all classes. \textbf{(e--h)} Same as (a--d), but for SAEs trained with tench held out.}
\label{fig:app_exp_ad_with_tench}
\end{figure}

\begin{table}[t]
\centering
\small
\caption{Statistical quantification (mean $\pm$ standard deviation) of the phenomena observed in Fig.~\ref{fig:app_exp_ad_with_tench}: the behavior of the unit concept learning curves (Fig.~\ref{fig:app_exp_ad_with_tench}(a,b,e,f)) and the relationship of $N^*$ and $G$ to the expansion factor (EF) and $L_0$. Shape trend encodes each curve's shape as 0 (fall), 1 (hump), or 2 (rise).}
\resizebox{1.0\textwidth}{!}{
\begin{tabular}{llrrrrrr}
\toprule
Quantity & SAE & $\eta^2$ $L_0$ & adj $\eta^2$ $L_0$ & $\rho_s$ ($\log L_0$) & $\eta^2$ EF & adj $\eta^2$ EF & $\rho_s$ ($\log$ EF) \\
\midrule
\multirow{2}{*}{N*} & ReLU & $0.41 \pm 0.08$ & $0.30 \pm 0.09$ & $0.75 \pm 0.26$ & $0.19 \pm 0.05$ & $0.10 \pm 0.06$ & $-0.70 \pm 0.34$ \\
 & Top-$K$ & $0.71 \pm 0.28$ & $0.66 \pm 0.33$ & $0.75 \pm 0.16$ & $0.06 \pm 0.03$ & $-0.02 \pm 0.04$ & $-0.27 \pm 0.48$ \\
\midrule
\multirow{2}{*}{Shape trend} & ReLU & $0.51 \pm 0.10$ & $0.42 \pm 0.12$ & $0.73 \pm 0.30$ & $0.13 \pm 0.08$ & $0.03 \pm 0.09$ & $-0.65 \pm 0.42$ \\
 & Top-$K$ & $0.72 \pm 0.27$ & $0.68 \pm 0.32$ & $0.72 \pm 0.17$ & $0.04 \pm 0.02$ & $-0.04 \pm 0.02$ & $-0.31 \pm 0.46$ \\
\midrule
\multirow{2}{*}{G} & ReLU & $0.33 \pm 0.11$ & $0.21 \pm 0.13$ & $0.55 \pm 0.21$ & $0.08 \pm 0.05$ & $-0.02 \pm 0.06$ & $-0.42 \pm 0.51$ \\
 & Top-$K$ & $0.48 \pm 0.16$ & $0.40 \pm 0.18$ & $0.73 \pm 0.14$ & $0.04 \pm 0.02$ & $-0.04 \pm 0.02$ & $-0.03 \pm 0.59$ \\
\bottomrule
\end{tabular}
}
\label{tab:app_exp_ad_stats}
\end{table}

\begin{figure}[t]
\centering
\includegraphics[width=\linewidth]{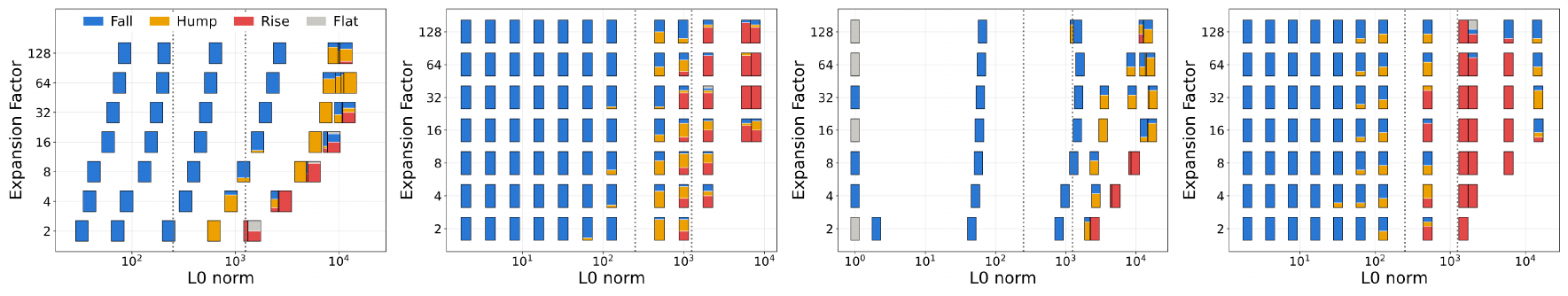}
\caption{Shape trend analysis. From left to right: ReLU and Top-$K$ SAEs on Imagenette, and ReLU and Top-$K$ SAEs on Yelp.}
\label{fig:app_exp_shape}
\vspace{-10pt}
\end{figure}
 
In the main text, we hold out the tench class of Imagenette during SAE training to test concept approximation, and we report concept learning quality aggregated over the remaining concepts only. Here, we examine how well SAEs learn tench even though they never see it during training. Fig.~\ref{fig:app_exp_ad_with_tench} shows the results aggregated over all concepts, including tench (cf.\ Fig.~\ref{fig:main_exp_ad_without_tench}, which excludes tench), and Fig.~\ref{fig:app_exp_ad_tench_only} presents a case study of chain saw and tench in this setting. For completeness, Table~\ref{tab:app_exp_ad_stats} reports the corresponding statistical quantification.
 
Comparing the two rows of Fig.~\ref{fig:app_exp_ad_with_tench}, we observe that holding out tench has little effect on the overall concept learning capacity and quality. The behavior of Top-$K$ SAEs changes more than that of ReLU SAEs, indicating that ReLU SAEs are less sensitive to changes in the training data. This may be due to the geometry of Top-$K$ gating, which can exclude more data through its rank-interference region (Section~\ref{app:topk}). In the case study (Fig.~\ref{fig:app_exp_ad_tench_only}), the concept learning quality of chain saw is largely unaffected by the training data, whereas the overall F1 score for tench is lower but remains within a reasonable range. This suggests that SAEs have the potential to learn, or at least detect, unseen concepts, pointing to possible uses in anomaly detection and the discovery of novel scientific concepts: if one concept shows lower concept learning quality while the others remain unaffected, it may be a novel concept.

{
Combining with results in previous section (Section.~\ref{app: approximation_rate}), we have the following important findings:

\textbf{SAE features can exclude most unseen inputs, but doing so can compromise detection accuracy on seen data.} On the overlapping synthetic data, constructed concept-boundary hyperplanes achieve the irreducible weighted errors of $0.0142$, $0.0122$, and $0.0033$ at unseen weights of $0\%$, $20\%$, and $80\%$, respectively. The best trained ReLU SAE remains above these values ($0.0235$, $0.0240$, and $0.0089$), with performance plateauing after three to four selected neurons. Despite this gap, most unseen inputs are successfully excluded. At $20\%$ unseen weight, only $1.05\%$ of unseen points fall within the selected concept regions on the overlapping data, compared with $2.2\%$ on Imagenette and $0.6\%$ on Yelp. However, increasing the weight of unseen negatives also increases the conditional error on seen data: from $0.0235$ to $0.0274$ and $0.0429$ on the overlapping data, from $0.041$ to $0.066$ for ReLU and $0.019$ to $0.049$ for Top-$K$ on Imagenette. The effect is much smaller on Yelp ($0.170$ to $0.178$). Thus, successfully excluding unseen inputs does not necessarily preserve the accuracy of concept detection on seen inputs.

\textbf{Successful exclusion of unseen inputs is not sufficient evidence of out-of-distribution robustness.} SAE training optimizes reconstruction and sparsity on observed data, without explicitly constraining semantic decision boundaries in unobserved regions. Consequently, the boundaries induced by learned features need not align with the boundaries required for reliable concept detection beyond the training distribution. The overlapping synthetic experiment illustrates this limitation: suitable hyperplanes exist in the original space, yet the learned SAE features do not recover equally effective boundaries under the evaluated neuron-selection procedure. The type of unseen data also matters. Our artificial unseen samples for Imagenette and Yelp are generated along random high-dimensional directions and may be easier to exclude than negatives near existing concepts. However, the substantial increase in seen error on Imagenette shows that even successfully excluding such samples need not be cost-free. These results highlight the importance of evaluating both unseen false positives and the accompanying changes in seen-data detection performance, rather than interpreting low unseen error alone as robustness.

\textbf{SAE features can be reused to detect a concept excluded from training, although the resulting representation may be incomplete.} Removing tench from SAE training changes the approximation error of the remaining nine classes by only $-0.009$ to $+0.012$, indicating little disruption to their detected concept regions. When tench is included in training, a single Top-$K$ neuron achieves F1 $=0.98$. When tench is held out, the best neuron combination still achieves F1 $=0.80$. For ReLU, F1 decreases from $0.87$ to $0.74$, with precision $0.94$ and recall $0.61$ in the held-out setting. These results indicate that features learned from other classes can be composed to detect a substantial subset of an unseen semantic concept, but do not necessarily recover its full extent. Notably, tench is the most linearly separable class in the DINO representation, with an in-sample one-vs-rest linear classification error of approximately $10^{-5}$.

\textbf{Concept-specific feature specialization and compositional concept representation are distinct properties of SAE dictionaries.} The high F1 achieved by a single Top-$K$ neuron when tench is present suggests that training-data coverage can influence the emergence of highly concept-selective features. Yet, even without such a dedicated feature, combinations of features learned from other classes can partially recover the held-out concept. Therefore, the ability to detect a concept using SAE feature combinations does not imply that the SAE was trained on that concept, nor does high detection precision guarantee complete semantic coverage. The selective degradation of tench performance, alongside relatively stable performance on the remaining classes, may provide a signal of missing concept-specific specialization. 

\textbf{The two experiments reveal different aspects of SAE behaviour on unseen data rather than contradictory results.} The seen/unseen experiments show that SAEs can exclude most unseen negatives, but stronger exclusion requirements may increase errors on seen concepts. The held-out tench experiment shows that an unseen semantic class can have little effect on the detection of existing classes while remaining partially detectable through combinations of previously learned features, which seems \textit{contradictory} to the seen/unseen data experiments in Section~\ref{app: approximation_rate}. These observations distinguish three properties: exclusion of unseen negatives, retention of performance on seen concepts, and representation of unseen concepts. Success in one does not guarantee success in the others.

\textbf{SAE behaviour on unseen data is influenced by the geometric relationship between unseen inputs and learned feature activation regions, whose boundaries are not explicitly calibrated outside the training distribution.} In the overlapping synthetic data, unseen negatives occur near existing concepts. The available SAE feature boundaries do not separate these negatives as effectively as the constructed concept-boundary hyperplanes, so selecting regions that exclude unseen inputs can incur additional errors on seen data. For Imagenette and Yelp, most artificially generated unseen points fall outside the selected concept regions. However, the different changes in seen error indicate that successful exclusion does not imply an identical geometric or performance cost across datasets. In the held-out tench experiment, the unseen class is well separated from the other classes in the original DINO representation. It has little impact on their detected regions, while a different combination of learned features can still identify part of tench. Together, these results suggest that unseen inputs need not have been observed during SAE training to be excluded, and unseen concepts need not have dedicated neurons to be partially represented. However, neither ability guarantees that the learned feature boundaries faithfully recover the semantic extent of unseen concepts.

\textbf{SAE generalization beyond the training distribution must distinguish exclusion, robustness, and semantic coverage.} First, low false-positive rates on unseen negatives do not guarantee robust concept detection: excluding unfamiliar inputs may coincide with substantial degradation on seen data. Second, SAE features exhibit compositional reuse across concepts, but accurate detection of an unseen concept may cover only part of its semantic extent. Third, information readily accessible in the original representation is not necessarily recovered by SAE feature combinations. On Imagenette, a supervised linear SVM achieves an in-sample F1 of approximately $0.998$, whereas the best evaluated Top-$K$ and ReLU SAE combinations reach $0.897$ and $0.752$, respectively. Although these are not supervision-matched comparisons, they reveal a substantial gap between linear accessibility in the original representation and accessibility through the evaluated SAE feature combinations. Finally, artificial off-manifold negatives and genuinely held-out semantic classes probe different aspects of generalization. Their outcomes should not be conflated when assessing SAE behaviour on unseen data.

Overall, the behaviour of learned SAE features outside the training distribution is not explicitly guaranteed by the reconstruction objective. Their geometric boundaries may support exclusion or partial concept representation, but need not provide stable and complete semantic detection.
}

{
\subsubsection{More Neuron Selection Algorithms}
Recall that in the main experiments, we mainly use Top-$K$ F-1 selection as the neuron selection method, in this section we present the results of more neuron selection algorithms for the synthetic disjoint-concept dataset, synthetic partially overlapping dataset, ImageNette dataset, and Yelp dataset. For synthetic datasets, we experiment matching pursuit, forward selection, Top-$K$ F-1 selection, and brute-force exhaustive search on SAEs with expansion factor $\le 16$; for realistic datasets, we experiment forward selection, regression, Top-$K$ F-1 selection, Top-$K$ recall selection, and Top-$K$ recall \& false positive rate (FPR) selection (i.e., we first select the neuron with Top-$2K$ highest F-1 score, then select the neurons with Top-$K$ lowest FPR). Specifically, for regression, we use sparse logistic regression and, for each number of selected neuron, we tune the sparsity regularization term until the sparsity (i.e., $L_0$ norm) of the regression matches or is close to the number of selected neuron. For each (expansion factor, $L_0$ norm) pair, we take the SAE with the best F-1 score across various seeds, but the results for taking the mean is similar. The results are shown in Fig.~\ref{fig: app_compare_selection}. Compared with the brute-force search, Top-$K$ F-1 selection reach only has subtle difference in the optimal F-1 score while requires more number of selected neurons to reach the optimal. In the realistic dataset, Top-$K$ selection in ReLU SAEs is the most neuron-efficient algorithm because it requires much fewer neurons to reach similar optimal F-1, sometimes even higher, as other algorithms while enjoying better time complexity than forward selection; in Top-$K$ SAEs, regression is the most neuron-efficient algorithm, but Top-$K$ F-1 selection generally brings the highest optimal F-1. The results suggest that Top-$K$ F-1 selection is indeed a near-optimal heuristic algorithm.
}

{
\subsubsection{Mean and Standard Deviation across Seeds}
For the synthetic datasets (Fig.~\ref{fig: main_exp_disjoint} and Fig.~\ref{fig: app_exp_synthetic}), we report the best seed for each configuration because our proofs are constructive. To assess how much the results depend on the random seed, we train each configuration with 10 seeds and report the mean and standard deviation in Fig.~\ref{fig: app_seed}. The seed can substantially affect concept learning quality, even for F1$(1)$, the F1 score with a single selected neuron. However, the overall trends and qualitative behavior remain the same as in the best-seed results.
}

\begin{figure}[htbp]
\centering
\begin{subfigure}{\linewidth}
\centering
\includegraphics[width=\linewidth]{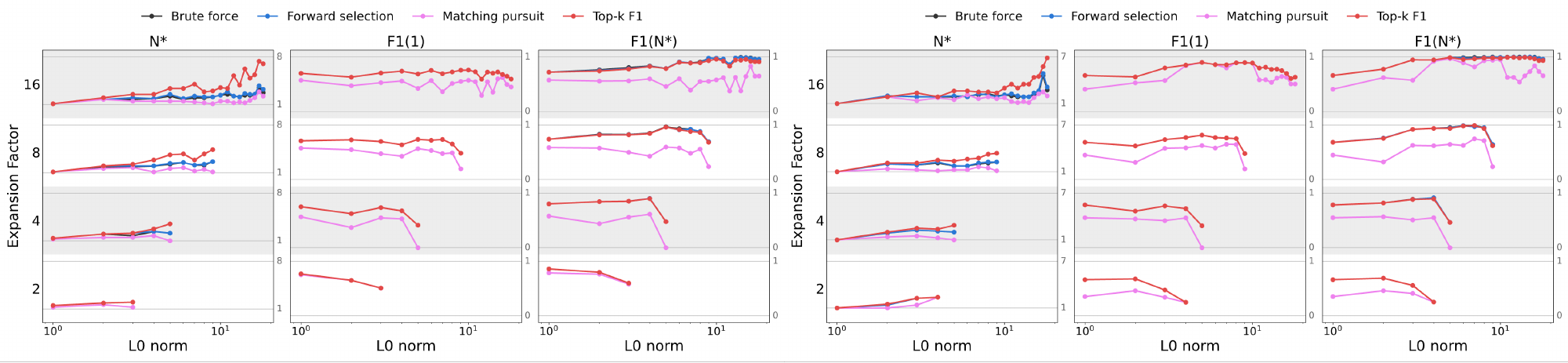}
\caption{{Synthetic disjoint-concept dataset: brute-force search, forward selection, matching pursuit, and Top-$K$ F1 selection (EF $\le 16$).}}
\label{fig: app_compare_selection_disjoint}
\end{subfigure}
\vspace{4pt}
\begin{subfigure}{\linewidth}
\centering
\includegraphics[width=\linewidth]{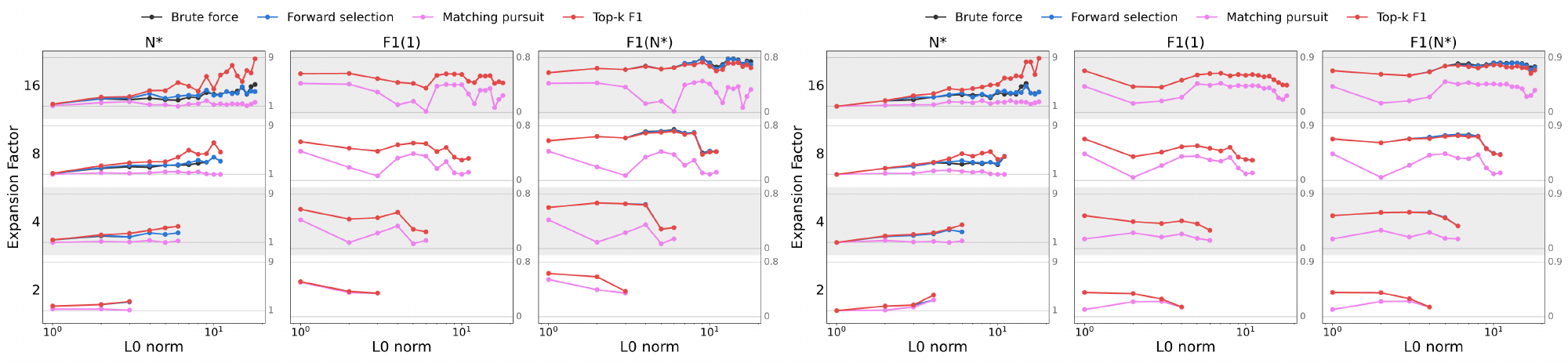}
\caption{{Synthetic partially overlapping dataset: same algorithms as in (a) (EF $\le 16$).}}
\label{fig: app_compare_selection_overlap}
\end{subfigure}
\vspace{4pt}
\begin{subfigure}{\linewidth}
\centering
\includegraphics[width=\linewidth]{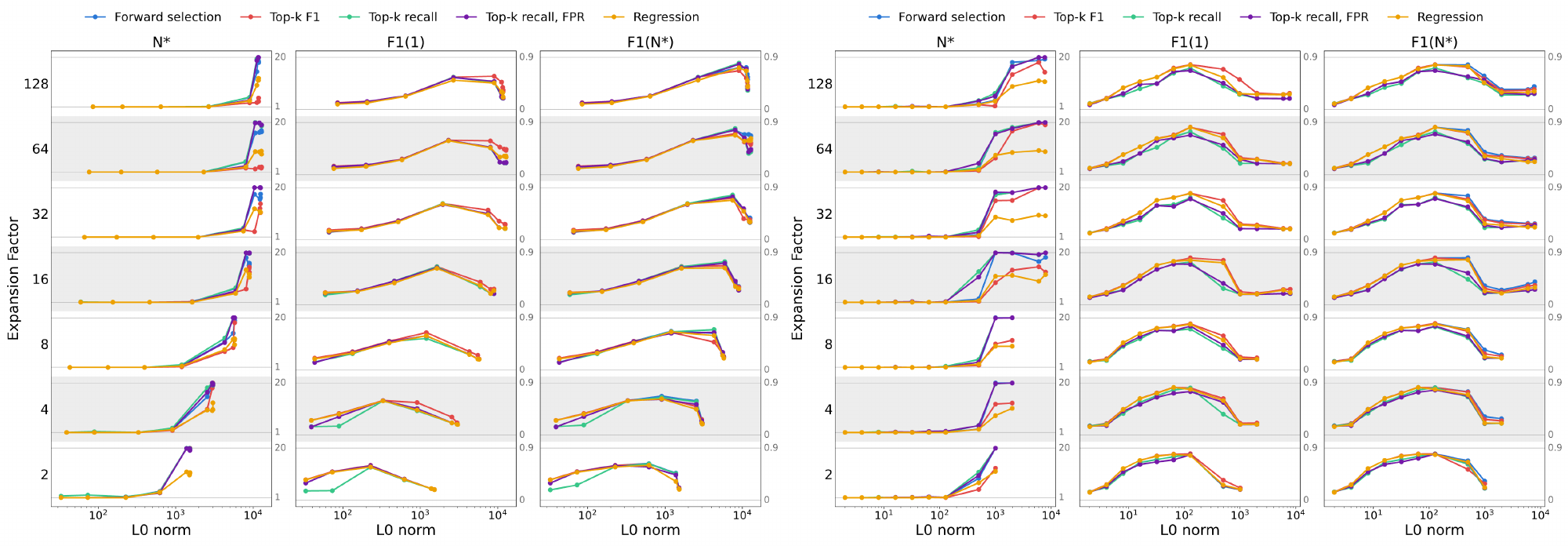}
\caption{{ImageNette: forward selection, Top-$K$ F1 selection, Top-$K$ recall selection, Top-$K$ recall \& FPR selection, and sparse logistic regression (EF from 2 to 128).}}
\label{fig: app_compare_selection_image}
\end{subfigure}
\vspace{4pt}
\begin{subfigure}{\linewidth}
\centering
\includegraphics[width=\linewidth]{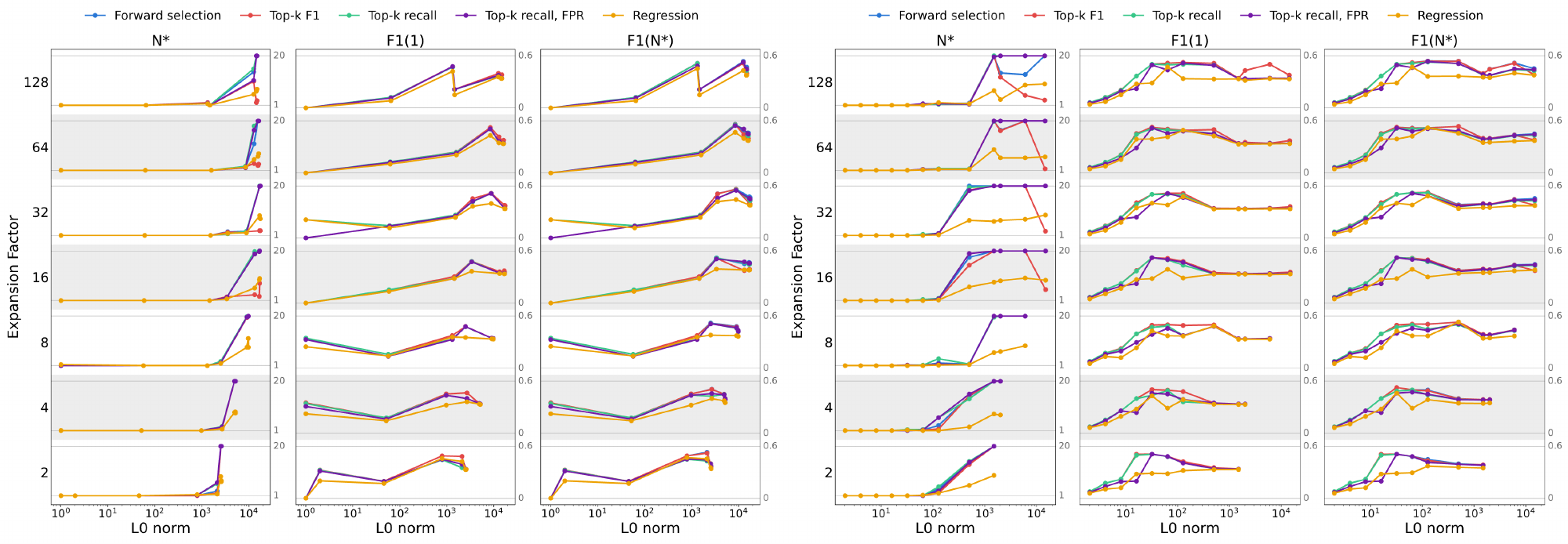}
\caption{{Yelp: same algorithms as in (c) (EF from 2 to 128).}}
\label{fig: app_compare_selection_text}
\end{subfigure}
\caption{{Comparison of neuron selection algorithms. In each subfigure, the left half shows ReLU SAEs and the right half shows Top-$K$ SAEs; each row corresponds to an expansion factor (EF), and the x-axis is the $L_0$ norm. $N^*$ is the number of selected neurons needed to reach the optimal F1, F1$(1)$ is the F1 score with a single selected neuron, and F1$(N^*)$ is the optimal F1 score. For each (EF, $L_0$) pair, we report the SAE with the best F1 across seeds.}}
\label{fig: app_compare_selection}
\end{figure}

\begin{figure}[htbp]
\centering
\begin{subfigure}{\linewidth}
\centering
\includegraphics[width=\linewidth]{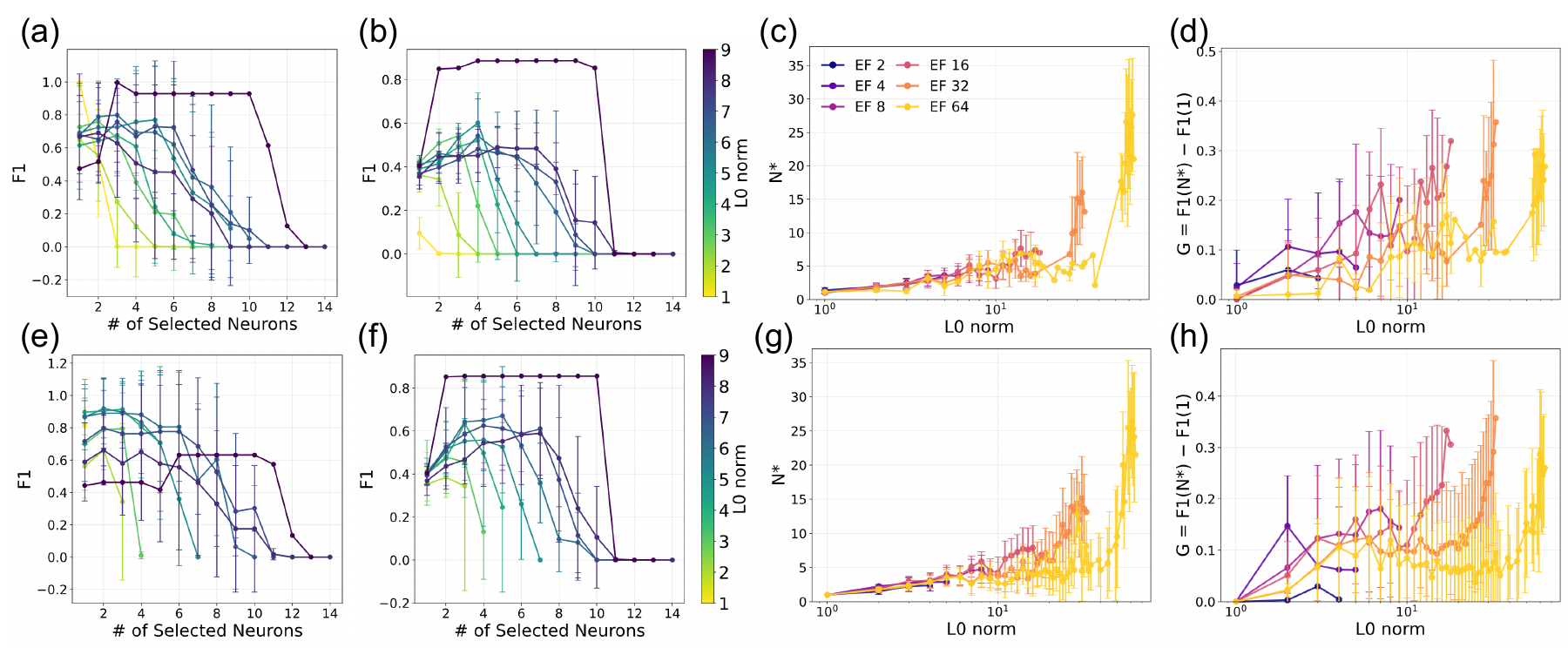}
\caption{{Synthetic disjoint-concept dataset (best-seed results in Fig.~\ref{fig: app_exp_disjoint}).}}
\label{fig: app_seed_disjoint}
\end{subfigure}
\vspace{4pt}
\begin{subfigure}{\linewidth}
\centering
\includegraphics[width=\linewidth]{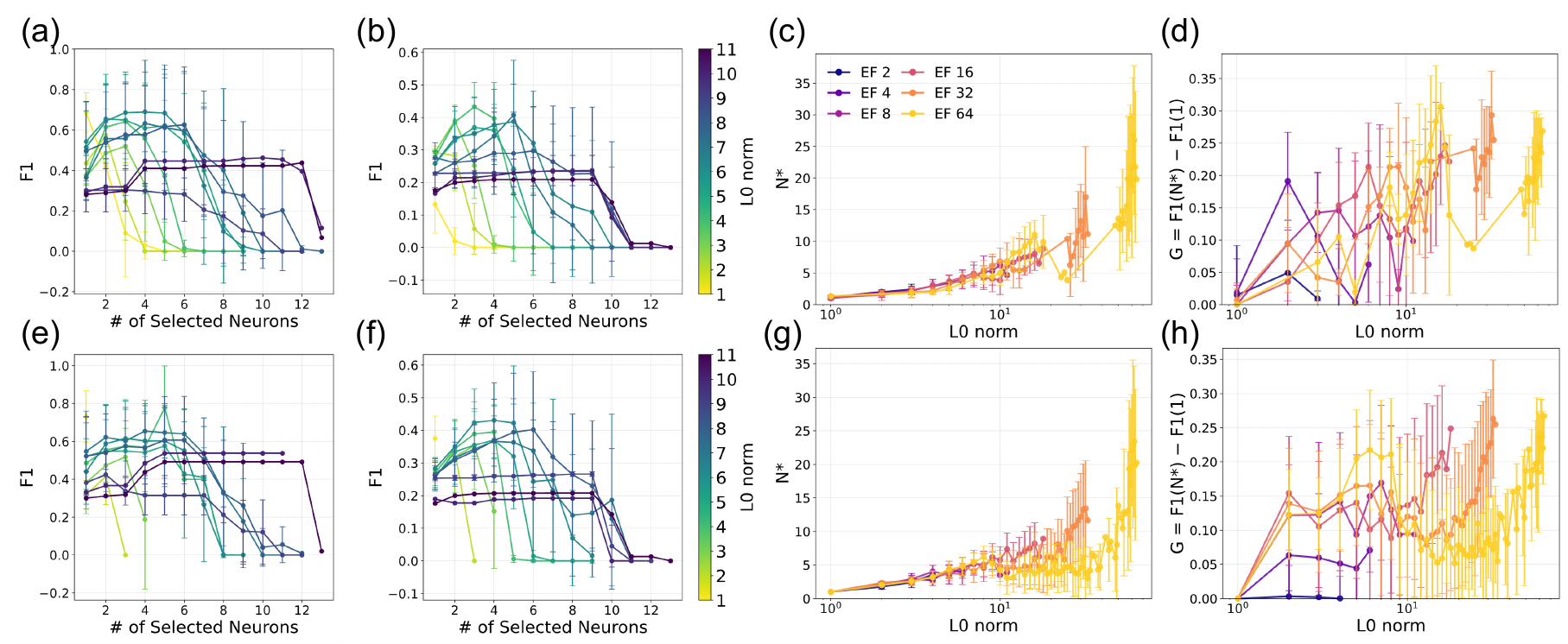}
\caption{{Synthetic partially overlapping dataset (best-seed results in Fig.~\ref{fig: app_exp_overlap}).}}
\label{fig: app_seed_overlap}
\end{subfigure}
\caption{{Seed variation of concept learning on the synthetic datasets. Each point is the mean over 10 seeds, and error bars denote one standard deviation. Within each subfigure, panels (a,b) show ReLU F1 curves for the same easy and hard concepts as in the best-seed figure at EF$=8$, panels (c,d) show mean $N^*$ and $G$ across concepts as $L_0$ and EF vary, and panels (e--h) show the corresponding Top-$K$ results.}}
\label{fig: app_seed}
\end{figure}

\clearpage
{
\subsection{FCA on an LLM: Details}\label{app:fca_llm}
This section gives the setup of Section~\ref{sec:exp_fca_llm} and the evidence for each of its claims. Each part states a finding, shows the data and examples behind it, and ends with what the finding implies. Section~\ref{app:fca_llm_setup} first defines every term that the later parts use.

\subsubsection{Setup}\label{app:fca_llm_setup}
We use the first-letter task on which \citet{featureabsorption} study feature absorption. The inputs are the 168,964 tokens of the Gemma-2-2B vocabulary that consist only of English letters, possibly with a leading space. As in their prompt, each token is preceded by two lines of random examples, such as ``tartan has the first letter: T'', and we take the residual stream after layer 12 at the position of the token itself. The label of a token is its first letter, ignoring case, so the human concepts are the 26 letters. For example, the letter s is the set of tokens that start with s, such as \texttt{stasis} and \texttt{stal}. Within each letter, we split the tokens at random into $80\%$ training and $20\%$ test tokens. Neurons and nodes are always chosen on the training tokens, and every number we report is on the test tokens. We study 36 SAEs from SAEBench~\citep{saebench}: ReLU and Top-$K$ SAEs with 4k, 16k and 65k neurons, each at six sparsity levels. We measure sparsity by $L_0$, the average number of neurons that fire on a token. On these tokens it ranges from 18 to 694 for the ReLU SAEs, and it equals $K$, from 20 to 640, for the Top-$K$ SAEs. The first letter is linearly decodable from the residual stream itself: a one-vs-rest logistic-regression probe predicts it with a test accuracy of $0.958$.

As in the rest of the paper, we treat concepts and neurons as sets of tokens. A concept is the set of tokens that carry a given label, and a neuron is described by its region, the set of tokens on which its activation is positive; we say that the neuron fires on these tokens. To compare any set of tokens with a concept $C$, we use precision, the share of the set that belongs to $C$, and recall, the share of $C$ that the set contains. F1 is always the usual harmonic mean of the two, with every token counted once. Inclusion between sets is relaxed throughout this section: a set $A$ lies in a set $B$, or $B$ holds $A$, when at least $70\%$ of the tokens of $A$ are in $B$.

FCA arranges the regions of all neurons, together with their intersections, into one lattice ordered by inclusion (Section~\ref{sec:neuron_interp}). To locate a human concept $C$ in this lattice, we look at the neurons that fire on most of it. We say that a neuron covers $C$ if it fires on at least $70\%$ of the training tokens of $C$, and we take the region of the smallest cover, the one that fires on the fewest training tokens, as the node of $C$, written $\theta_C$. The node is thus the most specific neuron that still fires on most of $C$. In FCA terms, the covers of $C$ form its intent. When no neuron covers $C$, or when the smallest cover is a dense neuron that fires on more than half of all tokens, we say that $C$ has an empty intent. Such a node would lie near the top of the lattice, so the statement that another concept lies below it would carry no information.

The smallest cover is only one way to obtain a node from the covers of $C$; we call any rule that does so a reading of the node. We use three. The single reading, our default, is the one just described. The intersection reading is the one of classical FCA: it takes the tokens on which all covers fire together. Section~\ref{app:fca_llm_readings} shows that on real activations this often loses $C$ itself, and Table~\ref{tab:app_fca_llm_toy} shows how this happens on a toy example. The multi-neuron reading lies in between. It starts from the smallest cover and adds covers one at a time, each time the one among the 50 smallest covers that most improves the training F1 of the intersection for $C$, as long as the intersection still holds $C$. It stops when F1 no longer improves or after 10 neurons. As a control, Section~\ref{app:fca_llm_human} also takes the neuron with the best F1 for $C$ as the node, which does not use the covers at all. Unless stated otherwise, $\theta_C$ refers to the single reading. All of these use only activations and labels, and none of them describes the node in words; that is a separate step, explained at the end of this section.

Comparing the node of a letter with the letter itself gives three cases. The node can be the letter ($\theta_C\approx C$): it holds the letter, and at least $70\%$ of its tokens belong to the letter. It can be the letter plus extra tokens ($\theta_C\supsetneq C$): it holds the letter and no other letter, but its precision is lower. For example, the node of b also contains \texttt{getBean} and \texttt{listBox} (Section~\ref{app:fca_llm_human}). Or it can hold another letter $C'$ as well ($\theta_C\supseteq C'$). Letters with an empty intent form a fourth case.

Tables~\ref{tab:app_fca_llm_phen} and~\ref{tab:app_fca_llm_readings} also count the FCA states of the main text. On the test tokens, we give each concept the first of the following states that applies: empty, if its intent is empty; self-miss, if its node does not hold the concept itself; merged, if its node and the node of another concept hold each other; parent, if its node holds another concept; child, if the node of another concept holds it; and isolated otherwise. A concept is label-closed if it is a child or isolated, that is, if its node holds no other concept. Label-closed only rules out other concepts, so a label-closed node can still contain many extra tokens.

We also compare neurons with each other. A neuron $m$ lies below a neuron $n$ if at least $70\%$ of the training tokens of $m$ are also in the region of $n$, so that $n$ is the broader of the two. If neither lies below the other, the two neurons are incomparable. For each neuron, we call the smallest neuron that it lies below its upper neuron. For example, the neurons for Reuters and for CNN both have a neuron for news media as their upper neuron.

Some human concepts contain others. Every token that starts with st also starts with s, so st $\subset$ s. For such a pair $C_1\subset C_2$, the lattice can show four outcomes. The two nodes can be nested ($\theta_{C_1}\subset\theta_{C_2}$), as the human concepts are. They can be the same node, and we then say that the two concepts are merged. Neither node can contain the other, and we then say that they are not nested. Finally, $C_2$ can have an empty intent.

To describe a set of tokens in words, we show 30 training tokens inside the set and 30 outside it to the language model Claude Haiku 5.5, and ask it for a rule of at most 20 words, such as ``starts with b'', or for ``no clear pattern'' if it sees none.
The outside tokens are matched to the inside ones in length, frequency and whether they are English words, so a rule cannot score well on these surface features alone. In a separate query, the model sees only the rule and selects the held-out test tokens that satisfy it, and we report the F1 of the selected tokens. These held-out samples contain a fixed number of tokens from each group rather than the true proportions, which favours broad rules (Section~\ref{app:fca_llm_human}). For letter nodes, we therefore compute F1 on the whole test vocabulary instead. The prompts are in the supplementary material.

\subsubsection{Absorption in these SAEs}\label{app:fca_llm_absorption}
Before turning to FCA, we check whether these SAEs exhibit feature absorption, using the probe criterion of \citet{featureabsorption}. For each letter, we first choose its parent neurons on the training tokens. The first one is the neuron whose encoder row is most aligned with the probe direction of the letter, that is, with the weight vector of the letter in the probe. Among the 20 most aligned neurons, we then add up to 10 more, each only if it raises the F1 of their union for the letter by more than $0.03$. A test token of the letter counts as absorbed when the probe classifies it correctly, none of the parent neurons fires on it, and another neuron that fires on it carries the letter's direction instead: its decoder row has a cosine of at least $0.025$ with the probe direction, and its contribution along that direction, its activation times the projection of its decoder row, is at least $40\%$ of the projection of the residual stream. Both thresholds come from SAEBench. The absorption rate of a letter is the share of its correctly classified test tokens that are absorbed. Averaged over the 26 letters, absorption grows as the SAE gets wider and sparser (Figure~\ref{fig:app_fca_llm_absorption_rate}). Top-$K$ 65k reaches $22$--$32\%$ at $L_0=20$--$40$, while every Top-$K$ SAE with $L_0\ge160$ stays at or below $2.4\%$. ReLU SAEs reach at most $7.8\%$. In ReLU SAEs with $L_0$ below about 50, the parent neurons fire on only $9$--$70\%$ of the letter, so these SAEs have no neuron that represents the letter, and the tokens they miss are not absorbed by any other neuron. The analyses below therefore cover SAEs with and without real absorption.

\begin{figure}[!htbp]
\centering
\includegraphics[width=0.8\linewidth]{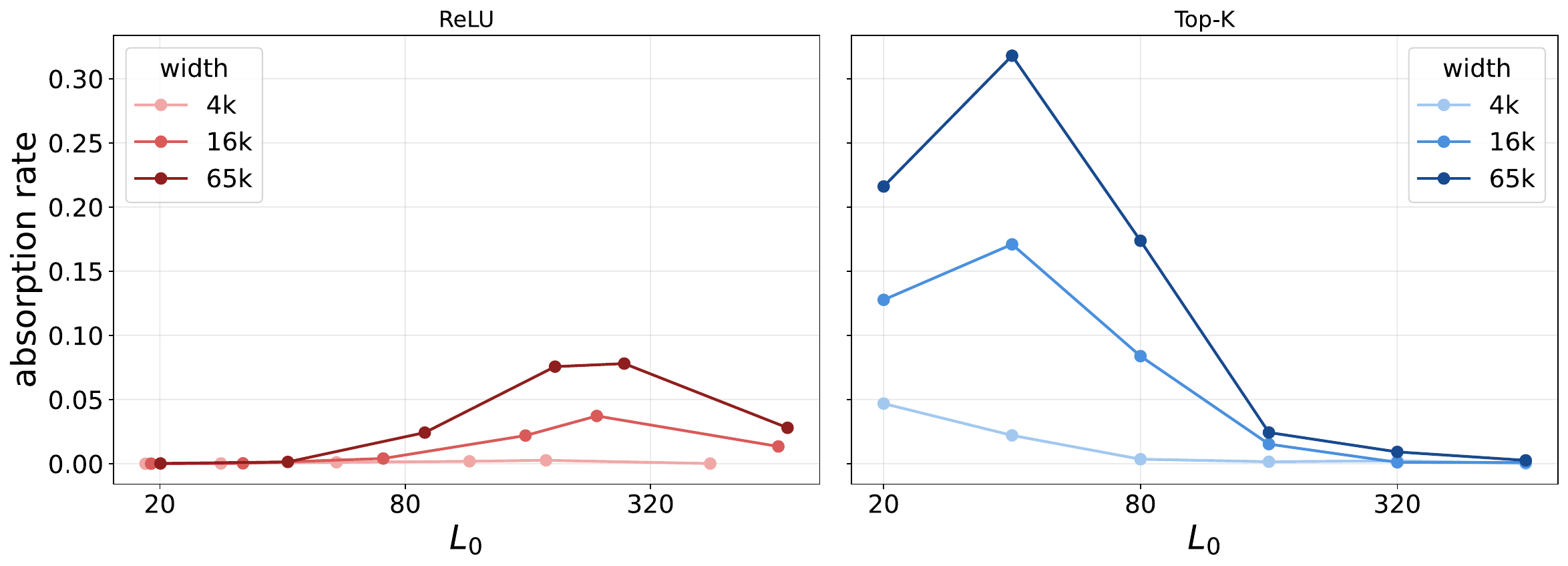}
\caption{Absorption rate (mean over the 26 letters) against $L_0$ for three widths (left: ReLU; right: Top-$K$).}
\label{fig:app_fca_llm_absorption_rate}
\end{figure}

\subsubsection{Five phenomena as stages of one structure}\label{app:fca_llm_phenomena}
Appendices~\ref{app:sae_phenomena} and~\ref{app:neuron_interp} define polysemanticity, feature splitting, hierarchy, absorption and feature families as relations between sets and in the lattice. Here we measure all five on real activations (Table~\ref{tab:app_fca_llm_phen}). A letter plays the role of the parent concept, and a two-letter prefix inside it, with at least 20 training tokens, plays the role of a child concept, for example st $\subset$ s. On the side of the SAE, the parent neurons of a letter are the ones chosen in Section~\ref{app:fca_llm_absorption}, and any other neuron with at least half of its training tokens in the letter is a child neuron. These two roles follow \citet{featureabsorption}; they describe a neuron for the whole letter and a neuron for part of it, and they are not the parent and child states of Section~\ref{app:fca_llm_setup}.

Table~\ref{tab:app_fca_llm_phen} measures each phenomenon twice, once by counting overlaps between sets and once in the lattice; the lattice measures are the rows marked ``(lattice)''. Here $L$ stands for any of the 26 letters. For polysemanticity, we count the letters on which the first parent neuron fires for at least half of the tokens, and in the lattice we check how much its F1 for the letter improves when we intersect it with up to four more neurons, each chosen to raise the training F1 the most. For feature splitting, we take the region of the letter in the 4k SAE at the same sparsity and call every neuron of a wider SAE with at least half of its training tokens in this region a fragment. We count how much of the region the fragments cover and how many fragments fire on each covered token. In the lattice, we compare the join of the fragments, the smallest neuron region that holds $90\%$ of their union, with the union itself. For hierarchy, we measure how often the parent neuron fires when a child neuron fires, and in the lattice the share of implications ``child concept $\Rightarrow$ parent'' that hold, where an implication holds if the parent neuron fires on at least half of the training tokens of the child concept. Feature families are groups of neurons that tend to fire together (Appendix~\ref{app:sae_phenomena}). For them, we count how many neurons of the letter fire on each token. In the lattice, we use the meet of two neurons, the set of tokens on which both fire, and compare the meet of the parent neuron and the space neuron, the neuron with the best F1 for a leading space, with the concept ``starts with $L$ and has a leading space''. Absorption is measured as described below.

As $L_0$ grows, the same parent--child pair changes its form (Figure~\ref{fig:app_fca_llm_phenomena}). In Top-$K$ 65k at $L_0=40$, a child neuron often fires while the parent neurons do not. This is absorption, and it happens on $35\%$ of the letter's tokens. At $L_0=160$, the child lies inside the parent instead: the probability that the parent fires when the child fires, $P(\text{parent}\mid\text{child})$, rises from $0.41$ to $0.91$. At $L_0=640$, the first parent neuron fires on at least half of the tokens of $8.9$ letters on average, which is polysemanticity. So absorption, hierarchy and polysemanticity are not separate defects, but one relation observed at three sparsity levels.

Absorption can also be measured without the probe. We count a token of the letter as absorbed in the set sense if a child neuron fires on it and no parent neuron does. Across Top-$K$ SAEs, this share correlates with the probe criterion at $r=0.975$. It correlates even better, at $r=0.993$, when we also choose the parent neurons without the probe, by ranking neurons by their mean activation inside the letter minus their mean activation outside it. Labels and activations therefore suffice to measure absorption.

For hierarchy and absorption, the lattice measures give the same information as the set measures. The lattice is informative mainly for feature families: the meet of the parent neuron and the space neuron fits ``starts with $L$ and has a leading space'' better than any single neuron (F1 $0.53$--$0.71$ versus $0.31$--$0.47$).

In short, whether an SAE appears flat, hierarchical or polysemantic depends on its sparsity $L_0$, so claims about the structure of an SAE should state its sparsity.

\begin{figure}[!htbp]
\centering
\includegraphics[width=\linewidth]{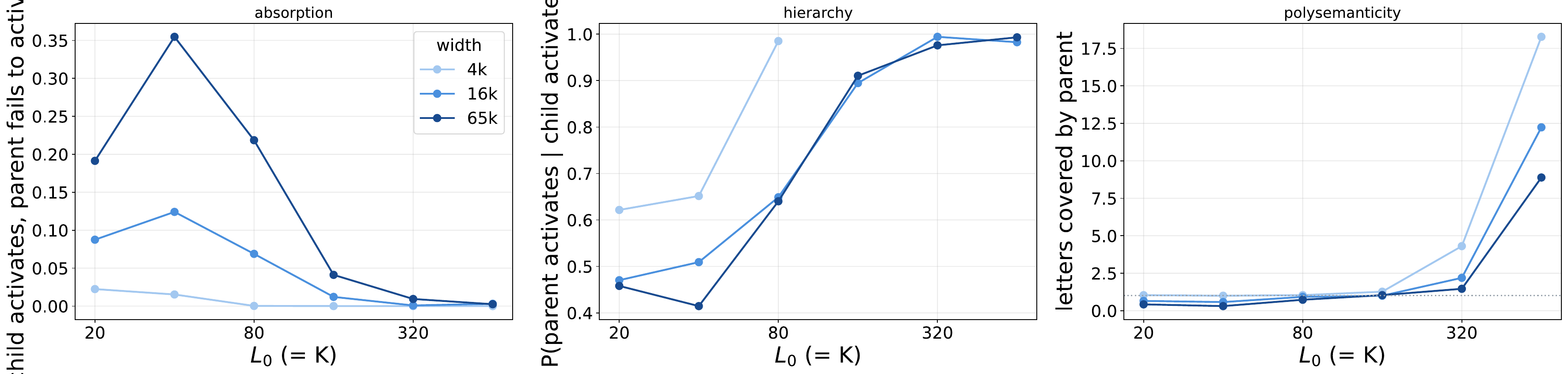}
\caption{Absorption, hierarchy and polysemanticity along $L_0$ in Top-$K$ SAEs (test tokens, mean over letters).}
\label{fig:app_fca_llm_phenomena}
\end{figure}

\begin{table}[!htbp]
\centering
\caption{The five phenomena on five reference SAEs (test tokens, mean over letters). Rows marked ``(lattice)'' are lattice measures, and the parent neuron is the first parent neuron of the letter; Section~\ref{app:fca_llm_phenomena} explains each measure.}
\label{tab:app_fca_llm_phen}
\resizebox{\textwidth}{!}{
\begin{tabular}{llccccc}
\toprule
Phenomenon & Measure & Top-$K$ 4k, $L_0$ 40 & Top-$K$ 65k, $L_0$ 40 & Top-$K$ 65k, $L_0$ 160 & Top-$K$ 65k, $L_0$ 640 & ReLU 65k, $L_0$ 187\\
\midrule
Polysemanticity & letters with $\geq 50\%$ of their tokens in the parent neuron & 1.0 & 0.3 & 1.0 & 8.9 & 0.9\\
 & precision of the parent neuron & 0.58 & 0.80 & 0.42 & 0.16 & 0.56\\
 & F1 of the parent neuron & 0.68 & 0.39 & 0.55 & 0.19 & 0.59\\
 & (lattice) F1 after intersecting 4 more neurons & 0.70 & 0.40 & 0.61 & 0.30 & 0.62\\
\midrule
Feature splitting & number of fragments (median) & -- & 579 & 1,922 & 39,400 & 117\\
 & coverage of the 4k region by the fragments & -- & 0.75 & 0.74 & 0.97 & 0.60\\
 & fragments per covered token & -- & 1.9 & 2.3 & 510.9 & 1.6\\
 & (lattice) join of fragments / union of fragments & -- & 11.1 & 4.9 & 0.9 & 5.0\\
\midrule
Hierarchy & $P(\text{parent fires}\mid\text{child fires})$ & 0.65 & 0.41 & 0.91 & 0.99 & 0.65\\
 & (lattice) implications ``child concept $\Rightarrow$ parent'' that hold & 0.99 & 0.62 & 1.00 & 0.93 & 0.97\\
\midrule
Absorption & share of tokens where only a child fires & 0.02 & 0.35 & 0.04 & 0.00 & 0.03\\
 & probe criterion (pooled over tokens) & 0.04 & 0.41 & 0.03 & 0.00 & 0.12\\
\midrule
Feature family & neurons of $L$ per token & 1.0 & 1.8 & 1.7 & 1.6 & 1.1\\
 & (lattice) F1 of ``parent $\wedge$ space'' for ``starts with $L$ and a space'' & 0.71 & 0.53 & 0.59 & 0.20 & 0.62\\
 & F1 of the best single neuron & 0.47 & 0.31 & 0.38 & 0.17 & 0.40\\
\midrule
FCA & label-closed share, single reading & 1.00 & 0.23 & 0.96 & 0.27 & 0.77\\
 & label-closed share, intersection reading & 0.31 & 0.04 & 0.00 & 0.00 & 0.00\\
\bottomrule
\end{tabular}}
\end{table}

\subsubsection{How to read a node}\label{app:fca_llm_readings}
Classical FCA reads the node of a concept as the intersection of all neurons that cover it. On real data, this intersection often excludes much of the concept itself. The covers include dense neurons and neurons that miss part of the letter, and their intersection omits many of the letter's tokens. As a result, $99\%$ (ReLU) and $84\%$ (Top-$K$) of the letters become self-misses: their node no longer holds them. Table~\ref{tab:app_fca_llm_toy} shows how this happens on a toy example.

Taking the smallest cover as the node avoids this problem. Self-misses then fall to $2$--$4\%$ on all three real datasets of Table~\ref{tab:app_fca_llm_readings}, namely the first-letter task, Imagenette and Yelp. Under this single reading, a concept is label-closed exactly when it has a selective cover, a cover that holds no other concept. Real concepts often have one. Synthetic concepts, which are cut out by several hyperplanes, usually do not, so the single reading is preferable on real data but not on synthetic data.

A label-closed node is not necessarily equal to its concept. Of the 570 (SAE, letter) pairs that are label-closed under the single reading, only 50 have a node with a precision of at least $0.7$. Section~\ref{app:fca_llm_human} studies what the other nodes are.

\begin{table}[!htbp]
\centering
\caption{A toy example. Neurons $a$, $b$ and $c$ cover $C$. Their intersection $\{t_2,t_3,t_4\}$ holds only $60\%$ of $C$ (a self-miss), while the smallest cover $c$ has recall $0.8$ and precision $1.0$.}
\label{tab:app_fca_llm_toy}
\small
\begin{tabular}{lccccc}
\toprule
token & in $C$ & $a$ & $b$ & $c$ & $d$\\
\midrule
$t_1$ & yes & $\bullet$ & $\bullet$ & & \\
$t_2$ & yes & $\bullet$ & $\bullet$ & $\bullet$ & \\
$t_3$ & yes & $\bullet$ & $\bullet$ & $\bullet$ & \\
$t_4$ & yes & $\bullet$ & $\bullet$ & $\bullet$ & \\
$t_5$ & yes & $\bullet$ & & $\bullet$ & $\bullet$\\
$t_6$ & no & $\bullet$ & $\bullet$ & & $\bullet$\\
$t_7$ & no & $\bullet$ & & & $\bullet$\\
$t_8$ & no & $\bullet$ & & & \\
\midrule
share of $C$ covered & & $100\%$ & $80\%$ & $80\%$ & $20\%$\\
\bottomrule
\end{tabular}
\end{table}

\begin{table}[!htbp]
\centering
\caption{FCA states under the intersection reading (left of the arrow) and the single reading (right of the arrow), $r=0.7$.}
\label{tab:app_fca_llm_readings}
\small
\begin{tabular}{llrcccc}
\toprule
Dataset & SAE & SAEs & label-closed & self-miss & merged & empty\\
\midrule
Gemma first letter & Top-$K$ & 18 & $0.05\to0.59$ & $0.84\to0.03$ & $0.06\to0.24$ & 0.04\\
 & ReLU & 18 & $0.00\to0.62$ & $0.99\to0.04$ & $0.01\to0.29$ & 0.00\\
DINOv3 / Imagenette & Top-$K$ & 77 & $0.14\to0.26$ & $0.38\to0.03$ & $0.00\to0.18$ & 0.49\\
 & ReLU & 63 & $0.02\to0.22$ & $0.59\to0.03$ & $0.00\to0.31$ & 0.39\\
Yelp & Top-$K$ & 78 & $0.03\to0.10$ & $0.44\to0.04$ & $0.12\to0.40$ & 0.38\\
 & ReLU & 63 & $0.01\to0.16$ & $0.43\to0.02$ & $0.00\to0.19$ & 0.56\\
Synthetic disjoint & Top-$K$ & 1,450 & $0.48\to0.34$ & $0.29\to0.00$ & $0.16\to0.34$ & 0.01\\
 & ReLU & 900 & $0.31\to0.18$ & $0.14\to0.00$ & $0.18\to0.37$ & 0.26\\
Synthetic overlap & Top-$K$ & 1,450 & $0.20\to0.13$ & $0.39\to0.01$ & $0.30\to0.68$ & 0.01\\
 & ReLU & 900 & $0.14\to0.06$ & $0.16\to0.00$ & $0.27\to0.54$ & 0.31\\
\bottomrule
\end{tabular}
\end{table}

\subsubsection{Why the parent misses: the encoder}\label{app:fca_llm_encoder}
We look at the tokens on which a child neuron fires and the parent neuron does not, and ask why the parent does not fire. Each SAE neuron first computes a pre-activation $z$, a linear function of the residual stream. A ReLU neuron fires when $z$ is positive, and a Top-$K$ neuron fires only when its $z$ is also among the $K$ largest on that token. So either the parent's gate is closed, meaning that its $z$ is not positive, or its gate is open but its $z$ is not among the $K$ largest.

In Top-$K$ SAEs, the second case dominates. At $K=20$, the parent's $z$ is positive on $85$--$96\%$ of these tokens, while a random neuron's $z$ is positive on only $2$--$12\%$ (Figure~\ref{fig:app_fca_llm_encoder}). The parent's gate is thus usually open. The parent does not narrowly miss the cut-off, however. Except at $K=20$, the parent ranks far below the cut-off: at $K\ge80$, its median rank is between 500 and 2,900. On these tokens, the parent's $z$ is among the top $1$--$4\%$ of all neurons, but far weaker than the $K$ neurons that are selected.

In ReLU SAEs, the gate is closed by definition, since a ReLU neuron that does not fire has $z\le0$. Even so, the parent's $z$ still ranks among the top $0.3$--$3\%$ of all neurons.

Geometry does not explain the nesting either. The encoder rows of a child neuron and its parent are nearly orthogonal, so their directions cannot explain why one region lies inside the other (see also Section~\ref{app:fca_llm_model}).

\begin{figure}[!htbp]
\centering
\includegraphics[width=0.85\linewidth]{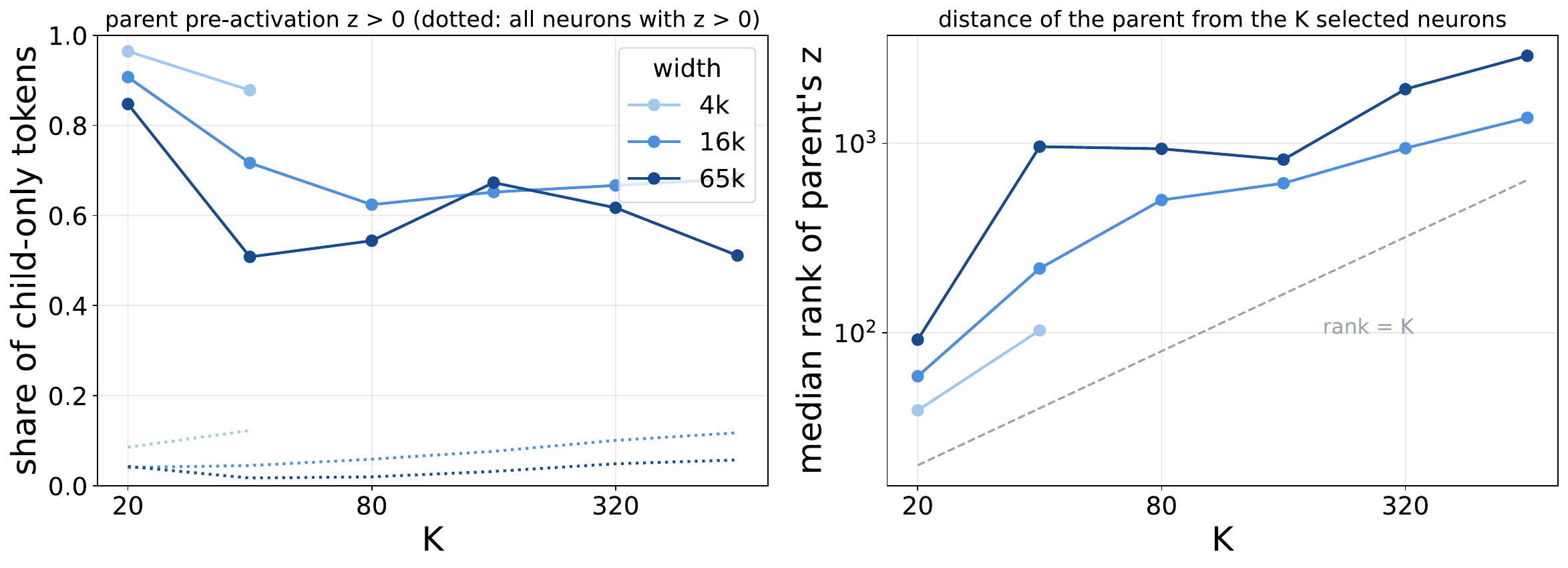}
\caption{Top-$K$ SAEs, tokens on which only a child neuron fires. Left: share of tokens on which the parent's pre-activation $z$ is positive (dotted: all neurons). Right: median rank of the parent's $z$ (dashed: rank $=K$).}
\label{fig:app_fca_llm_encoder}
\end{figure}

\subsubsection{Human-defined concepts}\label{app:fca_llm_human}
The node of a letter is rarely the letter itself. Over the 936 (SAE, letter) pairs, 36 SAEs times 26 letters, only $5\%$ have $\theta_C\approx C$ (Top-$K$: $9\%$; ReLU: $2\%$; Figure~\ref{fig:app_fca_llm_classes}), and only a few letters ever reach it, mainly b, h, f, g and a. The most common case is $\theta_C\supsetneq C$ ($55\%$): the node holds almost all of the letter (median recall $0.97$) but is $2.7$ times as large. This is not an artifact of choosing the smallest cover. If we instead take the neuron with the best F1 for the letter, $63\%$ of the pairs get the same neuron, and $\theta_C\approx C$ still holds for only $5\%$ (Figure~\ref{fig:app_fca_llm_best_classes}). Where the two choices differ, the best-F1 neuron covers only a small part of the letter or is unrelated to it. In none of these SAEs does a single neuron match a letter exactly.

The extra tokens have one property in common: the letter appears near the start. They are strongly enriched for having the letter in second place (Table~\ref{tab:app_fca_llm_diff}); for example, the node of b also holds \texttt{getBean} and \texttt{listBox}. In a direct count, every letter's node holds $25$--$90\%$ of the tokens that have the letter in second but not in first place, but only $3$--$27\%$ of the tokens that have it in neither place (Figure~\ref{fig:main_fca_llm}d). We also computed the node of the concept ``the letter is among the first two letters'' to see whether it coincides with the letter's node. It often does, but this is partly automatic: when the second-place tokens add only a little to the letter, the letter's node already covers the larger concept (Figure~\ref{fig:app_fca_llm_letter_closure}). The direct count is therefore the main evidence that the node means ``the letter near the start''.

Within the letter, the boundary of the node follows the language of the token. The letter tokens that a node misses are mostly non-English words: the node of b misses Indonesian \emph{ber-} verbs, and the node of m misses Indonesian \emph{meN-} verbs. These tokens are not absorbed by another neuron; the model itself encodes their first letter weakly.

No simple description matches the node either. For nodes with $\theta_C\supsetneq C$, we compare four simple descriptions (``starts with the letter'', ``the letter is among the first two letters'', ``the letter is among the first three letters'' and ``contains the letter'') with the rule written by the language model. All of them reach only F1 $0.55$--$0.65$ against the node on the whole vocabulary (Figure~\ref{fig:app_fca_llm_descriptions}). Held-out tests hide this gap. On a held-out sample with equally many tokens from each group, ``contains the letter'' and the rules scored much higher than the letter itself ($0.78$ and $0.73$ versus $0.56$), because such a sample does not penalize broad descriptions. Round trips are also unreliable. In a round trip, we treat the tokens that a rule accepts as a new concept and check whether its node is the node that the rule describes. The same rules return to their node in $26\%$ of cases when the node is the smallest cover, and in $66\%$ when it is the neuron with the best F1.

So the SAE learns a concept that is close to the human one, roughly ``the letter near the start'', but not equal to it, and label-closed rates and held-out scores hide this gap. Explanations should therefore be scored on the whole vocabulary, and round trips should state how the node is read. Whether a node appears polysemantic also depends on which human concepts we compare it with.

\begin{figure}[!htbp]
\centering
\begin{minipage}[t]{0.54\linewidth}
\centering
\includegraphics[width=\linewidth]{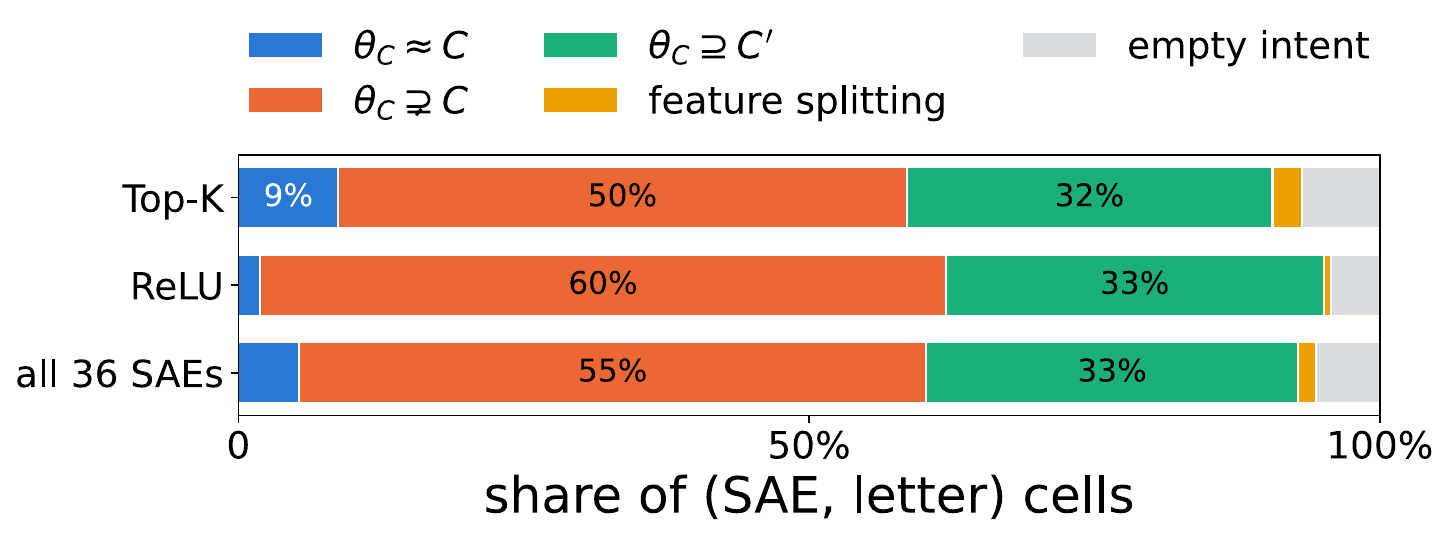}
\caption{Share of (SAE, letter) pairs in each class of the letter's node.}
\label{fig:app_fca_llm_classes}
\end{minipage}\hfill
\begin{minipage}[t]{0.43\linewidth}
\centering
\includegraphics[width=\linewidth]{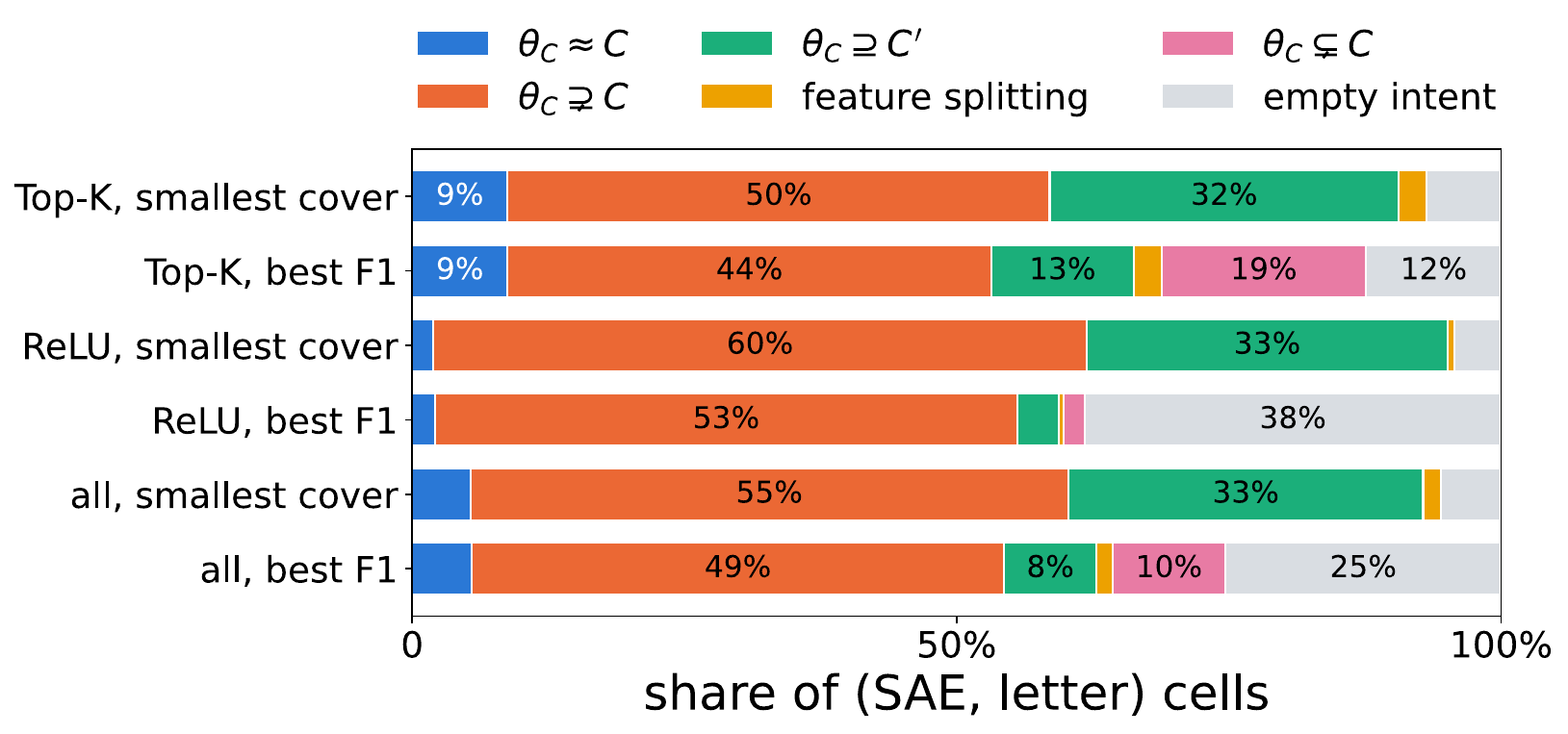}
\caption{Classes under the smallest-cover and the best-F1 choices of the node; $\theta_C\subsetneq C$ means a neuron that covers only a small part of the letter.}
\label{fig:app_fca_llm_best_classes}
\end{minipage}
\end{figure}

\begin{table}[!htbp]
\centering
\caption{Tokens that differ between the node and the letter (medians). A ratio is the share of a property among the extra tokens divided by its share among all other tokens.}
\label{tab:app_fca_llm_diff}
\small
\resizebox{\textwidth}{!}{\begin{tabular}{lcccccc}
\toprule
 & \multicolumn{3}{c}{extra tokens} & \multicolumn{3}{c}{missing tokens}\\
\cmidrule(lr){2-4}\cmidrule(lr){5-7}
Class & letter after 1st & 2nd letter & probe AUC & share of letter & probe acc. (node) & probe acc. (missing)\\
\midrule
$\theta_C\approx C$ & $5.2\times$ & $23.5\times$ & 0.97 & 0.18 & 0.99 & 0.81\\
$\theta_C\supsetneq C$ & $4.3\times$ & $9.5\times$ & 0.93 & 0.03 & 0.97 & 0.36\\
$\theta_C\supseteq C'$ & $1.0\times$ & $1.0\times$ & 0.54 & 0.22 & 0.97 & 0.93\\
empty intent & $1.0\times$ & $1.0\times$ & 0.53 & 0.30 & 0.96 & 0.94\\
\bottomrule
\end{tabular}}
\end{table}

\begin{figure}[!htbp]
\centering
\begin{minipage}[t]{0.415\linewidth}
\centering
\includegraphics[width=\linewidth]{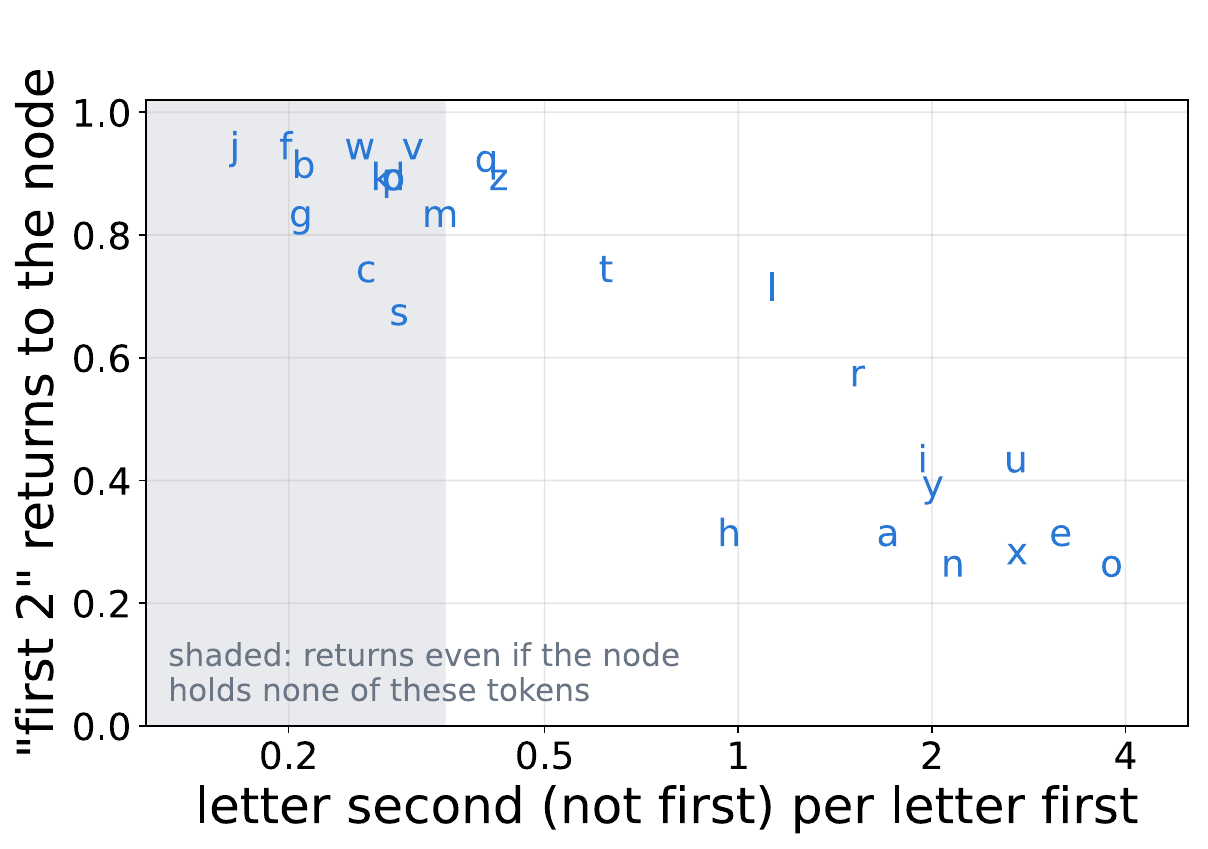}
\caption{Whether ``the letter is among the first two letters'' returns to the letter's node, against how much the second-place tokens enlarge the letter. In the shaded range, a return is automatic.}
\label{fig:app_fca_llm_letter_closure}
\end{minipage}\hfill
\begin{minipage}[t]{0.555\linewidth}
\centering
\includegraphics[width=\linewidth]{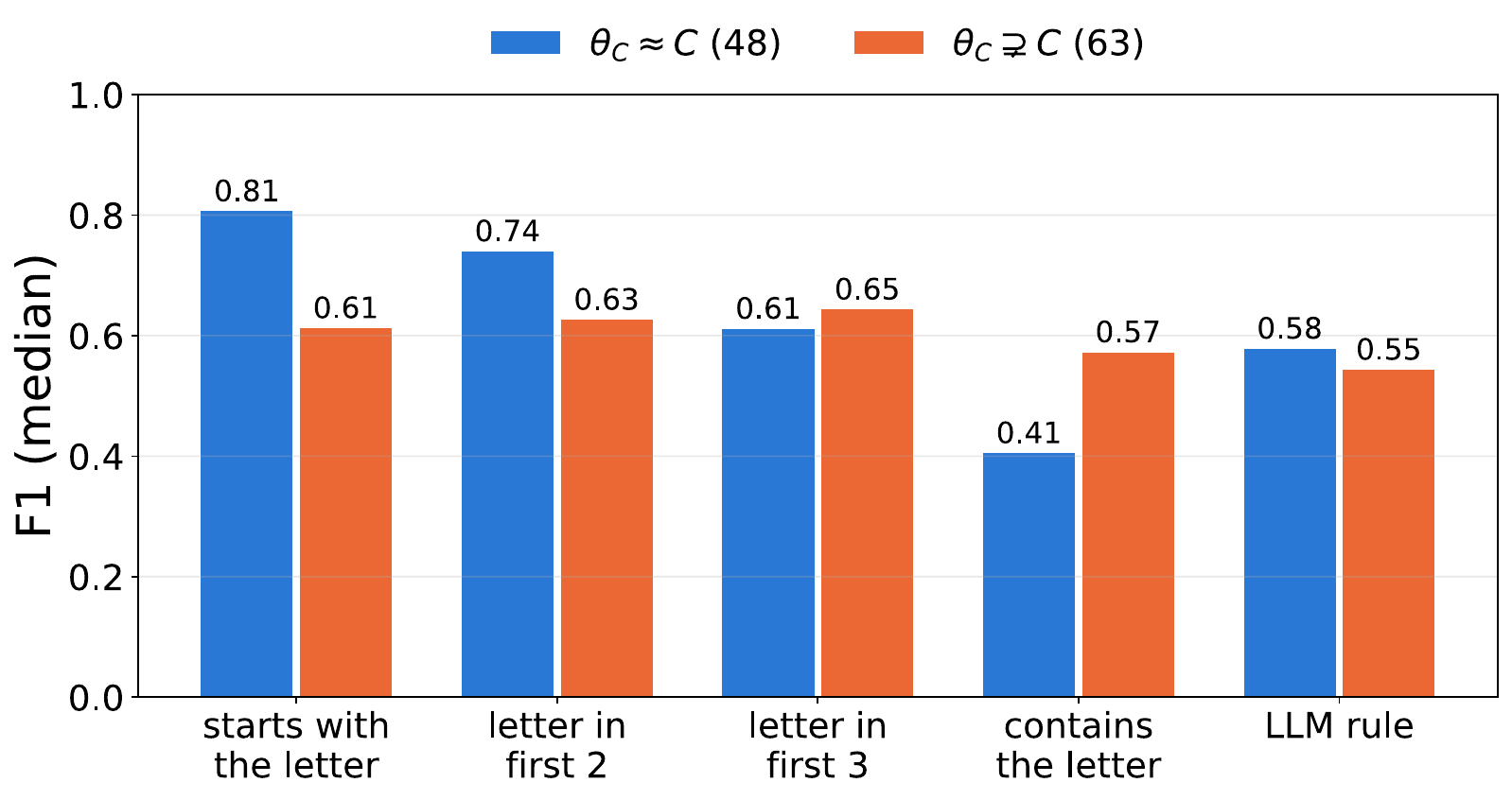}
\caption{F1 of five descriptions of a letter's node on the whole vocabulary (medians).}
\label{fig:app_fca_llm_descriptions}
\end{minipage}
\end{figure}

\subsubsection{Model-induced concepts}\label{app:fca_llm_model}
We now start from the neurons instead of the human concepts. We place in one lattice the 269,580 neurons of all 36 SAEs that fire on at least 30 training tokens and are not dense, ordered as in Section~\ref{app:fca_llm_setup}. As for concepts, each neuron gets the first state that applies: merged, if it and another neuron lie below each other; parent, if some other neuron lies below it; child, if it lies below another neuron; and isolated otherwise. From here on, a child and its upper neuron always refer to this order, not to the parent and child neurons of a letter.

A neuron's position in the lattice carries little information about its meaning. $89\%$ of the neurons match no letter, prefix or form, where a form is a leading space or a capital first letter. A rule usually describes only the core of a neuron: $79\%$ of random neurons get a rule, but only $26\%$ pass the held-out test. Passing means a balanced accuracy of at least $0.75$, where balanced accuracy is the average of the share of inside tokens that the rule accepts and the share of outside tokens that it rejects. Once we match neurons in size, children, parents and isolated neurons are about equally easy to explain (Table~\ref{tab:app_fca_llm_states}). The order does not reflect the encoder geometry either. A child neuron and its upper neuron have nearly orthogonal encoder rows, and the child lies inside its upper neuron because of how the data are distributed, not because of the encoder directions (Figure~\ref{fig:app_fca_llm_geometry}).

The size of the upper neuron determines whether a child and its upper neuron form a hypernym pair, that is, whether the child is a kind of its upper neuron. For each such pair, we write one rule for the child and one for the upper neuron, and test both on held-out tokens. The upper neuron counts as a hypernym of the child if the upper rule accepts the child's tokens and the child rule rejects the rest of the upper neuron. When the upper neuron is specific, firing on under $1\%$ of tokens, $70\%$ of the pairs pass; for example, the neurons for Reuters and for CNN lie under a neuron for news media. When the upper neuron is wide, firing on over $10\%$ of tokens, only $3\%$ pass, and control pairs, in which the upper neuron is replaced by a neuron of the same size, never pass (Table~\ref{tab:app_fca_llm_hierarchy}). Since $87\%$ of children lie under a wide neuron, most inclusions in the lattice are not hypernyms. The smallest specific upper neuron, however, gives a reliable hierarchy without any labels.

Meets of neurons create new concepts. The meet $\theta_A\cap\theta_B$ of two neurons $A$ and $B$ is the set of tokens on which both fire. We study 120 pairs of neurons that fire together at least $5$ (or $10$) times as often as chance, that is, on at least 5 (or 10) times as many tokens as expected if they fired independently, on both training and test tokens. Their meet is almost never an existing neuron, yet a rule describes it with F1 above $0.7$ for $31\%$ of the pairs. The two differences, such as $\theta_A\setminus\theta_B$, the tokens on which $A$ fires but $B$ does not, behave differently. Their rules are at chance (balanced accuracy $0.50$), and they accept the shared tokens more often than their own tokens ($0.34$--$0.41$ versus $0.27$; Figure~\ref{fig:main_fca_llm}a). We call this the combinatorial expressivity of FCA meets: two neurons can share a meaning that neither of them shows on its own, and their meet makes this shared meaning explicit. This is a label-free instance of Theorem~\ref{theorem:unit_separation_condition}, since an intersection of neurons isolates a set that no single neuron isolates. A language model looked at the tokens of each neuron and judged whether the meet captures what the two neurons share. By its judgement, the 37 well-described meets fall into three kinds:
\begin{itemize}
\item \emph{True conjunctions} (14 of 37). The meet is the evident common meaning of the two neurons. For example, the meet of a neuron for names of ancient peoples, rulers and places (Mesopotamia, Hadrian, Pompeii) and a neuron for words meaning outdated or old (archaic, vintage, legacy) consists of words about ancient times (``ancient, fossil, antiquity'').
\item \emph{Partial} (16 of 37). The meet captures only part of what the two neurons share.
\item \emph{Novel concepts} (7 of 37). The language model finds nothing in common between the two neurons, yet their meet is a clean concept. A neuron for forms of \emph{add} and its translations and a neuron for verbs of arranging and sorting have a meet that consists of verbs for \emph{add}, \emph{insert}, \emph{attach}, \emph{remove} or \emph{delete} (F1 $0.94$). Another pair has a meet that consists of chemical elements, metals and minerals (F1 $0.78$).
\end{itemize}
At the other extreme, 6 of the 120 meets have no describable pattern at all, even when the two neurons fire together $18$ times as often as chance. For most pairs, the meet is the most typical part of both neurons and the differences are their peripheral parts, so the differences are not concepts (Table~\ref{tab:app_fca_llm_region_examples}). A language model's judgement that a meet ``captures what $A$ and $B$ share'' does not predict the held-out F1 (judged ``yes'' for $38\%$ of good meets and $37\%$ of the others), so such judgements need a held-out check.

The lattice recovers the meet as a set of tokens, but not its meaning. The multi-neuron reading of the node of a meet recovers the meet exactly. To test its meaning, we take the rule written for the meet, treat the tokens that the rule accepts as a concept $C$, and find the node $D$ of $C$. This concept is much broader than the meet, and its node is broader still: only $19\%$ of $D$ belongs to $C$ (Table~\ref{tab:app_fca_llm_b3_sets}). In 24 of 37 pairs, $D$ is a third and much larger neuron, and half of those that the language model judged (12 of 23) mean something unrelated (Figure~\ref{fig:app_fca_llm_closure}). For example, one meet is described as farming terms, but the node of this concept is a neuron for everyday German, French, Spanish and Portuguese nouns.

In short, hierarchies derived from the lattice should be reported together with the size of the upper neuron. Meets of neurons can express concepts that no single neuron holds, but the differences of two neurons, and the node of a description, should not be interpreted as features.

\begin{table}[!htbp]
\centering
\caption{Held-out test of neurons by their position in the lattice. Left: size-matched sample. Right: random sample.}
\label{tab:app_fca_llm_states}
\resizebox{\textwidth}{!}{
\begin{tabular}{l rccccc rrcc}
\toprule
 & \multicolumn{6}{c}{size-matched (60 per state)} & \multicolumn{4}{c}{random sample}\\
\cmidrule(lr){2-7}\cmidrule(lr){8-11}
State & tokens & bal.\ acc. & rule F1 & explains & no rule & defined-concept F1 & neurons & tokens & bal.\ acc. & explains\\
\midrule
child & 3,159 & 0.54 & 0.35 & 22\% & 12\% & 0.55 & 78 & 228 & 0.62 & 33\%\\
isolated & 3,144 & 0.54 & 0.29 & 7\% & 25\% & 0.47 & 64 & 622 & 0.56 & 20\%\\
parent & 3,150 & 0.62 & 0.40 & 22\% & 13\% & 0.43 & 8 & 4,866 & 0.54 & 0\%\\
\bottomrule
\end{tabular}}
\end{table}

\begin{figure}[!htbp]
\centering
\includegraphics[width=0.85\linewidth]{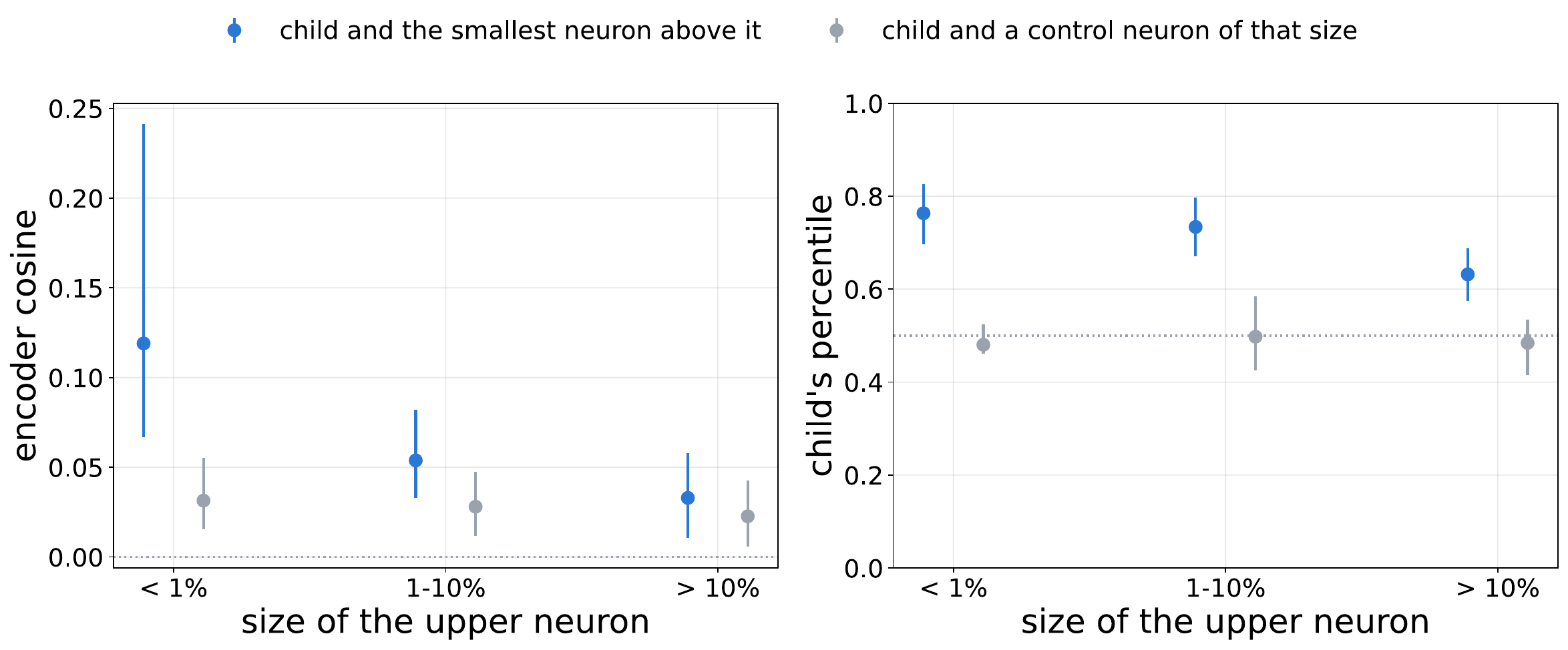}
\caption{Encoder cosine (left) and the child's percentile within its upper neuron (right), for child--upper pairs (blue) and size-matched controls (grey).}
\label{fig:app_fca_llm_geometry}
\end{figure}

\begin{table}[!htbp]
\centering
\caption{Two-way held-out test of child--upper neuron pairs, by the size of the upper neuron (mean acceptance).}
\label{tab:app_fca_llm_hierarchy}
\resizebox{\textwidth}{!}{
\begin{tabular}{lr ccc cc c cc}
\toprule
 & & \multicolumn{3}{c}{upper rule accepts} & \multicolumn{2}{c}{child rule accepts} & & \multicolumn{2}{c}{control}\\
\cmidrule(lr){3-5}\cmidrule(lr){6-7}\cmidrule(lr){9-10}
Upper neuron & pairs & child & upper only & neither & child & upper only & semantic hierarchy & control rule accepts child & semantic hierarchy\\
\midrule
$<1\%$ & 40 & 0.81 & 0.35 & 0.02 & 0.89 & 0.22 & \textbf{70\%} & 0.02 & 0\%\\
$1$--$10\%$ & 37 & 0.54 & 0.31 & 0.03 & 0.65 & 0.21 & 32\% & 0.01 & 0\%\\
$>10\%$ & 37 & 0.23 & 0.17 & 0.10 & 0.33 & 0.03 & 3\% & 0.05 & 0\%\\
\bottomrule
\end{tabular}}
\end{table}

\begin{table}[!htbp]
\centering
\caption{The best-described meets, with the rules for the two differences (paraphrased; balanced accuracy in brackets).}
\label{tab:app_fca_llm_region_examples}
\small
\begin{tabular}{p{0.3\linewidth}p{0.3\linewidth}p{0.3\linewidth}}
\toprule
$\theta_A\cap\theta_B$ & $\theta_A\setminus\theta_B$ & $\theta_B\setminus\theta_A$\\
\midrule
Verbs meaning add, insert, attach or delete, in several languages ($0.97$) & Past-tense, third-person and other forms of add, and its translations ($0.44$) & English action verbs for arranging, sorting, assigning or presenting ($0.50$)\\
Female first names, often with a leading space ($0.94$) & Mostly English female names, especially spellings of Sarah/Sara ($0.38$) & Person names, mostly female and Romance-style ($0.38$)\\
Words about ancient, prehistoric or old times: ancient, fossil, antiquity ($0.94$) & Names of ancient peoples, rulers, places and eras: Mesopotamia, Hadrian, Pompeii ($0.69$) & Words meaning outdated or old: archaic, vintage, legacy ($0.62$)\\
Words meaning popular or famous, in several languages ($0.94$) & The common/usual/frequent family ($0.50$) & Praising words for ``famous, outstanding'': legendary, acclaimed ($0.41$)\\
\bottomrule
\end{tabular}
\end{table}

\begin{figure}[!htbp]
\centering
\includegraphics[width=\linewidth]{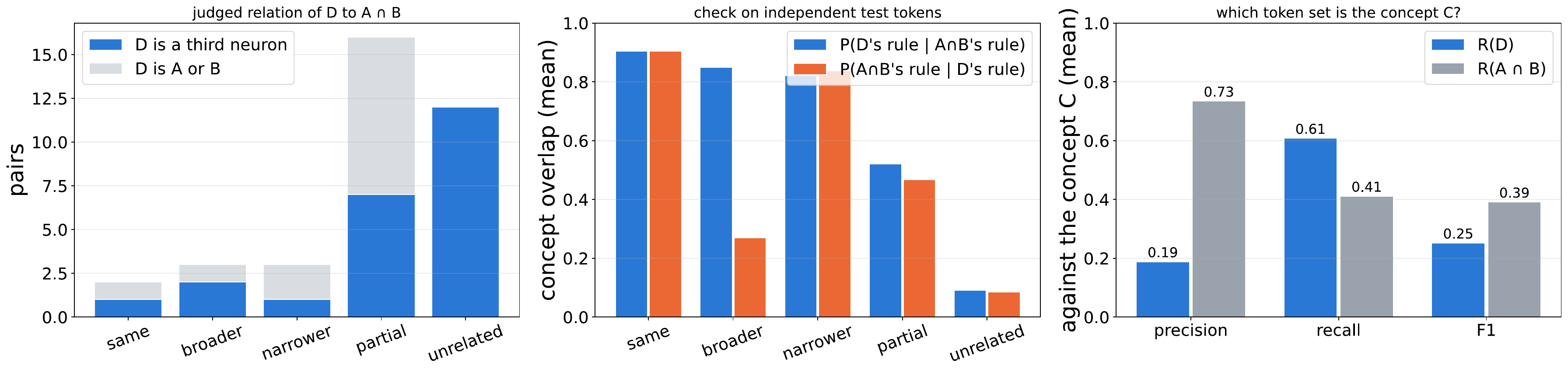}
\caption{The concept $C$ described by a meet's rule and its node $D$: judged relation (left), overlap on independent tokens (middle), and fit of $D$ and of the meet to $C$ (right).}
\label{fig:app_fca_llm_closure}
\end{figure}

\begin{table}[!htbp]
\centering
\caption{Fit of the node $D$ and of the meet to the concept $C$ (mean / median).}
\label{tab:app_fca_llm_b3_sets}
\small
\resizebox{\textwidth}{!}{\begin{tabular}{lcccc}
\toprule
Token set & precision (share in $C$) & recall (share of $C$) & F1 & size (median, tokens)\\
\midrule
$R(D)$: the node of $C$ & 0.19 / 0.14 & 0.61 / 0.81 & 0.25 / 0.22 & 9,798 (third neuron)\\
$R(\theta_A\cap\theta_B)$: the meet & 0.73 / 0.75 & 0.41 / 0.34 & 0.39 / 0.50 & 110\\
\bottomrule
\end{tabular}}
\end{table}

\subsubsection{Feature absorption}\label{app:fca_llm_absorption_case}
Absorbing neurons are not child concepts. We call a neuron absorbing if it carries the letter's direction in place of the parent neurons, in the sense of Section~\ref{app:fca_llm_absorption}, on at least 3 test tokens. Absorption is usually interpreted as a child feature that absorbs its parent, so the absorbing neuron should lie inside the parent. FCA lets us check this directly, because it places both neurons in the same lattice (Figure~\ref{fig:app_fca_llm_schematic}). The 31 SAEs that have absorbing neurons have 7,067 of them in total. None of them lies below its parent neuron, and $99.99\%$ are incomparable with it. About half of their tokens lie outside the letter (Figure~\ref{fig:app_fca_llm_composition}).

As an example, in Top-$K$ 4k ($L_0=40$), the node of s is its parent neuron 3477, with recall $0.77$ and precision $0.76$. It misses 682 test tokens of s. Of these, 486 are absorbed, 306 of them by a single neuron, 2108 (Table~\ref{tab:app_fca_llm_token_example}). Neuron 2108 fires on 4,379 tokens, and only $56\%$ of them start with s. The rest contain ``st'' elsewhere, as in \texttt{estanding}, \texttt{isStatic} and \texttt{RestController}, so it is a neuron for the letter sequence ``st''. The two concepts, ``s near the start'' and ``st anywhere'', overlap on tokens that start with st, such as \texttt{stal}, \texttt{stasis} and \texttt{STOR}. On these tokens, the SAE activates the more specific st neuron instead of the parent. Neither concept contains the other.

Absorption is spread over many neurons. Absorbing neurons are mostly words or word stems, often together with their translations (home, hogar, hjem, zuhause; love, Liebe, aime, amor), or letter sequences such as st. Each one absorbs only about $1\%$ of the tokens that the parent misses. In Top-$K$ 65k at $L_0=40$, a letter has about 79 absorbing neurons, and 26 of them together account for half of the absorbed tokens. All absorbing neurons of a letter together account for $50$--$71\%$ of the misses (Top-$K$ 16k and 65k, $L_0\le80$).

Not every miss is absorption. A missed token counts as absorbed only if another neuron fires on it in place of the parent. Across whole SAEs, the share of a letter's tokens that the parent misses while another letter-enriched neuron fires, one that does not lie below the parent, is strongly correlated with the absorption rate ($r=0.95$), while the plain miss rate does not ($r=0.15$; Figure~\ref{fig:main_fca_llm}b). Here a neuron is letter-enriched if the letter is at least five times as common among its tokens as among all tokens. The difference comes from low-$L_0$ ReLU SAEs, which have no neuron for the letter and miss many tokens on which no other letter-enriched neuron fires. Weak encoding is another source of misses. Of the Indonesian \emph{ber-} and \emph{meN-} verbs that the probe classifies correctly but the parent misses, only 1 of 84 and 6 of 243 are absorbed, because the model encodes their first letter weakly. Overall, only $19\%$ of the tokens missed by a letter's node are absorbed. For single neurons, the lattice structure is no more informative than precision for the letter, which already ranks absorbing neurons well.

In short, absorption is an overlap of two incomparable model concepts, not a child inside its parent: on the shared tokens, the SAE uses the more specific concept. The absorption rate is defined relative to a human parent label, and misses on which no other neuron fires are not absorption.

\begin{figure}[!htbp]
\centering
\begin{minipage}[t]{0.47\linewidth}
\centering
\resizebox{\linewidth}{!}{\begin{tikzpicture}[scale=0.85]
\fill[gray!35] (2.4,2.1) ellipse (2.6 and 1.44);
\fill[orange!35] (6.6,1.84) ellipse (2.2 and 1.16);
\draw[thick,dashed] (3.5,2.1) ellipse (2.6 and 1.6);
\node[font=\bfseries] at (0.6,3.0) {$P$};
\node[font=\scriptsize,align=left] at (1.2,2.6) {parent region};
\node[font=\bfseries] at (8.2,2.4) {$S$};
\node[font=\scriptsize] at (7.7,2.0) {absorbing neuron};
\node[font=\bfseries] at (3.5,4.0) {$C$};
\node[font=\scriptsize] at (5.1,4.0) {letter (starts with $L$)};
\node[font=\scriptsize,align=left] at (5.3,1.6) {$C\cap S$, $P$ inactive\\absorbed tokens};
\node[font=\scriptsize,align=left] at (7.6,1.2) {$S\setminus C$\\cross-label part};
\end{tikzpicture}}
\caption{The sets involved: the letter $C$, the parent region $P$, and the region $S$ of one absorbing neuron.}
\label{fig:app_fca_llm_schematic}
\end{minipage}\hfill
\begin{minipage}[t]{0.5\linewidth}
\centering
\includegraphics[width=\linewidth]{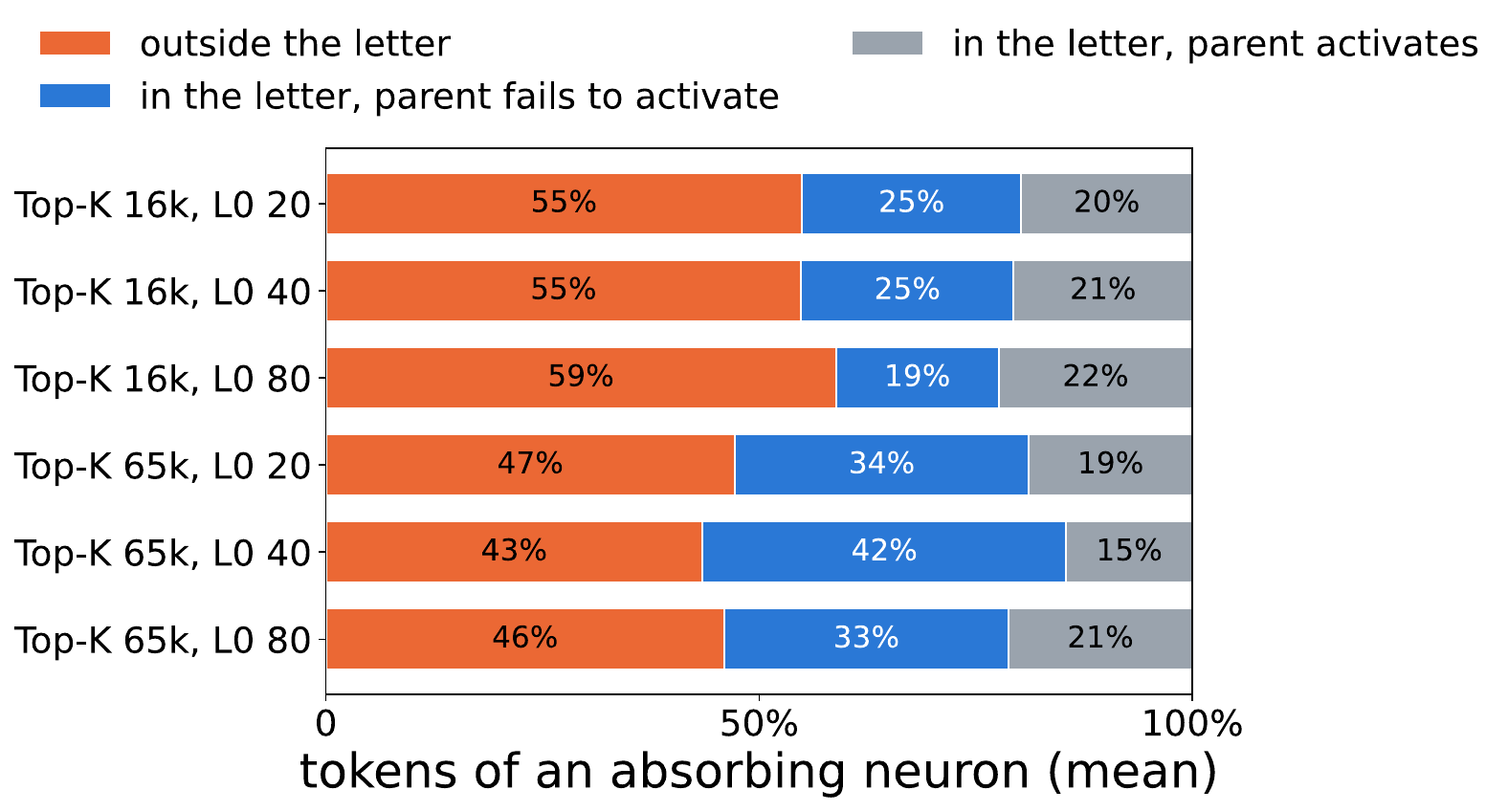}
\caption{Where the tokens of an absorbing neuron lie (the 6 Top-$K$ SAEs with the most absorption).}
\label{fig:app_fca_llm_composition}
\end{minipage}
\end{figure}

\begin{table}[!htbp]
\centering
\caption{Tokens of s absorbed by neuron 2108 (Top-$K$ 4k, $L_0=40$). The parent neuron 3477 fires on none of them.}
\label{tab:app_fca_llm_token_example}
\small
\begin{tabular}{lccccc}
\toprule
Token & activation of parent 3477 & in the node & activation of 2108 & cos & share\\
\midrule
\texttt{stal} & 0 & no & 22.16 & 0.476 & 2.05\\
\texttt{stasia} & 0 & no & 23.58 & 0.476 & 1.46\\
\texttt{stens} & 0 & no & 20.72 & 0.476 & 1.39\\
\texttt{stes} & 0 & no & 20.53 & 0.476 & 1.29\\
\texttt{\_stereotyp} & 0 & no & 21.66 & 0.476 & 1.26\\
\texttt{\_Stornier} & 0 & no & 16.36 & 0.476 & 1.25\\
\texttt{stasis} & 0 & no & 24.52 & 0.476 & 1.15\\
\texttt{STOR} & 0 & no & 21.64 & 0.476 & 1.14\\
\bottomrule
\end{tabular}
\end{table}

\subsubsection{Structure-preserving mapping}\label{app:fca_llm_mapping}
Finally, we map a whole system of human concepts into the lattice. It consists of 1,628 concepts with at least 30 training tokens each (Table~\ref{tab:app_fca_llm_concepts}): concepts read from the spelling of a token, such as letters, prefixes and forms, and 16 hand-written English word lists, such as weekdays and months. We map each concept $C$ to its node $\theta_C$. If this map were close to an isomorphism between human and model semantics, that is, a one-to-one map that keeps every relation, it would keep three structures: inclusion (order), intersections (meets) and round trips, which we define below.

Some of this structure is guaranteed by the mathematics and is therefore not a finding. With the exact closure of classical FCA, where a neuron counts in the intent of $C$ only if it fires on every token of $C$ ($r=1$), inclusion between concepts is kept by construction. This holds only trivially, however, because the exact closure maps almost every concept to the top of the lattice: in $96\%$ of the nested pairs, the larger concept has no node. At $r=0.7$ nothing is guaranteed, so the shares below are empirical findings.

With the single reading, inclusion is only weakly preserved. For a nested pair $C_1\subset C_2$ in which $C_2$ has a node, $C_1$ still lies in that node in $78$--$98\%$ of cases ($96\%$ for prefixes). The distortion is in the node of $C_1$ itself. Of 58,104 nested pairs, 1,614 pairs in each of the 36 SAEs, only $0.8\%$ keep $\theta_{C_1}\subset\theta_{C_2}$ (Table~\ref{tab:app_fca_llm_order}). Most are merged ($39\%$): $C_1$ has no node of its own and falls into the node of $C_2$, so the order is preserved but the distinction is lost. Many are not nested ($29\%$), because $C_1$ has a smaller cover elsewhere. This is not an artifact of the method: if we replace $C_1$ by a random subset of $C_2$ of the same size, the pair is not nested in only $1.4\%$ of cases, against $27\%$ for prefixes.

Pairs that are not nested arise from what the nodes mean. A letter's node means ``the letter near the start'', so a prefix lies in the node of its first letter and often also in the node of its second letter ($96\%$ and $44\%$). The smaller of the two becomes the prefix's node (Figure~\ref{fig:app_fca_llm_explain}). For example, tw $\subset$ t takes the node of w, which is smaller than the node of t, so tw lies under w and not under t (Table~\ref{tab:app_fca_llm_explain_examples}). The human hierarchy is a tree, but the SAE's is not: in the SAE, a prefix has two parents.

Semantic hierarchies break in the same way, so the pattern is not limited to spelling. Word lists such as weekdays and months have their own neurons, but these neurons lie outside the node of their group, such as ``time'': $60\%$ of list--group pairs are not nested, against $13\%$ for random subsets. Word lists are also the only kind with a sizable share of truly nested nodes ($12\%$). Distinct concepts, in contrast, are rarely confused: two different letters share a node in only $0.3\%$ of cases, and two different word lists in $1.1\%$.

Intersections are not preserved either. A conjunction is the intersection of two human concepts, such as ``starts with s'' and ``has a leading space''. When both concepts have a node, the node of their conjunction equals the meet of the two nodes in only $0$--$4\%$ of cases. No single neuron represents a conjunction (F1 above $0.7$ in at most $0.3\%$ of cases; Table~\ref{tab:app_fca_llm_meet}). When a conjunction is represented at all, it is represented by the meet of two neurons: the meet of the node of the leading space and the node of capitalization reaches F1 above $0.7$ for $36\%$ of SAEs. This is the same combinatorial expressivity as in Section~\ref{app:fca_llm_model}. A good meet requires two good nodes: for a letter combined with a form, the F1 of the meet has a rank correlation of $0.97$ with the weaker of the two node F1s.

Round trips are rarely preserved. In a round trip, we map a human concept to its node and then map the node back to the smallest human concept that holds $70\%$ of the node; concepts that cover more than half of all tokens are not candidates. Only $0.3\%$ of the concepts return to themselves, and letters return in $6\%$ of cases. A round trip makes a concept broader ($37.5\%$) but never narrower ($0\%$). The concepts that do return form about 180 fixed points, and they pair up one to one with the neurons that return to themselves when we start from the SAE (for example, 57 letters and 57 letter neurons). These fixed points are the part of the map that is truly isomorphic. Letters return mainly within a window of $L_0$, for example about half of them in Top-$K$ 4k at $L_0=20$--$40$.

The reverse direction preserves more structure. Starting from the SAE, $70\%$ of the pairs with a specific upper neuron are hypernyms (Section~\ref{app:fca_llm_model}), and $46\%$ of the rules written for neurons return to their neuron (Table~\ref{tab:app_fca_llm_both}). Part of this asymmetry comes from the setup: in this direction, the language model can describe any set of tokens, while in the other direction the SAE offers only its existing neurons. So the SAE's own hierarchy can mostly be described in human terms, but the human hierarchy is mostly not in the SAE's nodes.

The distortion is semantic, not geometric. Below, $C_1$ is the child and $C_2$ the parent of a nested pair, and we ask what determines whether the pair is not nested rather than merged. Of three kinds of explanation, meaning explains the most: the node under which a child falls follows what the nodes mean, as shown above. Region size determines how often this happens: the share of not-nested pairs rises from $12\%$ to $63\%$ as the child gets smaller and $L_0$ gets larger. Geometry explains little. Not-nested children point somewhat away from their parent's direction, but this separates the two outcomes only moderately (AUC $0.29$, where $0.5$ means no separation; $0.38$ after controlling for size), and the child's position along the parent's direction does not separate them at all (AUC $0.52$). A parent node that misses most of its child occurs in only $1.5$--$22\%$ of the pairs, so it explains little.

FCA organizes rather than measures. Every number here is a count of overlaps between sets of tokens. What FCA adds is a place in the lattice for every concept, even one without a matching neuron, and this is what distinguishes merged pairs from pairs that are not nested. If we instead match each prefix to the neuron with the best F1 for it, we find only a small neuron that holds $15\%$ of the prefix, and the neurons found for a prefix and for its letter are unrelated in $78$--$89\%$ of cases. The hierarchy then appears to vanish, and the two cases can no longer be distinguished. FCA does not, however, identify absorption better than plain overlap.

In short, the SAE reorganizes the human hierarchy rather than copying it. Evaluations of SAEs against human hierarchies should report the shares of pairs that stay nested, merge or are not nested, compare them with a random control, and keep the two directions apart. Together, these results support lattice- or graph-structured SAEs (Appendix~\ref{app:neuron_interp}).

\begin{table}[!htbp]
\centering
\caption{The human concept system. All concepts are read from token strings, except 16 hand-written word lists.}
\label{tab:app_fca_llm_concepts}
\resizebox{\textwidth}{!}{
\begin{tabular}{llp{0.42\linewidth}rr}
\toprule
Kind & Definition & Example & concepts & tokens (median)\\
\midrule
letter & starts with the letter (as in Section~\ref{app:fca_llm_human}) & s & 26 & 6,043\\
second letter & the second letter is it & $\cdot$t & 26 & 3,384\\
prefix & letter $\wedge$ second letter & st & 319 & 236\\
3-prefix & the first three letters & str & 1,101 & 76\\
first two & the letter is among the first two letters & s$\cdot$, $\cdot$s & 26 & 11,575\\
contains & contains the letter & $\cdot\cdot$s$\cdot\cdot$ & 26 & 28,181\\
form & leading space, capitalized, and their negations & \texttt{\_word}, Word & 4 & 84,482\\
letter $\wedge$ form & starts with the letter and has a leading space (or is capitalized) & \texttt{\_s}\dots, S\dots & 52 & 2,611\\
space $\wedge$ capital & both forms hold & \texttt{\_Word} & 1 & 37,072\\
word list & weekdays, months, time units, cardinals, ordinals, countries, US states, cities, kinship, names, occupations, colors, animals, body parts, languages, chemical elements & Monday, \texttt{\_monday} & 16 & 131\\
word-list group & time, number, place, person: each contains several lists above & time $\supset$ weekdays, months & 4 & 349\\
word list $\wedge$ form & tokens of a list with a leading space (or capitalized) & \texttt{\_Monday} & 27 & 84\\
\bottomrule
\end{tabular}}
\end{table}

\begin{table}[!htbp]
\centering
\caption{Outcomes of nested pairs (single node, $r=0.7$, \%), with random-subset controls.}
\label{tab:app_fca_llm_order}
\small
\resizebox{\textwidth}{!}{\begin{tabular}{lrccccc}
\toprule
Nested pair $C_1\subset C_2$ & pairs & $\theta_{C_1}\subset\theta_{C_2}$ & merged & $\theta_{C_1}\supset\theta_{C_2}$ & not nested & empty intent\\
\midrule
prefix $\subset$ letter (st $\subset$ s) & 11,484 & 0.3 & 38.2 & 0.3 & 26.8 & 34.3\\
\quad control: random subset of the letter & 11,484 & 0.1 & 64.1 & 0.2 & 1.4 & 34.3\\
3-prefix $\subset$ prefix (str $\subset$ st) & 39,636 & 0.9 & 39.4 & 0.4 & 32.3 & 27.0\\
letter $\wedge$ form $\subset$ letter & 1,872 & 0.3 & 65.5 & 0.1 & 3.7 & 30.4\\
\quad control: random subset of the letter & 1,872 & 0.0 & 69.1 & 0.1 & 0.4 & 30.4\\
letter $\wedge$ form $\subset$ form & 1,872 & 0.9 & 9.1 & 0.5 & 34.0 & 55.6\\
letter $\subset$ first two & 936 & 0.6 & 44.9 & 0.1 & 4.9 & 49.5\\
first two $\subset$ contains & 936 & 0.2 & 18.3 & 0.1 & 4.9 & 76.5\\
word list $\subset$ group (weekdays $\subset$ time) & 396 & 12.4 & 20.7 & 0.5 & 59.8 & 6.6\\
\quad control: random subset of the group & 396 & 0.8 & 79.0 & 0.3 & 13.4 & 6.6\\
word list $\wedge$ form $\subset$ word list & 972 & 1.2 & 82.1 & 0.8 & 15.2 & 0.6\\
all (without controls) & 58,104 & 0.8 & 39.4 & 0.4 & 29.3 & 30.1\\
\bottomrule
\end{tabular}}
\end{table}

\begin{figure}[!htbp]
\centering
\includegraphics[width=0.75\linewidth]{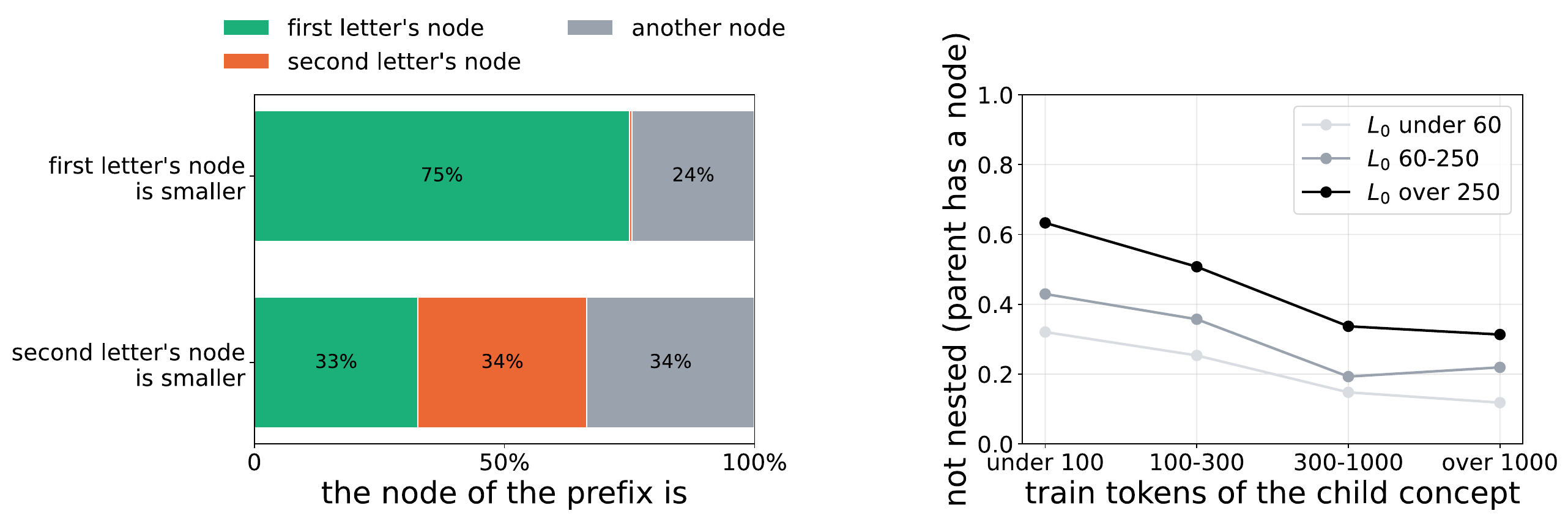}
\caption{Left: the node of a prefix, by which of its two letters has the smaller node. Right: the share of not-nested pairs by the size of the child and by $L_0$.}
\label{fig:app_fca_llm_explain}
\end{figure}

\begin{table}[!htbp]
\centering
\caption{Examples of prefix $\subset$ letter (training tokens of each node in brackets).}
\label{tab:app_fca_llm_explain_examples}
\resizebox{\textwidth}{!}{
\begin{tabular}{lllp{0.5\linewidth}l}
\toprule
SAE & pair & outcome & node of the prefix & node of the letter\\
\midrule
Top-$K$ 4k, $L_0=40$ & wa $\subset$ w & merged & the node of w (5,138): $83\%$ of its tokens contain w, $52\%$ start with w & the same\\
 & tw $\subset$ t & not nested & the node of w (5,138), smaller than that of t; only $6\%$ of it is in the node of t & node of t (10,347)\\
 & pd $\subset$ p & not nested & the node of d (8,724): $93\%$ contain d, $74\%$ start with d & node of p (9,264)\\
 & am $\subset$ a & not nested & a ``contains m'' neuron (12,147): $88\%$ contain m, only $6\%$ start with m & node of a (12,716)\\
 & ar $\subset$ a & not nested & a ``contains ar'' neuron (6,302): $75\%$ contain ar, $13\%$ start with ar & node of a (12,716)\\
Top-$K$ 65k, $L_0=160$ & sw $\subset$ s & $\theta_{C_1}\subset\theta_{C_2}$ & a dedicated sw neuron (331): $47\%$ start with sw, $72\%$ lie in the node of s & node of s (14,947)\\
 & go $\subset$ g & merged & the node of g (6,984): $78\%$ contain g, $60\%$ start with g & the same\\
 & cs $\subset$ c & not nested & a ``contains s'' neuron (12,682): $82\%$ contain s, $10\%$ start with s & node of c (22,676)\\
 & ry $\subset$ r & not nested & the node of y (3,308): $67\%$ contain y, $22\%$ start with y & node of r (15,350)\\
\bottomrule
\end{tabular}}
\end{table}

\begin{table}[!htbp]
\centering
\caption{Share of cases in which each token set reaches F1 $>0.7$ for a human intersection (median F1 in brackets).}
\label{tab:app_fca_llm_meet}
\resizebox{\textwidth}{!}{
\begin{tabular}{lrcccc}
\toprule
Intersection $C_1\cap C_2$ & cases & node of the intersection & meet of the two nodes & best multi-neuron meet & best single neuron\\
\midrule
letter $\wedge$ second letter ($=$ prefix) & 11,484 & 0.2\% (0.02) & 0.0\% (0.03) & 1.3\% (0.11) & 0.3\% (0.07)\\
letter $\wedge$ form & 1,872 & 0.0\% (0.16) & 7.3\% (0.24) & 11.7\% (0.31) & 0.0\% (0.25)\\
space $\wedge$ capital & 36 & 0\% (0.56) & 36\% (0.60) & 33\% (0.63) & 0\% (0.59)\\
word list $\wedge$ form & 972 & 0.0\% (0.04) & 0.4\% (0.07) & 3.8\% (0.26) & 0.1\% (0.10)\\
\bottomrule
\end{tabular}}
\end{table}

\begin{table}[!htbp]
\centering
\caption{The three properties in both directions. The two directions use different concept families.}
\label{tab:app_fca_llm_both}
\small
\begin{tabular}{lp{0.42\linewidth}p{0.38\linewidth}}
\toprule
Property & Human $\to$ model & Model $\to$ human\\
\midrule
order & $\theta_{C_1}\subset\theta_{C_2}$ \textbf{0.8\%}; merged 39\%, not nested 29\%, empty intent 30\%, $\theta_{C_1}\supset\theta_{C_2}$ 0.4\% & with a specific upper neuron ($<1\%$ of tokens), \textbf{70\%} are semantic hypernyms; 32\% at $1$--$10\%$; 3\% above $10\%$; controls 0\%\\
meet & when both have a node, the node of the intersection equals the meet of the two nodes in \textbf{0--4\%}; the meet of the two nodes reaches F1 $>0.7$ for the intersection in 0--7\% (space $\wedge$ capital 36\%) & the meet's own rule reaches F1 $>0.7$ for \textbf{31\%}; ``both rules accept'' 7.5\%\\
round trip & returns: letters \textbf{6\%}, all concepts 0.3\%; 37.5\% return to a broader concept & neuron rules return \textbf{46\%}; letter-node rules 25\%; meet rules 2\%\\
\bottomrule
\end{tabular}
\end{table}

\subsubsection{Limitations}\label{app:fca_llm_limits}
We use one task on isolated tokens, so only lexical features are visible. The language-model samples are moderate in size (tens to a few hundred per analysis), and the language model over-generalizes from salient tokens, so we check each of its claims against counts. The thresholds ($r=0.7$, F1 above $0.7$, and the co-activation levels) are set by hand, with only partial sensitivity checks. The positional concepts, the direct count of second-place tokens, the best-F1 control and the second-letter analysis were added after we saw the first results. Finally, the two directions of the mapping use different concept families, so their shares are not directly comparable.
}

\end{document}